\documentclass[12pt, draftclsnofoot, onecolumn]{IEEEtran}
\usepackage{indentfirst}
\usepackage{hyperref}
\usepackage{algorithm}
\usepackage{algorithmicx}  
\usepackage{algpseudocode}
\usepackage{amsmath}
\usepackage{parskip}
\usepackage{listings}
\usepackage{indentfirst}
\usepackage{graphicx}
\usepackage{float}
\usepackage{extarrows}
\usepackage{amssymb}
\usepackage{amsthm}
\usepackage{epstopdf}
\usepackage{caption}
\usepackage{subcaption}
\usepackage{color}
\usepackage{todonotes}
\usepackage{tikz,pgfplots}
\usepackage{xcolor}
\usetikzlibrary{automata,positioning}
\usetikzlibrary{arrows,shapes,chains}
\usepgflibrary{patterns}
\newtheorem{lemma}{Lemma}

\newtheorem{theorem}{Theorem}
\newtheorem{definition}{Definition}

\newtheorem{remark}{Remark}
\newtheorem{assumption}{Assumption}
\newtheorem{example}{Example}

\begin{document}

\title{Containing a spread through sequential learning: to exploit or to explore?}
\date{}

\author{ Xingran Chen, \IEEEmembership{}
			  Hesam Nikpey, \IEEEmembership{}
			 Jungyeol Kim, \IEEEmembership{}
			 Saswati Sarkar, \IEEEmembership{}
			 Shirin Saeedi-Bidokhti \IEEEmembership{}
	
			\IEEEcompsocitemizethanks
			{\IEEEcompsocthanksitem Xingran Chen, Saswati Sarkar, and Shirin Saeedi-Bidokhti are with the Department of Electrical and System Engineering,  University of Pennsylvania.\quad Email: \{xingranc, swati, saeedi\}@seas.upenn.edu
			\IEEEcompsocthanksitem Hesam Nikpey is with the Deparment of Computer and Information Science, University of Pennsylvania. \quad Email: hesam@seas.upenn.edu
			\IEEEcompsocthanksitem Jungyeol Kim is with the JPMorgan Chase \& Co. \quad Email: jungyeol@alumni.upenn.edu}
			\thanks{This paper has been accepted by Transactions on Machine Learning Research (TMLR).}

} 

\maketitle
	
\begin{abstract} 
	The spread of an undesirable contact process, such as an infectious disease (e.g. COVID-19), is contained through testing and  isolation of infected nodes. The temporal and spatial evolution of the process (along with containment through isolation) render such detection as fundamentally different from active search detection strategies. In this work, through an active learning approach, we design testing and isolation strategies to contain the spread and minimize the cumulative infections under a given test budget. We prove that the objective can be optimized, with performance guarantees,  by greedily selecting the nodes to test.  We further design reward-based methodologies that effectively minimize an upper bound on the cumulative infections and are computationally more tractable in large networks. These policies, however, need knowledge about the nodes' infection probabilities which are dynamically changing and have to be learned by sequential testing. We develop a message-passing framework for this purpose and, building on that,  show novel tradeoffs between exploitation of knowledge through reward-based heuristics and exploration of the unknown through a carefully designed probabilistic testing. The tradeoffs are fundamentally distinct from the classical counterparts under active search or multi-armed bandit problems (MABs).
	We provably show the necessity of exploration in a stylized network and show through simulations that exploration can outperform exploitation in various synthetic and real-data networks depending on the parameters of the network and the spread.
\end{abstract}
	
	\section{Introduction}\label{sec: introduction}
	
	We consider learning and decision making in networked systems for processes that evolve both temporally and spatially. An important example in this class of processes is COVID-19 infection. It evolves in time (e.g. through different stages of the disease for an infected individual) and over a contact network and its spread  can be contained by testing and isolation. Public health systems need to judiciously decide who should be tested and isolated in presence of limitations on the number of individuals who can be tested and isolated on a given day.
	
	Most existing works on this  topic have  investigated the spread of COVID-19 through dynamic systems such SIR models and their variants \cite{BSLSZA1998, PTJC1996, LSBSZA2000, YTWMEB2000, JLA1988, LJSA1994}. These models are made more complex to fit the real data in \cite{AMMAB2021, AGG2020, ASRPSN2020, BMNLT2020, ZJS2020, CMRKBKM2020}. Estimation of the model parameters by learning-based methods are considered and verified by real data in \cite{EBP2020, BMNLTDS2020, IRAGPAFC2021, GHJG2021, RVLFRG2022, HBKDVG2021}. Other attributes such as lockdown policy \cite{SAAMMK2020}, multi-wave prediction \cite{GPDSPSL2020}, herd immunity threshold \cite{SCSRIC2021} are also considered by data-driven experiments. 
	These works mostly focus on  the estimation of model parameters thorough real data, and aim to make a more accurate prediction of the spread. None of
	them, however, consider testing and isolation policies.  {Our work complements these investigations by designing sequential testing and isolation policies in order to minimize the cumulative infections}.  For this purpose, we have assumed full statistical knowledge of the spread model and the underlying contact network and we are not concerned with prediction and estimation of model parameters.

	Designing optimal testing and control policies in dynamic networked systems often involves computational challenges. These challenges have been alleviated in control literature by capturing the spread through differential equations  \cite{WCES2021, AD2021, ALGB2020, tsay2020modeling,piguillem2020optimal}. The differential equations rely on classical mean-field approximations, considering  neighbors of each node as  “socially averaged hypothetical neighbors”.  Refinements of the mean-field approximations such as pair approximation \cite{kuga2021pair}, degree-based approximation \cite{kabir2020impact}, meta-population approximation \cite{kabir2019evolutionary} etc, all resort to some form of averaging  of neighborhoods or more generally groups of nodes. The averaging does not capture the heterogeneity of a real-world complex social network and in effect disregards the contact network topology. But, in practice, the contact network topologies are often  partially known, for example, from contact tracing apps that individuals launch on their phones.  Thus testing and control strategies must exploit the partial topological information to control the spread. The most widely deployed testing and control policy, the (forward and backward) contact tracing (and its variants)  \cite{LWNH2021, JKXCSSBSS, aleta2020modelling, hellewell2020feasibility, kucharski2020effectiveness, KSHDLM2021, Lacent2020, APMCJGM2020, OU2020}, relies on partial knowledge of the network topology (ie, the neighbors of infectious nodes who have been detected), and therefore does not lend itself to mean-field analysis. Our proposed framework considers both the SIR evolution of the disease for each node and the spread of the disease through a given network.
	
	The following challenges arise in the design of intelligent testing strategies if one seeks to exploit the spatio-temporal evolution of the disease process and comply with limited testing budget. Observing the state of a node at time $t$ will provide information about  the state of (i) the node in time $t+1$ and (ii) the neighbors of the node at time $t, t+1, \ldots$. This is due to the inherent correlation that exists between  states of neighboring nodes because an infectious disease spreads through contact. Thus, testing has a dual role. It has to both detect/isolate infected nodes and learn the spread in various localities of the network. The spread can often be silent: an undetected  node (that may not be particularly likely to be infected based on previous observations) can infect its neighbors. Thus, testing nodes that do not necessarily appear to be infected may lead to timely discovery of even larger clusters of infected nodes waiting to explode. In other words, there is \emph{an intrinsic tradeoff between  exploitation of knowledge~vs.~exploration of the unknown}.  
	Exploration~vs.~exploitation tradeoffs were originally studied in  classical multi-armed bandit (MAB) problems where there is the notion of a single optimal arm that can be found by repeating a set of fixed actions \cite{PANCBPF2002, SANG2011, SBNCB2012}. MAB testing strategies have also been designed for exploring partially observable networks \cite{KMTM2019}. Our problem differs from what is mainly studied in the MAB literature because (i) the number of arms (potential infected nodes) is time-variant and actions cannot be repeated; (ii) the exploration vs. exploitation tradeoff in our context arises due to lack of knowledge about the time-evolving  set of infected nodes, rather than lack of knowledge about the network or the process model and its parameters. 
	
	Note that contact tracing policies are in a sense exploitation policies: upon finding positive nodes, they exploit that knowledge and trace the contacts. While relatively practical, they have two main shortcomings, as implemented today: (i) They are not able to  prioritize nodes based on their likelihood of being infected (beyond the coarse notion of contact or lack thereof). For example, consider an infectious node that has two neighbors, with different degrees. Under current contact tracing strategies, both neighbors have the same status. But in order to contain the spread as soon as possible, the node with a large degree should be prioritized for testing. A similar drawback becomes apparent if the neighbors themselves  have a different number of infectious neighbors; one with a larger number of infectious neighbors should be prioritized for testing, but current contact tracing strategies accord both the same priority.  (ii) Contact tracing strategies do not incorporate any type of exploration. This may be a fundamental limitation of contact tracing.  \cite{OU2020} has shown that, with high cost, contact tracing policies perform better when they incorporate exploration (active case finding). In contrast, our work provides a probabilistic framework to not only allow for exploitation in a fundamental manner but also to incorporate exploration in order to minimize the number of infections.

	Finally, our problem is also related to  active search in graphs where the goal is to test/search for a set of (fixed) target nodes under a set of given (static) similarity values between pairs of nodes \cite{MBLMLG2010, XWRGJS2013, YMTKHJS2015, RGYKDWJS2011}. But the target nodes in these works are assumed  fixed, whereas the target is dynamic in our setting because the infection spreads over time and space (i.e, over the contact network). Thus,  a node may need to be tested multiple times. The importance of exploitation/exploration is also known,  implicitly and/or explicitly,  in various reinforcement learning literature \cite{DZDCXT2018, BY18, BM16}.
	
We now distinguish our work from testing strategies that combine exploitation and exploration in some form  \cite{RFNS2020,GH2020, EMHMSMGC2021}.  Through a theoretical approach, \cite{RFNS2020} models the testing problem as a  partially observable Markov decision process (POMDP). An optimal policy can, in principle, be formulated through POMDP, but such strategies are intractable in their general form (and heuristics are often far from optimal) \cite{KLM1998, MG82}. \cite{RFNS2020} devises tractable approximate algorithms with a significant caveat: In the design, analysis, and evaluation of the proposed algorithms, it is assumed that at each time the process can spread only on a single random edge of the network. This is a very special case that is hard to justify in practice and it is not clear how one could go beyond this assumption. On the other hand, \cite{GH2020} proposes a heuristic by implementing classical learning methods such as Linear support vector machine (SVM) and Polynomial SVM to rank nodes based on a notion of risk score (constructed by real-data) while reserving a portion of the test budget for random testing which can be understood as exploration. No spread model or contact network is assumed.  \cite{EMHMSMGC2021} and this work were done concurrently. In \cite{EMHMSMGC2021}, a tractable scheme to control dynamical processes on temporal graphs was proposed, through a POMDP solution with a combination of Graph Neural Networks (GNN) and Reinforcement Learning (RL) algorithm. Nodes are tested based on some scores obtained by the sequential learning framework, but no fundamental probabilities of the states of nodes were revealed. Different from \cite{GH2020, EMHMSMGC2021}, our approach is model-based and we observe novel exploration-exploitation tradeoffs that arise not due to a lack of knowledge about the model or network, but rather because the set of infected nodes is unknown and evolves with time. We can also utilize knowledge about both the model and the contact network to devise a probabilistic framework for decision making.
		
		We now summarize the contribution of some significant works that consider only exploitation and do not utilize any exploration \cite{Lacent2020, APMCJGM2020, OU2020}.  \cite{Lacent2020, APMCJGM2020} have considered a combination of isolation and contact tracing sequential policies, and \cite{Lacent2020} has shown that the sequential strategies would reduce transmission more than mass testing or self-isolation alone, while \cite{APMCJGM2020} has shown that the sequential strategies can reduce the amount of quarantine time served  and cost, hence individuals may increase participation in contact tracing, enabling less frequent and severe lockdown measures. \cite{OU2020} have proposed a novel approach to modeling multi-round network-based screening/contact tracing under uncertainty.

	\paragraph{Our Contributions}
	In this work, we study a spread process such as  Covid-19 and  design sequential testing and isolation policies to contain the spread. Our contributions are as follows.
	
	\begin{itemize}
		\item Formulating the spread process through a compartmental model and a given contact network, we show that the problem of minimizing the total cumulative infections under a given test budget reduces to minimizing  a supermodular function expressed in terms of nodes' probabilities  of infection and it thus admits a \emph{near-optimal greedy} policy. We further design reward-based algorithms that  minimize an upper bound on the  cumulative infections and are computationally more tractable in large networks.

		\item The greedy policy and its reward-based derivatives are applicable if nodes' probabilities  of infection were known. However, since the set of infected nodes are unknown, these probabilities are unknown and can only be learned through \emph{sequential testing}. We provide a {message-passing framework} for sequential estimation of nodes' \emph{posterior} probabilities of infection given the history of test observations.
		
		\item We argue that testing has a dual role: (i) discovering and isolating the infected nodes to contain the spread, and (ii) providing more accurate estimates for nodes' infection probabilities which are used for decision making. In this sense, \emph{exploitation} policies in which decision making only targets (i) can be suboptimal. We prove in a stylized network that when the belief about the probabilities is wrong, exploitation can be arbitrarily bad, while a policy that combines exploitation with random testing  can contain the spread. This points to novel \emph{exploitation-exploration tradeoffs} that stem from the lack of knowledge about the location of infected nodes, rather than the network or  spread process.
		
		\item Following these findings, we propose  \emph{exploration} policies that test each node  probabilistically according to its reward. The core idea is to balance exploitation of knowledge (about the nodes' infection probabilities and the resulting rewards) and exploration of the unknown (to get more accurate estimates of the infection probabilities). 
		Through simulations, we compare the performance of exploration and exploitation policies in several synthetic and real-data networks. In particular, we investigate the role of  three parameters on when exploration outperforms exploitation: (i)  the \emph{unregulated delay}, i.e., the time period when the disease spreads without intervention; (ii) the \emph{global clustering coefficient} of the  network, and (iii) the average \emph{shortest path length} of the  network. We show that when the above  parameters increase, exploration becomes more beneficial  as it provides better estimates of the nodes' probabilities of infection.
		
	\end{itemize}

	\section{Modeling}\label{sec: Modeling}
	To describe a spread process, we use a {\it discrete time compartmental  model} \cite{compartmental}. Over decades, compartmental models have been key in the study  of epidemics and opinion dynamics, albeit often disregarding the network topology. In this work, we capture the spread on a given contact network. For clarity of presentation, we focus on a model for the spread of COVID-19. The ideas can naturally be generalized to other applications. The main notations in the full paper are given in Table~\ref{tabla: notation}.
	
	\begin{table}[!htbp]
		\centering
		\begin{tabular}{|c|c|}
			\hline  
			Notations & Definitions\\
			\hline
			$\beta$ &  transmission probability\\
			\hline
			$1/\gamma$ &  mean duration in the latent state\\
			\hline  
			$1/\lambda$&  mean duration in the infectious state\\
			\hline
			$\sigma_i(t)$ &  state of node $i$ at time $t$, $\sigma_i(t)\in\{I, S, L, R\}$\\
			\hline
			$\mathcal{G}(t)$ & contact network at time $t$\\
			\hline
			$\mathcal{V}(t)$ &  set of nodes at time $t$\\
			\hline 
			$\mathcal{E}(t)$ &  set of edges at time $t$\\
			\hline 
			$N(t)$ &  cardinality of $\mathcal{V}(t)$\\
			\hline 
			$N$ & $N = N(0)$\\
			\hline 
			$\partial_i(t)$ &  neighbors of node $i$ at time $t$\\
			\hline
			$\partial^+_i(t)$ & $\{i\}\cup\partial_i(t)$\\
			\hline 
			$Y_i(t)$ &  testing result of node $i$ at time $t$\\
			\hline 
			$\mathcal{O}(t)$ &  set of nodes tested at time $t$\\
			\hline 
			$\underline{Y}(t)$ & $\{Y_i(t)\}_{\{i\in\mathcal{O}(t)\}}$\\
			\hline 
			$B(t)$ &  testing budget at time $t$\\
			\hline 
			$\pi$ & a testing and isolation policy\\
			\hline 
			$C^\pi(t)$ &  cumulative infections at time $t$\\
			\hline 
			$\mathcal{K}^\pi(t)$ &  set of nodes tested at time $t$ (under policy $\pi$)\\
			\hline 
			$K^\pi(t)$ & $K^\pi(t)=|\mathcal{K}^\pi(t)|$\\
			\hline 
			$T$ &  time horizon\\
			\hline 
			$\underline{v}_i(t)$ &  true probability vector of node $i$\\
			\hline 
			$\underline{u}_i(t)$ &  prior probability vector of node $i$\\
			\hline 
			$\underline{w}_i(t)$ &  posterior probability vector of node $i$\\
			\hline 
			$\underline{e}_i(t)$ &  updated posterior probability vector of node $i$\\
			\hline 
			$r_i(t)$ &  rewards of selecting node $i$ at time $t$\\
			\hline 
			$\hat{r}_i(t)$ &  estimated rewards of node $i$ at time $t$\\
			\hline 
			$\Psi_i(t)$ & $\Psi_i(t) = \mathcal{O}(t)\cap\partial^+_i(t-1)$\\
			\hline 
			$\Phi_i(t)$ & $\Phi_i(t) = \{j|j\in\partial^+_k(t-1), k\in\Psi_i(t)\}\backslash\{i\}$\\
			\hline 
			$\theta_i(t)$ & $\theta_i(t) = \sigma_i(t)|_{\{\underline{Y}(\tau)\}_{\tau=1}^{t-1}}$, $\theta_i(t)\in\{I, S, L, R\}$\\
			\hline 
			$\zeta_i(t)$ & $\zeta_i(t) = \sigma_i(t)|_{\{\underline{Y}(\tau)\}_{\tau=1}^{t}}$, $\zeta_i(t)\in\{I, S, L, R\}$\\
			\hline 
		\end{tabular}
		\caption{Summary of main notations}\label{tabla: notation}
	\end{table}

	We model the progression of Covid-19 per individual, in time, through four stages or states: {\it Susceptible ($S$)}, {\it Latent ($L$)},  {\it Infectious ($I$)}, and {\it Recovered ($R$)}.  Per contact, an infectious individual infects a susceptible individual with  transmission probability  $\beta$.
	An infected individual is initially in the latent  state $L$, subsequently he becomes infectious (state $I$), finally he recovers (state $R$). Fig.~\ref{parameters}~(left) depicts the evolution. The durations in the latent and infectious states are geometrically distributed, with means $1/\lambda, 1/\gamma$ respectively. 
	We represent the state of node $i$ at time $t$ by random variable $\sigma_i(t)$ and its support set $\mathcal{X}= \{S,  L, I, R\}$. We assume that the parameters $\beta$, $\lambda$ and $\gamma$ are known to the public health authority. This is a practical assumption because the parameters can be estimated by the public health authority based on the pandenmic data collected \cite{MSZY2020, ABSJ2020,  contact}.
	\begin{figure}[t!]
		\centering
		\includegraphics[width=0.4\textwidth]{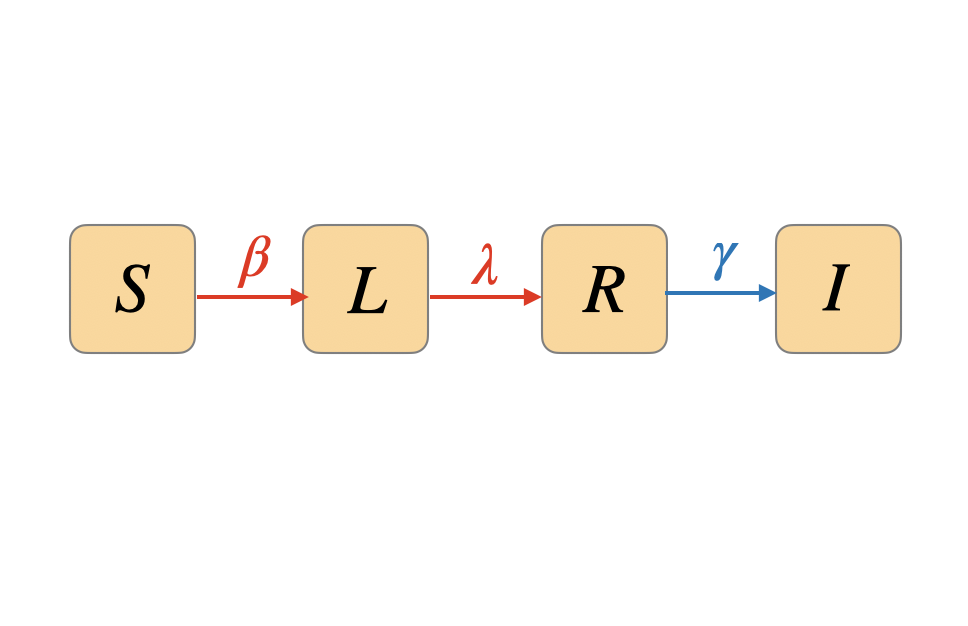}\quad\quad\quad
		\includegraphics[width=0.35\textwidth]{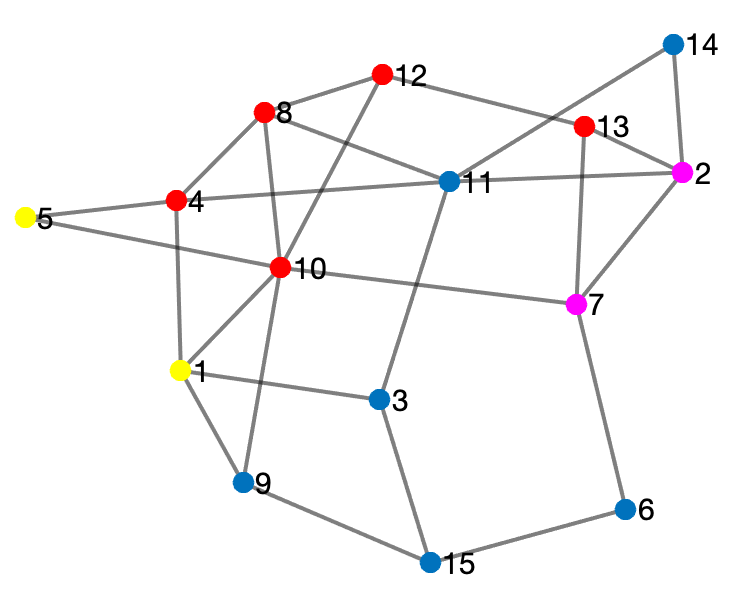}   
		\caption{Left: Time evolution of the process per individual nodes. Right: A contact network with nodes in states susceptible (blue), latency  (pink), infectious  (red),  recovered (yellow).}
		\label{parameters}
	\end{figure}
	
	Let $\mathcal{G}(t)=(\mathcal{V}(t),\mathcal{E}(t))$ denote the contact network at time $t$, where $\mathcal{V}(t)$ is the set of nodes/individuals, of cardinality $N(t)$, and $\mathcal{E}(t)$ is the set of edges between the nodes, describing interactions/contacts on day $t$. Let $\mathcal{V}=\mathcal{V}(0)$, $\mathcal{E} = \mathcal{E}(0)$, $\mathcal{G} = \mathcal{G}(0)$, and $N = N(0)$. The network is time-dependent not only because interactions change on a daily basis, but also because  nodes may be tested and isolated. If a node is tested positive on any day $t$, it will be isolated immediately. If a node is isolated on any day $t$, we assume that it remains in isolation until he recovers. We assume that a recovered node can not be reinfected again. Thus a node that is isolated on any day $t$ has no impact on the network from then onwards. Such nodes can be regarded as ``removed''. Therefore, it is removed from the contact network for all subsequent times $t, t+1,\ldots$.  Fig.~\ref{parameters}~(right) depicts a contact network at a given time~$t.$ We assume that a public health authority knows the entire contact network and decides who to test based on this information. This assumption has been made in several other works in this genre eg in \cite{OU2020}.
	
	Denote the set of neighbors of node $i$, in  day $t$, by  $\partial_i(t)$.
	The state of each node at time $t+1$ depends on the state of its neighbors $\partial_i(t)$, as well as its own state in day $t$, as given by the following conditional probability:  
	$$\Pr\big(\sigma_i(t+1)|\{\sigma_j(t)\}_{j\in\partial^+_i(t)}\big)$$
	where $\partial^+_i(t) = \partial_i(t)\cup\{i\}$.

	Node $i$ is tested positive on day $t$ if it is in the infectious state ($I$)\footnote{We assume that a node in the latent state $L$ is infected, but not infectious. We further assume that latent nodes test negative.}. Let  $Y_i(t)$ denote the test result:
	\begin{align}
		Y_i(t)=\left\{\begin{array}{cc}1&\sigma_i(t)=I\\0&\sigma_i(t)\in\{S,L,R\}. \end{array}\right.
	\end{align}	
	We do not assume any type of error in testing and  $Y_i(t)$ is hence a deterministic function of $\sigma_i(t)$.
	Let  $\mathcal{O}(t)$ be the set of nodes that have been tested (observed) in day $t$ and denote  the {\it network observations} at time~$t$ by $\underline{Y}(t) = \{Y_i(t)\}_{i\in\mathcal{O}(t)}$.

	Our goal in this paper is to design testing and isolation strategies in order to contain the spread and minimize the cumulative infections.
	Naturally, testing resources (and hence observations) are often limited and such constraints make decision making challenging.
	Let $B(t)$ be the maximum number of tests that could be performed on day $t$, called the {\it testing budget}.  $B(t)$ can evolve based on the system necessities, e.g., in contact tracing that is widely deployed  for COVID-19, the number of tests is chosen based on the history of observations\footnote{In practical implementations, scheduling constraints do play a role but we disregard that in this work.}.
	Also, governments often upgrade testing infrastructure as the number of cases increase. Our framework captures both fixed and time-dependent budget $B(t)$, but we focus on time-dependent $B(t)$ for  simulations.
	
	Define
	the cumulative infections on day $t$, denoted by $C^\pi(t)$, as the number of nodes who have been infected before and including day $t$,
	where $\pi$ is the testing and isolation policy.
	Let $\mathcal{K}^{\pi}(t)$ denote the set of tests $\pi$ performs on day $t$. Given a large time horizon $T$, our objective is: 
	\begin{equation}\label{eq: goal1}
		\begin{aligned}
			&\min_{\pi}&\quad&\mathbb{E}[C^\pi(T)]\\
			&s.t.&\quad& |\mathcal{K}^{\pi}(t)|\leq B(t),\, 0\leq t\leq T-1.
		\end{aligned}
	\end{equation}
	
	Recall that $\sigma_i(t)$, the state of node $i$ on day $t$, is a random variable and unknown. For each node $i$, define a probability vector  $\underline{v}_i(t)$ of size $|\mathcal{X}|$, where each coordinate is the probability of the node being in a particular state at the end of time $t$. The coordinates of $\underline{v}_i(t)$ follow the order $(I, L, R, S)$ and we have
	\begin{align}\label{eq: rowvectorv}
		\underline{v}_i(t) =&  \big[v_x^{(i)}(t)\big]_{ x\in\mathcal{X}},\qquad v_{x}^{(i)}(t)= \Pr\big(\sigma_i(t) = x\big).
	\end{align}
	For example, $v_{S}^{(i)}(t)$ represents the probability of node $i$ being in state $S$ in time $t$.
	We now define $F_i(\mathcal{D};t)$ to be the {\it conditional} probability of node~$i$ being infected by nodes in $\mathcal{D}$ (for the first time) at day $t$, as a function of the nodes' states $\{\sigma_i(t)\}_{i\in\mathcal{V}(t)}$. We have
	\begin{equation}\label{eq: fiK(t)}
		\begin{aligned}
			F_i(\mathcal{D}; t)= 1_{\{\sigma_i(t)=S\}}\cdot\prod_{j\in\partial_i(t)\backslash\mathcal{D}}\big(1-\beta 1_{\{\sigma_j(t)=I\}}\big)\cdot\big(1 - \prod_{j\in\mathcal{D}\cap\partial_i(t)}(1-\beta 1_{\{\sigma_j(t)=I\}})\big).
		\end{aligned}
	\end{equation}
	
	Equation (\ref{eq: fiK(t)}), captures the impact that the nodes in $\mathcal{D}$ have on infecting node $i$ at day $t$. In this equation  we assume that the infections from different nodes are independent. The same assumption has also been made in several other papers in this genre, eg in
		\cite{kuga2021pair, kabir2020impact, kabir2019evolutionary}.
	Then, we find the expectation (with respect to $\{\sigma_i(t)\}_{i\in\mathcal{V}(t)}$) of (\ref{eq: fiK(t)}) as follows:
	\begin{equation}\label{eq: fEF}
		\begin{aligned}
			\mathbb{E}_{\{\sigma_i(t)\}_{i\in\mathcal{V}(t)}}\left[F_i(\mathcal{D}; t)\right]= v_S^{(i)}(t)\cdot\big\{\prod_{j\in\partial_i(t)\backslash\mathcal{D}}(1-\beta v_I^{(j)}(t))\big\}\cdot\big\{1 - \prod_{j\in\mathcal{D}\cap\partial_i(t)}(1-\beta v_I^{(j)}(t))\big\}.
		\end{aligned}
	\end{equation}
	It is worth noting that (\ref{eq: fiK(t)}) is a probability conditioned on $\{\sigma_i(t)\}_{i\in\mathcal{V}(t)}$, while (\ref{eq: fEF}) is an unconditional probability. To obtain (5), we have indeed assumed that the states of the nodes are independent. This assumption does not hold in general and we only utilize it here to obtain a simple expression in (5) in terms of the infection probabilities. We do not use this independence assumption in the rest of the paper.
	Define 
	\begin{align}\label{eq: supmodular}
		S\big(\mathcal{D}; t\big) = \sum_{i\in\mathcal{V}(t)} \mathbb{E}\left[F_i(\mathcal{D}; t)\right].
	\end{align}
	Here, $S\big(\mathcal{D}; t\big)$ represents the (expected) number  of newly infectious nodes  incurred by nodes in $\mathcal{D}$ at day $t$.
	Recall that $\mathcal{K}^\pi(t)$ be the set of nodes that are tested at time $t$. We show the following result in  Appendix~\ref{App: lemma telescopic sum}.
	\begin{lemma}\label{lem: telescopic sum}
		$\mathbb{E}\big[C^\pi(t+1) - C^\pi(t)\big] = S\big(\mathcal{V}(t)\backslash\mathcal{K}^\pi(t); t\big)$.
	\end{lemma}

	\subsection{Supermodularity}\label{sec: greedy max}
	
	It is complex to solve (\ref{eq: goal1}) globally, especially if one seeks to find solutions that are optimal looking into the future. 
	We thus simplify the optimization (\ref{eq: goal1}) for policies that are myopic in time as follows.
	First, note that $C^\pi(T)$ can be re-written as follows through a telescopic sum:
	\begin{align}
		\label{eq:CT}
		C^\pi(T)=\sum_{t=0}^{T-1} C^\pi(t+1)-C^\pi(t).
	\end{align}
	Then, we restrict attention to myopic policies that at each time minimize $\mathbb{E}\left[C^\pi(t+1)-C^\pi(t)\right]$. We  then 
	show how  $\mathbb{E}\left[C^\pi(t+1)-C^\pi(t)\right]$ can be expressed in terms of a supermodular function.

	Using (\ref{eq:CT}) along with Lemma~\ref{lem: telescopic sum}, we seek to solve the following optimization sequentially in time for $0\leq t\leq T-1$:
	\begin{equation}\label{eq: goal3}
		\begin{aligned}
			\min_{|\mathcal{K}^\pi(t)|\leq B(t)}\quad S\big(\mathcal{V}(t)\backslash\mathcal{K}^\pi(t); t\big).
		\end{aligned}
	\end{equation}
	
	We now prove some desired properties for the set function $S(\mathcal{K}^\pi(t); t)$ (see  Appendix~\ref{App: proof of supmodular}). 
	\begin{theorem}\label{thm: supmodular}
		$S\big(\mathcal{K}^\pi(t); t\big)$ defined in (\ref{eq: supmodular}) is a supermodular\footnote{Let $\mathcal{X}$ be a finite set. A function $f:\, 2^{\mathcal{X}}\to \mathbb{R} $ is supermodular if for any $\mathcal{A}\subset \mathcal{B}\subset \mathcal{X}$, and $x\in \mathcal{X}\backslash \mathcal{B}$,  $f(\mathcal{A}\cup \{x\})-f(\mathcal{A})\leq f(\mathcal{B}\cup \{x\})-f(\mathcal{B})$.} and increasing monotone function on $\mathcal{K}^\pi(t)$.
	\end{theorem}
	On day $t$, and given the network, the probability vectors of all nodes, and $\mathcal{K}_1^\pi(t)\subset\mathcal{K}_2^\pi(t)$, for any node $i\notin\mathcal{K}_2^\pi(t)$, node $i$ will incur larger increment of newly infectious nodes under $\mathcal{K}_2^\pi(t)$ than that under $\mathcal{K}_1^\pi(t)$. This is because node $i$ may have common neighbors with nodes in $\mathcal{K}_2^\pi(t)$. So, supermodularity holds in Theorem~\ref{thm: supmodular}.

	The optimization (\ref{eq: goal3}) is NP-hard \cite{submodularbook}. However, using the supermodularity of $S\big(\mathcal{V}(t)\backslash\mathcal{K}^\pi(t); t\big)$, we propose Algorithm~\ref{alg: Greedy Algorithm} based on \cite[Algorithm~A]{V2001} to greedily optimize (\ref{eq: goal3}) in every day $t$. Denote the optimum solution of (\ref{eq: goal3}) as $\text{OPT}$. As proved in \cite{V2001}, on every day $t$, Algorithm~\ref{alg: Greedy Algorithm} attains a solution, denoted by $\tilde{\mathcal{K}}^\pi(t)$, such that $\big(\mathcal{V}(t)\backslash\tilde{\mathcal{K}}^\pi(t); t\big)\leq\big(1+\epsilon(t)\big)\cdot\text{OPT}$, i.e., the solution $\tilde{\mathcal{K}}^\pi(t)$ is an $\epsilon(t)$-approximation of the optimum solution. Here, on day $t$, the constant $\epsilon(t)$, which is the steepness of the set function $S(\cdot; t)$ as described in \cite{V2001}, can be calculated as follows, $\epsilon(t)=\frac{\epsilon'}{4(1-\epsilon')}$ and $\epsilon'=\max_{a\in\mathcal{V}(t)}\frac{S(\mathcal{V}(t); t)-S(\mathcal{V}(t)\backslash\{a\}; t)-S(\{a\}; t)}{S(\mathcal{V}(t); t)-S(\mathcal{V}(t)\backslash\{a\}; t)}$.

	In Algorithm~\ref{alg: Greedy Algorithm}, on every day $t$, in every step, we choose the node who provides the minimum increment on $S(\cdot; t)$ based on the results in the previous step, and then remove the node from the current node set. Algorithm~\ref{alg: Greedy Algorithm} is stopped when $K^\pi(t)$ nodes have been chosen.
	The complexity of this algorithm is discussed in Appendix~\ref{App: complexity of greedy}.
	\begin{algorithm}
		\caption{Greedy Algorithm}\label{alg: Greedy Algorithm}
		\begin{algorithmic}
			\State {\bf Step~0}: On day $t$, input $\{\underline{v}_i(t)\}_{i\in\mathcal{V}(t)}$, set $\mathcal{A}_0 = \mathcal{V}(t)$.
			\Repeat 
			\State {\bf Step i}: Let $\mathcal{A}_i=\mathcal{A}_{i-1}\backslash\{a_i\}$, where $$a_i=\arg\min_{a\in\mathcal{A}_{i-1}}S\big(\{a\}\cup\{a_1,\cdots,a_{i-1}\}; t\big).$$ 
			\Until $i=N(t)-|\mathcal{K}^\pi(t)|$, and return $\mathcal{K}^\pi(t)=\mathcal{A}_{i}$.
		\end{algorithmic}
	\end{algorithm}

	\section{Exploitation and Exploration}\label{sec: The Proposed Policies}
	
	In Section~\ref{sec: greedy max}, we  proposed a near-optimal greedy algorithm to sequentially (in time) select the nodes to test. However, Algorithm~\ref{alg: Greedy Algorithm} has two shortcomings. (i) The computation is costly when $N$ and/or $T$ are large (see Appendix~\ref{App: complexity of greedy}). (ii) The objective function $S\big(\mathcal{V}(t)\backslash\mathcal{K}^\pi(t)\big)$ is dependent on $\{\underline{v}_i(t)\}_{i\in\mathcal{V}(t)}$ which is unknown, even though the network and the process are stochastically fully given (see Section~\ref{sec: Modeling}). This is because the set of infected nodes are unknown and time-evolving. 
	
	To overcome the first shortcoming, we propose a simpler reward maximization policy by minimizing an upper bound on the objective function in (\ref{eq: goal3}). To overcome the second shortcoming, we estimate $\{\underline{v}_i(t)\}_{i\in\mathcal{V}(t)}$ using the history of test observations $\{\underline{Y}(\tau)\}_{\tau=0}^{t}$ (as presented in Section~\ref{sec: Posterior Probabilities}). we refer to the estimates as $\{\underline{u}_i(t)\}_{i\in\mathcal{V}(t)}$. Both the greedy policy and its reward-based variant that we will propose in this section thus need to perform decision making based on the estimates $\{\underline{u}_i(t)\}_{i\in\mathcal{V}(t)}$ and we refer to them as ``exploitation'' policies.
	
	It now becomes clear that testing has two roles: to find the infected in order to isolate them and contain the spread, and to provide better estimates of $\{\underline{v}_i(t)\}_{i\in\mathcal{V}(t)}$. This leads to interesting tradeoffs between exploitation and exploration as we will discuss next. Under exploitation policies, we test nodes deterministically based on a function of $\{\underline{v}_i(t)\}_{i\in\mathcal{V}(t)}$, (which is called ``reward'', and will be defined later); while under exploration policies, nodes are tested according to a probabilistic framework (based on rewards of all nodes).

	To simplify the decision making into reward maximization, we first derive an upper bound on $S\big(\mathcal{V}(t)\backslash\mathcal{K}^\pi(t); t\big)$. 
	Define 
	\begin{align}\label{eq: reward r}
		r_i(t) =  S\big(\{i\}; t\big).
	\end{align}
	Using the supermodularity of the function $S(\cdot)$, we prove the following lemma in Appendix~\ref{App: supermodularity}.
	\begin{lemma}\label{lem: upperbound}
		$S\big(\mathcal{V}(t)\backslash\mathcal{K}^\pi(t); t\big)\leq S\big(\mathcal{V}(t); t\big)-\sum_{i\in\mathcal{K}^\pi(t)}r_i(t)$.
	\end{lemma}
	\begin{remark}
		Recall that $S(\cdot; t)$ is a supermodular function, then the amount of newly infectious nodes incurred by the set $\mathcal{K}^\pi(t)$, $S(\mathcal{K}^\pi(t); t)$, is larger than the sum of the amount of newly infectious nodes by every individual node in $\mathcal{K}^\pi(t)$, i.e., $\sum_{i\in\mathcal{K}^\pi(t)}r_i(t)$. Thus, $S\big(\mathcal{V}(t)\backslash\mathcal{K}^\pi(t); t\big)$ is upper bounded by $S\big(\mathcal{V}(t); t\big)-\sum_{i\in\mathcal{K}^\pi(t)}r_i(t)$.
	\end{remark}

	We propose to minimize the upper bound in Lemma~\ref{lem: upperbound} instead of  $S\big(\mathcal{V}(t)\backslash\mathcal{K}^\pi(t);t\big)$. Since $\mathcal{V}(t)$ is known and  $S\big(\mathcal{V}(t);t\big)$ is hence a constant, the problem reduces to solving: \begin{equation}\label{eq: goal4} 
		\begin{aligned} \max_{|\mathcal{K}^\pi(t)|\leq B(t)}\quad \sum_{i\in\mathcal{K}^\pi(t)}r_i(t). 
		\end{aligned} 
	\end{equation}
	Given  probabilities $\{\underline{v}_i(t)\}_{i\in\mathcal{V}(t)}$, the solution to (\ref{eq: goal4}) is  to pick the nodes associated with the $B(t)$ largest  values $r_i(t)$. We thus refer to $r_i(t)$ as the \emph{reward} of selecting node $i$. 
	
	Let $\{\underline{u}_i(t)\}_{i\in\mathcal{V}(t)}$ be an estimate for $\{\underline{v}_i(t)\}_{i\in\mathcal{V}(t)}$ found by estimating the conditional probability of the state of node $i$ given the history of observations $\{\underline{Y}(\tau)\}_{\tau=0}^{t-1}$. 
	Our proposed reward-based Exploitation (RbEx) policy follows the same idea of selecting the nodes with the highest rewards. Note that $\{\underline{v}_i(t)\}_{i\in\mathcal{V}(t)}$ is unknown to all nodes. Instead of using the true probabilities $\{\underline{v}_i(t)\}_{i\in\mathcal{V}(t)}$, we consider the estimates of it which we sequentially update by computing the prior probabilities $\{\underline{u}_i(t)\}_{i\in\mathcal{V}(t)}$ and the posterior probabilities $\{\underline{w}_i(t)\}_{i\in\mathcal{V}(t)}$.  In particular, $\{\underline{u}_i(0)\}_{i\in\mathcal{V}(0)}$ and $\{\underline{w}_i(0)\}_{i\in\mathcal{V}(0)}$  are the prior probabilities and the posterior probabilities on the initial day, respectively. Hence, we calculate the estimate of rewards, denoted by $\hat{r}_i(t)$, by replacing $\{\underline{v}_i(t)\}_{i\in\mathcal{V}(t)}$ with $\{\underline{u}_i(t)\}_{i\in\mathcal{V}(t)}$ in (\ref{eq: supmodular}) and (\ref{eq: reward r}).
	\begin{algorithm}
		\caption{Reward-based Exploitation (RbEx) Policy}\label{alg: reward-based exploitation}
		\begin{algorithmic}
			\State Input $\{\underline{w}_i(0)\}_{i\in\mathcal{V}(t)}$, $\{\underline{u}_i(0)\}_{i\in\mathcal{V}(0)}$, $\underline{Y}(0)$, and $t=0$.
			
			\noindent{\bf Repeat} for $t=1,2,\cdots,T-1$.
			\State {\bf Step 1}: Calculate $\{\hat{r}_i(t)\}_{i\in\mathcal{V}(t)}$ based on $\{\underline{u}_i(t)\}_{i\in\mathcal{V}(t)}$ and \eqref{eq: reward r}.
			\State {\bf Step 2}: Re-arrange the sequence $\{\hat{r}_i(t)\}_{i\in\mathcal{V}(t)}$ in descending order, and test the first $B(t)$ nodes. Get the new observations $\underline{Y}(t)$.
			\State {\bf Step 3}: Based on $\underline{Y}(t)$, update $\{\underline{u}_i(t+1)\}_{i\in\mathcal{V}(t+1)}$ by  Algorithm~\ref{alg: Backward Update} (Step 0 $\sim$ Step 2) in Section~\ref{sec: Posterior Probabilities}.
		\end{algorithmic}
	\end{algorithm}
	
	The shortcoming of Algorithm~\ref{alg: reward-based exploitation} is that it targets maximizing the estimated sum rewards, even though the estimates may be inaccurate. In this case, testing is heavily biased towards the history of testing and it does not provide opportunities for getting better estimates of the rewards. For example, consider a network with several clusters. If one positive node is known by Algorithm~\ref{alg: reward-based exploitation}, then it may get stuck in that cluster and fail to locate more positives in other clusters.
	
	In Section~\ref{sec: Exploiration and Exploration}, we will prove, in a line network, that the exploitation policy described in Algorithm~\ref{alg: reward-based exploitation} can be improved by a constant factor (in terms of the resulting cumulative infections) if a simple form of  exploration is incorporated.
	
	We next propose an exploration policy.
	Our proposed policy is probabilistic in the sense that the nodes are randomly tested with probabilities that are proportional to their corresponding estimated rewards.
	This approach has similarities and differences to Thompson sampling and more generally posterior sampling. The similarity lies in the probabilistic nature of testing using posterior probabilities. The difference is that in our setting decision making depends on the distributions of decision variables, but not  samples  of the decision variables.

	More specifically, at  time $t$, node $i$ is tested with probability $\min\{1, \frac{B(t)\hat{r}_i(t)}{\sum_{j\in\mathcal{V}(t)}\hat{r}_j(t)}\}$, which depends on the budget $B(t)$.
	Note that each node is tested with probability at most $1$; so if $\frac{B(t)\hat{r}_i(t)}{\sum_{j\in\mathcal{V}(t)}\hat{r}_j(t)}>1$ for some node $i$,  then we would not fully utilize the budget. The unused budget is thus
	\begin{align}\label{eq: compensatoryfactor}
		c(t) = \sum_{i\in\mathcal{V}(t)}\big(\frac{B(t)\hat{r}_i(t)}{\sum_{j\in\mathcal{V}(t)}\hat{r}_j(t)}-1\big)^+
	\end{align}
	and can be used for further testing\footnote{\label{ft}Note that $c(t)$ is not always an integer. Instead of $c(t)$, we use ${\tt Int}\big(c(t)\big)$ with probability $|{\tt Int}\big(c(t)\big) - c(t)|$ where ${\tt Int}(\cdot)\in\{\lfloor\cdot\rfloor, \lceil\cdot\rceil\}$.}.
	Algorithm~\ref{alg: reward-based exploration} outlines our proposed  Reward-based Exploitation-Exploration (REEr) policy. 
	
	\begin{algorithm}
		\caption{Reward-based Exploitation-Exploration (REEr) Policy}\label{alg: reward-based exploration}
		\begin{algorithmic}
			\State Input $\{\underline{w}_i(0)\}_{i\in\mathcal{V}(t)}$, $\{\underline{u}_i(0)\}_{i\in\mathcal{V}(0)}$, $\underline{Y}(0)$, and $t=0$.
			
			\noindent{\bf Repeat} for $t=1,2,\cdots,T-1$
			\State {\bf Step 1}: Calculate $\{\hat{r}_i(t)\}_{i\in\mathcal{V}(t)}$ based on $\{\underline{u}_i(t)\}_{i\in\mathcal{V}(t)}$ and \eqref{eq: reward r}.
			\State {\bf Step 2}:  Test node $i$ with probability $\min\{1, \frac{B(t)\hat{r}_i(t)}{\sum_{j\in\mathcal{V}(t)}\hat{r}_j(t)}\}$. After that, randomly select $c(t)$ \big(defined in (\ref{eq: compensatoryfactor})\big) further nodes to test (see Footnote~\ref{ft}). Get the new observations $\underline{Y}(t)$.
			\State {\bf Step 3}: Based on $\underline{Y}(t)$, update $\{\underline{u}_i(t+1)\}_{i\in\mathcal{V}(t+1)}$ by Algorithm~\ref{alg: Backward Update} (Step 0 $\sim$ Step 2) in Section~\ref{sec: Posterior Probabilities}.
		\end{algorithmic}
	\end{algorithm}

	\section{Message-Passing Framework}\label{sec: Posterior Probabilities}
	
	As discussed in Section~\ref{sec: The Proposed Policies}, the probabilities $\{\underline{v}_i(t)\}_i$ are unknown.  In this section, we develop a message passing framework to sequentially estimate  $\{\underline{v}_i(t)\}_i$  based on the network observations and the dynamics of the spread process.  We refer to these estimates as $\{\underline{u}_i(t)\}_i$.

	When node $i$ is tested on day $t$, an observation $Y_i(t)$ is provided about its state.  Knowing the state of node $i$ provides two types of information: (i) it provides information about the  state of the neighboring nodes in future time slots $t+1,t+2, \ldots$ (because of the evolution of the spread in time and on the network), and (ii) it also provides information about the past of the spread, meaning that we can infer about the state of the (unobserved) nodes at previous time slots. For example, if node $i$ is tested positive in time $t$, we would know that  (i) its neighbors are more likely to be infected in time $t+1$ and (ii) some of its neighbors must have been infected in a previous time for node $i$ to be infected now. This forms the basis for our backward-forward message passing framework.

	Given the spread model of Section~\ref{sec: Modeling}, we first describe the forward propagation of belief.
	Suppose that at time $t$, the probability vector $\underline{v}_i(t)$ is given for all $i$. The probability vector $\underline{v}_i(t+1)$ can be computed as follows (see Appendix~\ref{App: Local Transition Equations}):
	\begin{align}\label{eq: matrix form}
		\underline{v}_i(t+1) = \underline{v}_i(t)\times{\tt P}_i\big(\{\underline{v}_j(t)\}_{j\in\partial^+_i(t)}\big)
	\end{align}
	where  ${\tt P}_i\big(\{\underline{v}_j(t)\}_{j\in\partial^+_i(t)}\big)$ is
	a local transition probability matrix given in Appendix~\ref{App: Local Transition Equations}. %

	Recall that $\underline{Y}(t)$ denotes the collection of network observations on day $t$. The history of observations is then denoted by  $\{\underline{Y}(\tau)\}_{\tau=1}^{t-1}$. 
	Based on these observations, 
	we wish to find an estimate of the probability vector $\underline{v}_i(t)$ for each $i\in\mathcal{V}(t)$. We denote  this estimate by 
	$\underline{u}_i(t) = (u_x^{(i)}(t), x\in\mathcal{X})$ and refer to it, in this section, as the {\it prior probability} of node $i$ in time $t$. We further define 
	the  {\it posterior probability}  $\underline{w}_i(t) = (w_x^{(i)}(t), x\in\mathcal{X})$ of node $i$ in time $t$ (after obtaining new observations $\underline{Y}(t)$). In particular, 
	\begin{align*}
		&u_x^{(i)}(t) = \Pr\big(\sigma_i(t)=x|{\{\underline{Y}(\tau)\}_{\tau=1}^{t-1}}\big)\\
		&w_x^{(i)}(t) = \Pr\big(\sigma_i(t)=x|{\{\underline{Y}(\tau)\}_{\tau=1}^{t}}\big).
	\end{align*}
	Here, the prior probability is defined {\it at the beginning of} every day, and the posterior probability is defined {\it at the end of} every day. Conditioning all probabilities in (\ref{eq: matrix form}) on $\{\underline{Y}(\tau)\}_{\tau=1}^{t}$, we obtain the following  {\it forward-update rule} (see Appendix~\ref{App: Proof two equations})
	\begin{align}\label{eq: estimate matrix form1}
		\underline{u}_i(t+1) = \underline{w}_i(t)\times{\tt P}_i\big(\{\underline{w}_j(t)\}_{j\in\partial^+_i(t)}\big).
	\end{align}
	\begin{remark}\label{remark1}
		Following (\ref{eq: estimate matrix form1}), we need to utilize the observations $\underline{Y}(t)$ and the underlying dependency among nodes' states to update the posterior probabilities $\{\underline{w}_i(t)\}_i$, and consequently update $\{u_i(t+1)\}_i$ based on the forward-update rule (\ref{eq: estimate matrix form1}). This is however non-trivial. A Naive approach would be to locally incorporate node $i$'s observation $Y_i(t)$  into $w_i(t)$ and obtain $\underline{u}_i(t+1)$ using \eqref{eq: estimate matrix form1}. 
		This approach, however, does not fully exploit the observations and it disregards the dependency among nodes' states, as caused by the nature of the spread (An example is provided in Appendix~\ref{App: simpleexample}).
	\end{remark}

	\paragraph{Backward Propagation of Belief}\label{sec: posterior probabilities}
	
	To  capture the dependency of nodes' states and thus best utilize the observations, we proceed as follows. First, denote 
	\begin{align*}
		&\underline{e}_i(t-1) = (e_x^{(i)}(t-1), x\in\mathcal{X})\\ 
		&e_x^{(i)}(t-1) = \Pr\big(\sigma_i(t-1)=x|{\{\underline{Y}(\tau)\}_{\tau=1}^{t}}\big).
	\end{align*}
	Vector $\underline{e}_i(t-1)$ is the posterior probability of node $i$ at time $t-1$, after obtaining the history of observations up to and including time $t$. By computing $\underline{e}_i(t-1)$, we are effectively correcting our belief on the state of the nodes in the previous time slot by \emph{inference} based on the observations acquired at time $t$. This constitutes the {\it backward  step} of our framework and we will expand on it shortly. The backward step can be repeated to correct our belief also in times $t-2$, $t-3$, etc. For clarity of presentation and tractability of our analysis and experiments, we truncate the backward step at time $t-1$ and present assumptions under which this truncation is theoretically justifiable. Considering larger truncation windows is straightforward but out of the scope of this paper.

	Once our belief about nodes' states is updated in prior time slots (e.g., $\underline{e}_i(t-1)$ is obtained), it is propagated forward in time for \emph{prediction} and to provide a more accurate estimate of the nodes' posterior and prior probabilities. More specifically, consider (\ref{eq: matrix form}) written for time $t$ (rather than $t+1$) and condition all probabilities  on $\{\underline{Y}(\tau)\}_{\tau=1}^{t}$. We obtain the following update rule (see Appendix~\ref{App: Proof two equations}):
	\begin{align}\label{eq: estimate matrix form2}
		\underline{w}_i(t) = \underline{e}_i(t-1)\times\tilde{{\tt P}}_i\big(\{\underline{e}_j(t-1)\}_{j\in\partial^+_i(t-1)}\big)
	\end{align}
	where $\tilde{{\tt P}}_i\big(\{\underline{e}_j(t-1)\}_{j\in\partial^+_i(t-1)}\big)$ is given in Appendix~\ref{App: Proof two equations}.
	{\it Note that the local transition matrix in (\ref{eq: estimate matrix form2}) is not the same as  (\ref{eq: estimate matrix form1}). This is because  ``future" observations were available  in $\tilde{{\tt P}}_i\big(\{\underline{e}_j(t-1)\}_{j\in\partial^+_i(t-1)}\big)$.} The probability vectors $\{\underline{e}_i(t-1)\}_i$ provide better estimates for $\{\underline{w}_i(t)\}_i$ through (\ref{eq: estimate matrix form2}) and the prior probabilities $\{\underline{u}_i(t+1)\}_i$ are  then computed using (\ref{eq: estimate matrix form1}) to be used for decision making in time $t+1$. The block diagram in Fig.~\ref{fig: process} depicts the high-level idea of our framework.  It is worth noting that ${\tt P}_i\big(\{\underline{w}_j(t)\}_{j\in\partial^+_i(t)}\big)$ in \eqref{eq: estimate matrix form1} and $\tilde{{\tt P}}_i\big(\{\underline{e}_j(t-1)\}_{j\in\partial^+_i(t-1)}\big)$ in \eqref{eq: estimate matrix form2} both depend on the observations, $\{\underline{Y}(\tau)\}_{\tau=1}^{t}$.

	We next discuss how $\underline{e}_j(t-1)$ can be computed, starting with some notations. Denote by  
	\begin{align}\label{eq: zetathetat}
		\zeta_i(t) = \sigma_i(t)|_{\{\underline{Y}(\tau)\}_{\tau=1}^{t}},\quad \theta_j(t)=\sigma_i(t)|_{\{\underline{Y}(\tau)\}_{\tau=1}^{t-1}},
	\end{align}
	the state of the nodes in the posterior probability spaces conditioned on the observations $\{\underline{Y}(\tau)\}_{\tau=1}^{t}$ and $\{\underline{Y}(\tau)\}_{\tau=1}^{t-1}$, respectively. We further define $\Psi_i(t)$ to be the set of those neighbors of node $i$ at time $t-1$, including node $i$, who are observed/tested at time $t$. This set consists of all nodes whose posterior probabilities will be updated at time $t-1$ (given a new observation $Y_i(t)$).
	The set of all neighbors (except node $i$) of the nodes in $\Psi_i(t)$ then defines $\Phi_i(t)$. The set $\Phi_i(t)$ consists of all nodes whose posterior probabilities at time $t$ is updated by the observation $Y_i(t)$. More precisely, we have 
	\begin{align*}
		&\Psi_i(t) = \mathcal{O}(t)\cap\partial^+_i(t-1),\\
		&\Phi_i(t) = \{j|j\in\partial^+_k(t-1), k\in\Psi_i(t)\}\backslash\{i\},\\
		&\Theta_i(t) = \{j| j\in\partial^+_k(t-1), k\in\mathcal{O}(t)\}\backslash \{i\}
	\end{align*}
	where $\mathcal{O}(t)$ is the set of observed nodes at time $t$ (see Figure~\ref{fig: phipsi}). 
	\begin{figure}[t!]
		\centering
		\includegraphics[width=.6\textwidth]{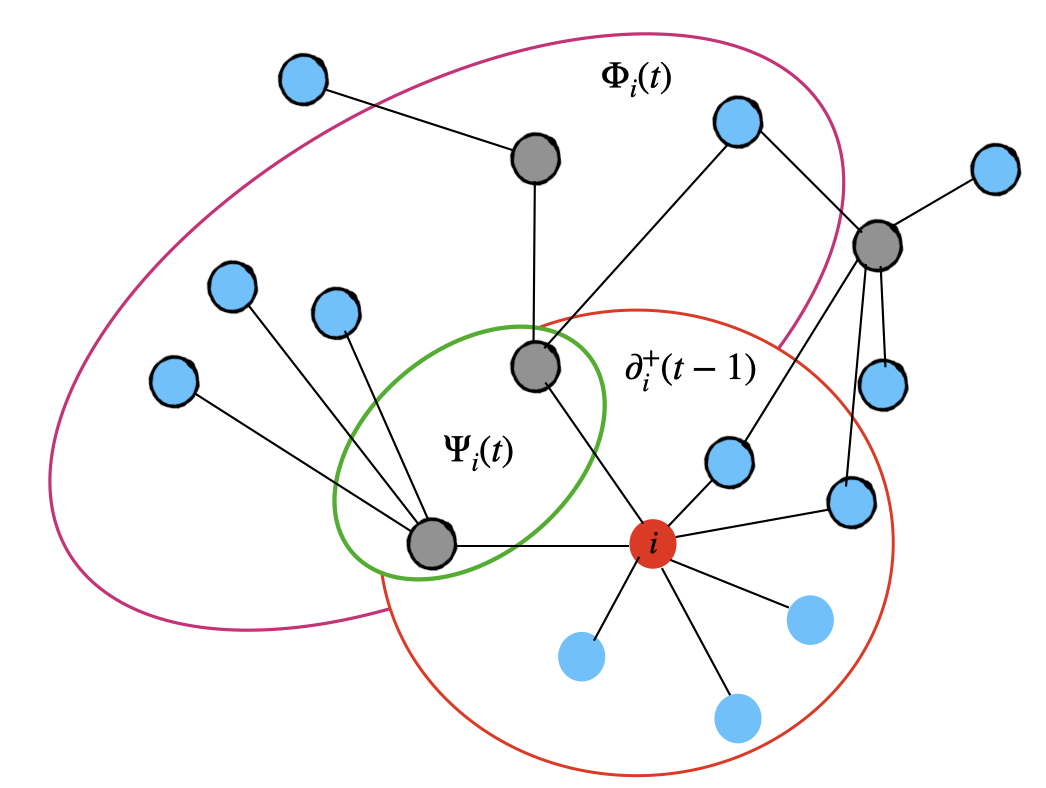} 
		\caption{An example of $\Psi_i(t)$, $\Phi_i(t)$ and $\Theta_i(t)$. Node~$i$ is marked in red, and its neighborhood  $\partial^+_i(t-1)$ is shown by the red contour. Suppose that the gray nodes are tested on day $t-1$, then $\Psi_i(t)$ is the set of nodes within the green contour, and  $\Phi_i(t)$ consists of the nodes in the purple contour. Finally, nodes in $\Theta_i(t)$ are marked with bold black border}.
		\label{fig: phipsi}
	\end{figure}
	In Appendix~\ref{App: proof for e}, we show 
	\begin{equation}\label{eq: chain rule of e2}
		\begin{aligned}
			e_x^{(i)}(t-1) \!=\!\frac{\Pr\big(\underline{Y}(t)|\zeta_i(t\!-\!1)\! =\! x\big)\ w^{(i)}_x(t-1)}{\Pr\big(\underline{Y}(t)\big)}.
		\end{aligned}
	\end{equation}
	It suffices to find $\Pr\big(\underline{Y}(t)|\zeta_i(t-1) = x\big)$. The denominator $\Pr\big(\underline{Y}(t)\big)$ is then found by normalization of the enumerator in (\ref{eq: chain rule of e2}). Let  $\{x_j\}_{j\in\mathcal{O}(t)}$ be a realization of $\{\theta_j(t)\}_{j\in\mathcal{O}(t)}$ and  $\{y_l\}_{l\in\Theta_i(t)}$ be a realization of $\{\zeta_l(t-1)\}_{l\in\Theta_i(t)}$.
	We prove the following in Appendix~\ref{App: proof for e} under a simplifying truncation assumption (see Assumption~\ref{assu: independence in terminate} in Appendix~\ref{App: proof for e}) where the backward step is truncated in time $t-1$:
	\begin{equation}\label{eq: fenzi34}
		\begin{aligned}
			&\Pr\big(\underline{Y}(t)|\zeta_i(t-1) = x\big)=\Pr\big(\{Y_j(t)\}_{j\in\Psi_i(t)}|\zeta_i(t-1) = x\big)\\
			&=\sum_{\{x_j\}_{j\in\Psi_i(t)}}\prod_{j\in \Psi_i(t)}\Pr\big(Y_j(t)|\theta_j(t)\big)\\
			&\times\sum_{\{y_l\}_{l\in\Phi_i(t)}}\prod_{j\in\Psi_i(t)}\Pr\big(x_j|\{y_l\}_{l\in\partial_j^+(t-1)\backslash\{i\}},x\big)\times\prod_{l\in\{\Phi_i(t)\}}w_{y_l}^{(i)}(t-1).
		\end{aligned}
	\end{equation}
	
	We finally present our {\it Backward-Forward Algorithm} to sequentially compute estimates  $\{\underline{u}_i(t)\}_i$  in Algorithm \ref{alg: Backward Update}. The process of Algorithm~\ref{alg: Backward Update} is given in Fig~\ref{fig: process}, and we also give a simple example to show the process of Algorithm~\ref{alg: Backward Update} in Appendix~\ref{App: simpleexample}. 
	\begin{figure}[t!]
		\centering 
		\includegraphics[width=.6\textwidth]{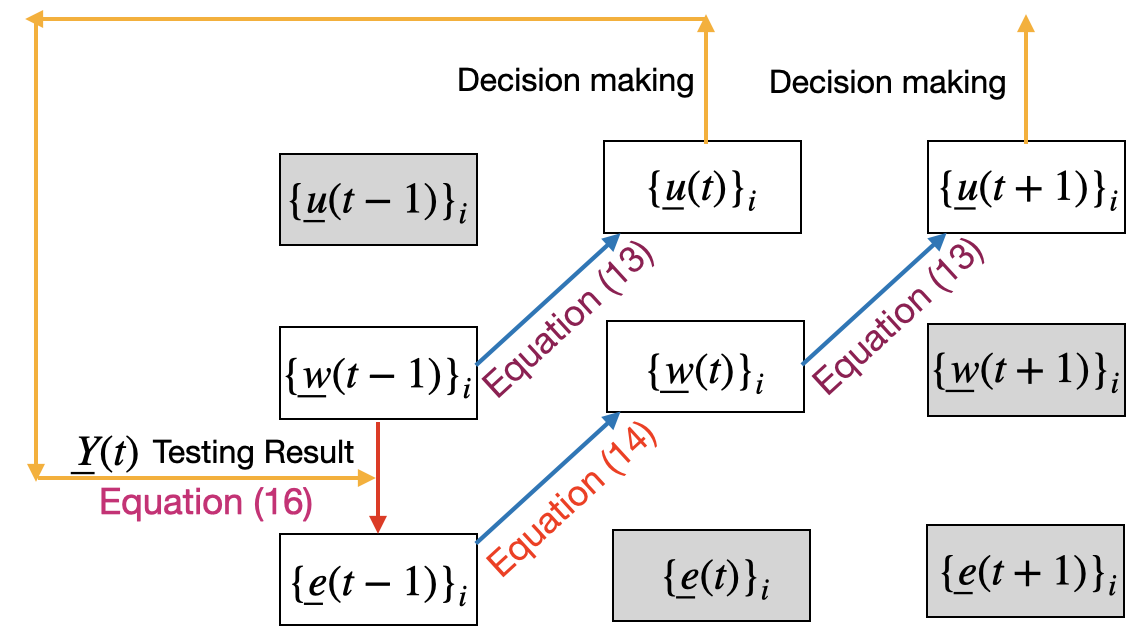} 
		\caption{The process of Algorithm~\ref{alg: Backward Update}. One example of a complete process is given in unshaded blocks. Recall that $\underline{u}_i(\tau)$, $\underline{w}_i(\tau)$, and $\underline{e}_i(\tau)$, where $\tau\in\{t-1, t, t+1\}$, are the prior probabilities, the posterior probabilities, and the updated posterior probabilities, respectively.}
		\label{fig: process}
	\end{figure}
	\begin{algorithm}
		\caption{Backward-Forward Algorithm}\label{alg: Backward Update}
		\begin{algorithmic}			
			\State Input $\underline{Y}(0)$,  $\{\underline{e}_i(0)\}_{i\in\mathcal{V}(0)}$, $\{\underline{w}_i(0)\}_{i\in\mathcal{V}(0)}$, $\{\underline{u}_i(0)\}_{i\in\mathcal{V}(0)}$.
			
			\noindent{\bf Repeat} for $t=1,2,\cdots,T-1$
			\State {\bf Step 0}: Based on \underline{Y}(t), get $\mathcal{V}(t)$ from $\mathcal{V}(t-1)$.
			\State {\bf Step 1}: Backward step. Update $\underline{e}_i(t-1)$ by (\ref{eq: chain rule of e2}), (\ref{eq: fenzi34}), and then compute $\underline{w}_i(t)$ by~ (\ref{eq: estimate matrix form2}). 
			\State {\bf Step 2}: Forward step. Compute $\underline{u}_i(t+1)$ by~(\ref{eq: estimate matrix form1}).
		\end{algorithmic}
	\end{algorithm}

	\subsection{Necessity of Backward Updating}\label{section:backward}
	Now we provide an example which illustrates the necessity of backward updating. 
	\begin{example}\label{ex: Necessity of Backward Updating}
		Consider a line network with the node set $\mathcal{V} = \{1,2,\ldots,N\}$ and the edge set $\mathcal{E}=\{(i, i+1), 1 \leq i \leq N-1$\} (see Figure~\ref{fig: exampleline}). On the initial day, we assume that each node is infected independently with probability $1/N$. Let $\beta=1$, $\lambda=0$,  $\gamma=0$\footnote{Here, $\lambda=0$ implies there is no latent state, and $\gamma=0$ implies that nodes never recover.}, and $B(t)=1$. We further assume that there is no isolation when a positive node is tested.
	\end{example}
	\begin{figure}[t!]
		\centering
		\includegraphics[width=.8\textwidth]{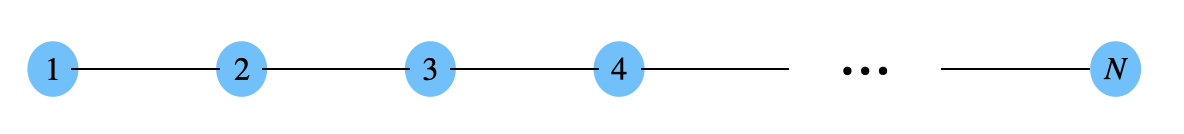} 
		\caption{The line network in Example~\ref{ex: Necessity of Backward Updating}.}
		\label{fig: exampleline}
	\end{figure}
	Based on Example~\ref{ex: Necessity of Backward Updating}, we show that  the naive approach of Remark~\ref{remark1} (i.e., forward-only updating) will cause the estimated probabilities to never converge to the true probabilities of infection. Nonetheless, if we use the Backward-Forward Algorithm~\ref{alg: Backward Update}, the estimated probabilities  converge to the true probabilities after a certain number of steps. Formally, we prove the following result in Appendix~\ref{example:backward}.
	
	\begin{theorem}\label{theorem:backward}
		For any testing policy that sequentially computes $\{\underline{u}_i(t)\}_{i}$  based on~(\ref{eq: estimate matrix form1}) (see Remark~\ref{remark1}), with probability (approximately) $\frac{1}{e}$, we have $\sum_{i=1}^{N}||\underline{v}_i(t) - \underline{u}_i(t) || \overset{t\to\infty}{\rightarrow} \Theta(N)$,\  for large $N$ \footnote{ Theorem~\ref{theorem:backward} holds for all kinds of noem due to the equivalence of norms. In addition, the convergence is topological convergence.}.  On the other hand, there exists a testing policy that sequentially updates $\{\underline{u}_i(t)\}_{i}$  based on Algorithm~\ref{alg: Backward Update} and attains $\sum_{i=1}^{N}||\underline{v}_i(t) - \underline{u}_i(t) || = 0,\  t \geq 2N.$  
	\end{theorem}
	{\it Roadmap of proof: Consider a simple case where every node is susceptible. Since each node is infected with probability $1/N$, then the case occurs with probability $\simeq 1/e$.
		
		Under the case above, consider any testing  policy based on  the algorithm in Remark~\ref{remark1} . If a node is tested on day $t$, then the policy ``clears'' the tested node. Since the updating rule of the algorithm can not go back to the information on day $t-1$, then it can not ``clear'' any neighbors of the tested node and its probability of infection updates to a non-zero value in the next day.  Furthermore, we show that almost all nodes have an significantly large probability of infection when time horizon is sufficiently large, hence  $\sum_{i=1}^{N}||\underline{v}_i(t) - \underline{u}_i(t) || \overset{t\to\infty}{\rightarrow} \Theta(N)$.		
		
		On the other hand, we can propose a specific testing policy. Note that there is no infection, if Algorithm~\ref{alg: Backward Update} is used to update probabilities, then it can reveal the states of all nodes under the specific testing policy after at most $2N$ days. So we have $\sum_{i=1}^{N}||\underline{v}_i(t) - \underline{u}_i(t) || = 0,\  t \geq 2N.$ 
	}

	In Theorem~\ref{theorem:backward}, we illustrate the necessity of backward updating when  testing is limited. In essence, we want to ``clear'' the graph and confirm that there are no infections. If the number of tests is limited, we have mathematically shown that no algorithm can correctly estimate the nodes' infection probabilities if it does not use the backward (inference) step. On the contrary, there is an algorithm that uses the backward step along with the forward step and the estimates that it provides for the nodes' infection probabilities converge  to the true probabilities of the nodes after some finite steps. 
	Even though the considered graph is simple but the phenomena it captures is general. 
	
	As discussed in Theorem~\ref{theorem:backward}, the backward updating is necessary. However, bacward updating can be computationally expensive in large dense graphs. To trade off the impact of backward updating and the reduction of computation complexity, we propose an $\alpha$-linking backward updating algorithm in Appendix~\ref{App: Reduction of Complexity}, where Algorithm~\ref{alg: Backward Update} is applied on a random subgraph with fewer edges.

	\subsection{Necessity of Exploration}\label{sec: Exploiration and Exploration} 
	Note that in reality we have no information for $\{\underline{v}_i(t)\}_i$, and only have the estimates $\{\underline{u}_i(t)\}_i$. One may wonder if exploitation based on wrong initial estimated probability vectors, i.e., $\{\underline{u}_i(0)\}_i$, misleads decision making by providing poorer and poorer estimates of the probabilities of infection. If so, exploration may be necessary. 
	
	\begin{example}\label{ex: effect of prob} Consider $\mathcal{V} = \{1,2,\ldots,N\}$ and edges  $\mathcal{E}=\{(i, i+1), 1 \leq i \leq N-1\}$ (see Figure~\ref{fig: exampleline}). Let $N\gg 10$, $\beta=1$, $\lambda=0$, $\gamma = 0$, and $B(t)=10$. Suppose that on the initial day, node $1$ is infected and all other nodes are susceptible. Consider a wrong initial estimate: $w_{I}^{(i)}(0) = u_{I}^{(i)}(0) = 0$ if $i \leq \frac{9N}{10}$, and $w_{I}^{(i)}(0) = u_{I}^{(i)}(0) = \frac{10\epsilon}{N}$ otherwise, where $\epsilon > 0$. With this  initial belief,  we have $\sum_{i=1}^{N}||\underline{w}_i(0)-\underline{v}_i(0)|| =  O(1 + \epsilon).$ 
	\end{example} 
	Different from Example~\ref{ex: Necessity of Backward Updating}, here we consider the isolation of nodes that are tested positive. In Example~\ref{ex: effect of prob}, suppose that a specific exploration policy is applied: $1$ (out of $10$) tests is done randomly, and the other $9$ tests are done following exploitation. Now, in Appendix~\ref{example:explor}, we show that under the RbEx policy, the cumulative infection is at least $aN$ for a constant $a$, while under the exploration policy defined above,  the cumulative infection is at most $bN$ with very high probability, and the ratio $a/b$ can be any constant for a large enough $N$. More formally, we have the following theorem. Let  $p_0$ be a large probability and consider a large time horizon $T$.  Denote the cumulative infections under the RbEx policy by $C^{RbEx}(T)$ and under the specific exploration policy defined above by $C^{exp}(T)$. We prove the necessity of exploration in the following Theorem.
		\begin{theorem}\label{theorem:explor} With probability $p_0 \geq \frac{99}{100}$, $\frac{C^{RbEx}(T)}{C^{exp}(T)}\geq c(N, p_0)$, where $c(N, p_0)$ is a constant only depending on $N$ and $p_0$.
		\end{theorem} 

	{\it Roadmap of proof:
		Under the RbEx policy, we test nodes based on their predicted  probabilities. Since the nodes that are located towards the end of the line (right side in Fig. \ref{fig: exampleline}) have non-zero probabilities, they are tested first while the disease spreads on the other end of the network (left side in Fig. \ref{fig: exampleline}). Mathematically, suppose that for the first time, an infectious node is tested at day $t=aN$, then there are at least $\min\{aN, N\}$ infectious nodes before the spread can be contained. 
		
		Under the specific exploration policy described above, consider the event that, for the first time, an infectious node is explored on day $t = b'N$ ($b'<a$). 
		We argue that with probability $p_0$, the exploration policy catches at least two new infections at each step after $t = b'N$. After $2t$, the algorithm catches all the infections, and we have at most $2b'N$ infections. Let $b=2b'$. This is an improvement by a factor of at least $\frac{a}{b}$ in comparison to the RbEx strategy. Factor $\frac{a}{b}$ depends on the values of $N$ and $p_0$. 
	}

	In Theorem~\ref{theorem:explor}, we show the necessity of exploration when our initial belief is slightly wrong, i.e., it is slightly biased toward the other end of the network (In general, this could be due to a wrong belief, prior test results, etc). We have formally proved that when the testing capacity is limited, exploration can significantly improve the cumulative infections, i.e., contain the spread. This motivates the design of exploration policies. Even though the setting is simple, the phenomena it captures is much more general.

	\section{Simulations} \label{sec: experiments}
	\subsection{Overview}\label{sec: overview}
	In this section, we use simulations
	to study the performance of the proposed exploitation and exploration policies for various synthetic and real-data networks. Towards this end, we define some metrics that quantify how different metrics perform and key network parameters and attributes that determine the values of these metrics and thereby how exploitation and exploration compare. We also identify benchmark policies which represent the extreme ends of the tradeoff between exploration and exploitation to compare with the policies we propose and assess the performance enhancements brought about by judicious combinations of exploration and exploitation. 
	Through our experiments, we aim to answer two main questions for various synthetic and real-data networks: (i) Can exploration policies do better that exploitation policies and if so, when would that be the case? (ii) What parameters would affect the performance of exploration and exploitation policies? These are important questions to shed light on the role of exploration. These questions are particularly raised by  Theorem~\ref{theorem:explor} in which we prove that exploration can significantly outperform exploitation in some (stylized) networks. We design the experiments in order to shed light on the above questions and to understand  the extent of the necessity of exploration in different network models and scenarios.

	\paragraph{Network parameters}
	We consider the following parameters:   (i) The unregulated delay $\ell$ which is the  time from the initial start of the spread to the first time testing and intervention starts; (ii) The {\it (global) clustering coefficient} \cite[Chapter~3]{HANSEN202031}, denoted by $\gamma_c$, which is defined as a measure of the degree to which nodes in a graph tend to cluster together; (iii) The {\it path-length}, denoted by $L_p$, which measures the average shortest distance between every possible pair of nodes.  We consider attributes such as the initialization of the process, and the lack of knowledge about $\{\underline{v}_i(t)\}_{i}$.

	\paragraph{Performance metrics}
	
	We consider the expected number of infected nodes in a time horizon $[0, T]$ as the performance measure for various policies. Let $C^0(T)$ be the number of infected nodes if there is no testing and isolation, $C^{RbEx}(T), C^{REEr}(T)$ be the corresponding numbers respectively for the ${RbEx}$ policy (Algorithm~\ref{alg: reward-based exploitation}) and the ${REEr}$ policy (Algorithm~\ref{alg: reward-based exploration}). We consider a ratio between the expectations of these: 
	\begin{align}
		{\tt Ratio} =& \frac{\mathbb{E}[C^{RbEx}(T)] - \mathbb{E}[C^{REEr}(T)]}{\mathbb{E}[C^0(T)]}.\label{eq: ratio prob-based}
	\end{align}
	
	We define the {\it estimation error} ${\tt Err}_\pi(t)$ towards capturing the impact of the lack of knowledge about $\{\underline{v}_i(t)\}_{i}$. 
	\begin{align}\label{eq: estimation error}
	{\tt Err}_\pi(t) = \frac{1}{N(t)}\sum_{i\in\mathcal{G}(t)}||\underline{v}_i(t)-\underline{u}_i(t)||_2^2.
	\end{align}
	We consider the difference between the estimation errors of ${RbEx}$ and ${REEr}$ policies:   $\Delta_{\tt Err} = {\tt Err}_{RbEx}(T) - {\tt Err}_{REEr}(T)$.

	\paragraph{Benchmark policies} 
	We will compare the proposed policies with $4$ benchmark policies. (i) (Forward) Contact Tracing: we tested every day the nodes who have infectious neighbors (in a forward manner), denoted by {\it candidate nodes}. Only some candidate nodes are selected randomly due to testing resources being limited. Note that only exploitation is utilized under this benchmark.  (ii) Random Testing: Every day, we randomly select nodes to test. Typical testing policies that could come out of SIR optimal control formulations for our problem would naturally reduce to random testing as they treat all nodes to be statistically identical and ignore the impact of  network topology. One can interpret that  random testing implements exploration to its full extent.  (iii) Contact Tracing with Active Case Finding: A small portion of (for example, $5\%$) testing budget is utilized for active case finding \cite{OU2020}. This portion of the testing budget is used to test nodes by Random Testing. The remaining budget is utilized for forward contact tracing. (iv) Logistic Regression: We use ideas presented in \cite{GH2020}, where simple classifiers were proposed based on the features of real data. In our setting, we choose the classifier to be based on logistic regression, and we define the feature of node $i$ as $X_i(t) = [1, n_i(t)+\epsilon]^T$. Here, $n_i(t)$ is the number of quarantined neighbors node $i$ has contacted before and including day $t$, and $\epsilon\neq0$ is a superparameter aiming to avoid the case where $n_i(t)=0$. In simulations, we set $\epsilon=0.1$. Let the observation $Y_i(t)$ be the testing result of node $i$. In particular, if node $i$ is not tested on day $t$, then we do {\it not} collect the data $(X_i(t), Y_i(t))$. Thus, the probability of node $i$ being infectious is defined as the {\it Sigmoid} function
	\begin{align*}
		\frac{1}{1+\exp(-X_i(t)\cdot w^T)},
	\end{align*}
	where $w$ is the parameter which should be learned.

	\paragraph{Simulation Setting}\label{sec: Simulation setting}
	We consider a process as described in Section \ref{sec: Modeling} with $n_0$ randomly located initial infected nodes. The process evolved  without any testing/intervention for $\ell$ days and we refer to $\ell$ as the {\it unregulated delay}. After that, one of the (initial) infectious nodes, denoted by node~$i_0$, is (randomly) provided to the policies. Subsequently, the initial estimated probability vector is set to $\underline{u}_{i_0}(\ell) = (1,0,0,0)$, and $\underline{u}_i(\ell) = (0,0,0,1)$ when $i\neq i_0$.
	We consider the budget to be equal to the expected number of infected nodes at time $t$, i.e., $B(t) = \sum_{j=1}^{N(t)}v_I^{(j)}(t).$
	
	We choose model parameters considering the particular application of COVID-19 spread. In particular, 1) the mean latency period is $1/\lambda = 1\,\,\text{or}\,\,2$ days \cite{MSZY2020}; 2) the mean duration in the infectious state (I) is $1/\gamma=7\sim 14$ days \cite{MSZY2020, ABSJ2020,  contact}; 3) we choose the transmission rate $\beta$ in a specific network such that after a long time horizon, if no testing and isolation policies were applied, then around $60\sim90$ percent individuals are infected. We did not consider the case where $100$ percent individuals are infected because given the recovery rate (and the topology), the spread may not reach every node.
	
We consider both synthetic networks such as Watts-Strogatz (WS) networks \cite{DWSS1998}, Scale-free (SF) networks \cite{ADBAC2019},  Stochastic Block Models (SBM) \cite{CLDW2019} and a variant of it (V-SBM), as well as real-data networks. Descriptions and further results for the synthetic networks and real networks are presented in Appendix~\ref{sec: Further Results and Real-Data Networks}.
	
	\paragraph{Watts-Strogatz  Networks.}
	We consider a network WS$(N, d, \delta)$ with $N$ nodes, degree $d$, and rewiring probability $\delta$. The  transmission probability of the spread is set to  $\beta=0.4$ and the number of initial seed is $n_0=3$. 
	
		\paragraph{Scale-free Networks.} We consider a  network SF$(N, \alpha)$ with $N$ nodes, and the fraction of nodes with degree $k$ follows a power law $k^{-\alpha}$, where $\alpha=2.1, 2.3, 2.5, 2.7, 2.9$. The  transmission probability of the spread is set to  $\beta=0.5$ and the number of initial seeds is $n_0=3$.

		\paragraph{Stochastic Block Models.} 	The SBM is a generative model for random graphs. The graph is divided into several communities, and subsets of nodes are characterized by being connected with particular edge densities. The intra-connection probability is $p_1$, and the inter-connection probability is $p_2$. We denote the SBM as SBM$(N, M, p_1, p_2)$\footnote{Here, we assume that $M$ is an exact divisor of $N$.}. The  transmission probability of the spread is set to  $\beta=0.04$ and the number of initial seed is $n_0=3$.  The construction of SBM is given in Appendix~\ref{App: simulations SBM V-SBM}.
		
		\paragraph{A Variant of Stochastic Block Models.}
		Different from SBM, we only allow nodes in cluster~$i$ to connect to nodes in successive clusters (the neighbor clusters). Denote a variant of SBM as V-SBM$(N, M, p_1, p_2)$.    The  transmission probability of the spread is set to  $\beta=0.04$ and the number of initial seed is $n_0=3$.  The construction of V-SBM is given in Appendix~\ref{App: simulations SBM V-SBM}.
		
		\paragraph{Real-data Network I.} We consider a contact network of  university students in the Copenhagen Networks Study \cite{PSASDDLSL2019}. The network is built based on the proximity between participating students recorded by smartphones, at 5 minute resolution. According to the definition of { close contact} by \cite{contact}, we only used proximity events between individuals that lasted more than 15 minutes to construct the daily  contact network. The  contact network has $672$ individuals spanning $28$ days. To guarantee a long time-horizon, we replicate the contact network $4$ times so that the  time-horizon is $112$ days.  We set $\beta=0.05$ and $n_0=5$ to have a realistic simulation of the Covid-19 spread. Note that the network is relatively dense, so we choose a relatively small value of $\beta$  to avoid the unrealistic case in which the disease spreads very fast (see Figure~\ref{realdatanetwork12}~(left)).
		
		\paragraph{Real-data Network II.} We consider a publicly available dataset on human social interactions collected specifically for modeling infectious disease dynamics \cite{SKPK2020, PKSK2018, JAFJH2020}. The data set consists of pairwise distances between users of the BBC Pandemic Haslemere app over time. The  contact network has $469$ individuals spanning $576$ days.  Since the network is very sparse, then we compress contacts among individuals during $4$ successive days to one day. Then, we have $469$ individuals spanning $144$ days. We set $\beta=0.95$ and $n_0=30$ to have a realistic simulation of the Covid-19 spread. Note that the network is relatively sparse, so we choose a relatively large value of $\beta$  to avoid the unrealistic case in which the disease spreads very slow (see Figure~\ref{realdatanetwork12}~(left)).

\subsection{Simulation Results in Synthetic networks}

In this section, we compare the performances of our proposed policies and the benckmarks (defined in Section~\ref{sec: overview}) in synthetic networks. We start with some specific networks and parameters for this purpose (see  Figure~\ref{WS_5},  Figure~\ref{powerlaw25}, Figure~\ref{SBM_1}, and Figure~\ref{VSBM_1}).   The figures reveal that our proposed policies, i.e., the RbEx and REEr policies, outperform the benchmarks. In particular, in Figure~\ref{WS_5} and Figure~\ref{powerlaw25} (i.e., the WS and SF networks), the REEr policy outperforms the RbEx policy, and the REEr policy provides a more accurate estimation for $\{\underline{v}_i(t)\}_i$. In  Figure~\ref{SBM_1} and Figure~\ref{VSBM_1}  (i.e., the SBM and V-SBM networks), the RbEx policy outperforms the REEr policy, and the RbEx policy provides a more accurate estimation for $\{\underline{v}_i(t)\}_i$. In addition, in Figure~\ref{WS_5}, we show that Algorithm~\ref{alg: Greedy Algorithm}  outperforms the RbEx policy but performs worse than the REEr policy (recall that the compuation time of Algorithm~\ref{alg: Greedy Algorithm} is high, we therefore only plot the performance of Algorithm~\ref{alg: Greedy Algorithm} in Figure~\ref{WS_5} as an example). This implies that without exploration, the exploitation in a greedy manner can not perform well in WS networks.

From the discussions above, the advantages of exploration in distinct settings (different network topologies with variant parameters) are different. To investigate the advantages of exploration in distinct settings, it suffices to show how the main parameters affect the exploration. In this work, we consider three main parameters which are defined in  Section~\ref{sec: overview}, i.e., the unregulated delay $\ell$, the global clustering coefficient $\gamma_c$, and the path-length $L_p$. Detailed discussions are later given in Section~\ref{sec: Impact of Network Parameters}.

	\begin{figure}
		\centering
		\includegraphics[width=0.9\textwidth]{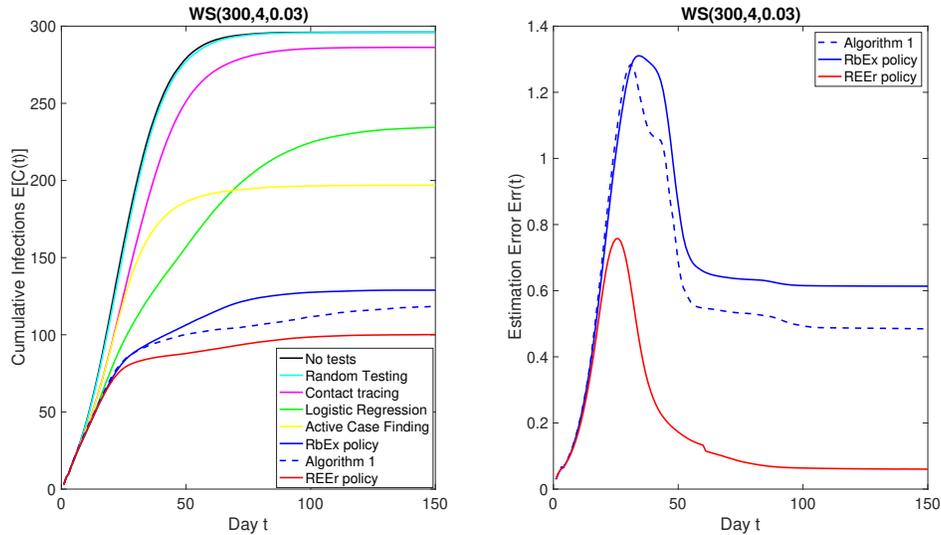}
		\caption{Performances and estimation errors of different policies in WS(300, 4, 0.03) when $\ell=3$.}
		\label{WS_5}
	\end{figure}

\begin{figure}
	\centering
	\includegraphics[width=0.9\textwidth]{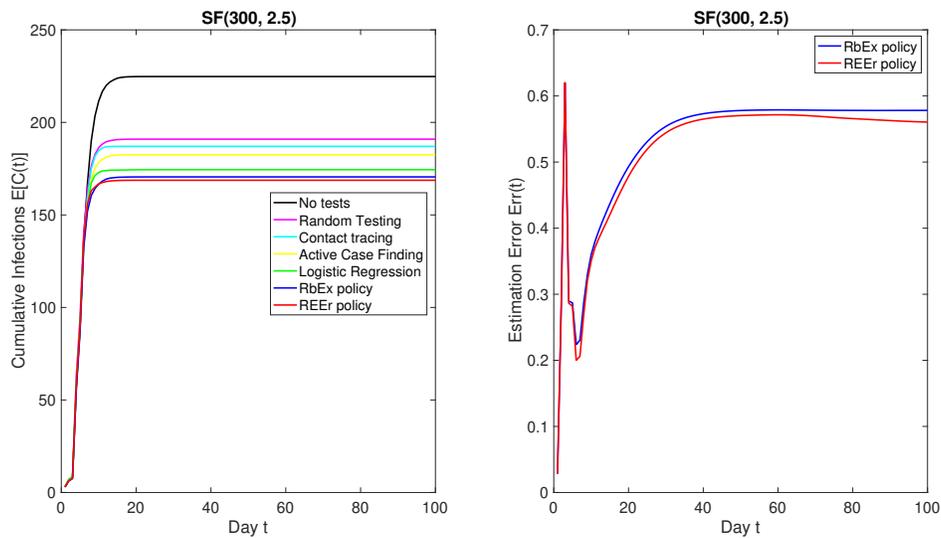}
	\caption{Performances and estimation errors of different policies in SF(300, 2.5) when $\ell=3$.}
	\label{powerlaw25}
\end{figure}

	\begin{figure}
	\centering
	\includegraphics[width=0.9\textwidth]{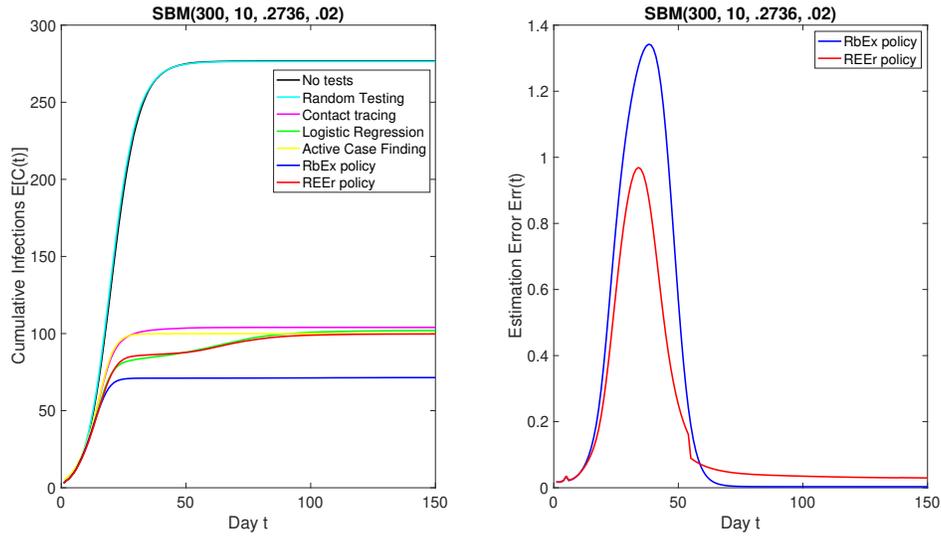}
	\caption{Performances and estimation errors of different policies in SBM(300, 10, .2736, .02) when $\ell=5$.}
	\label{SBM_1}
\end{figure}

	\begin{figure}
	\centering
	\includegraphics[width=0.9\textwidth]{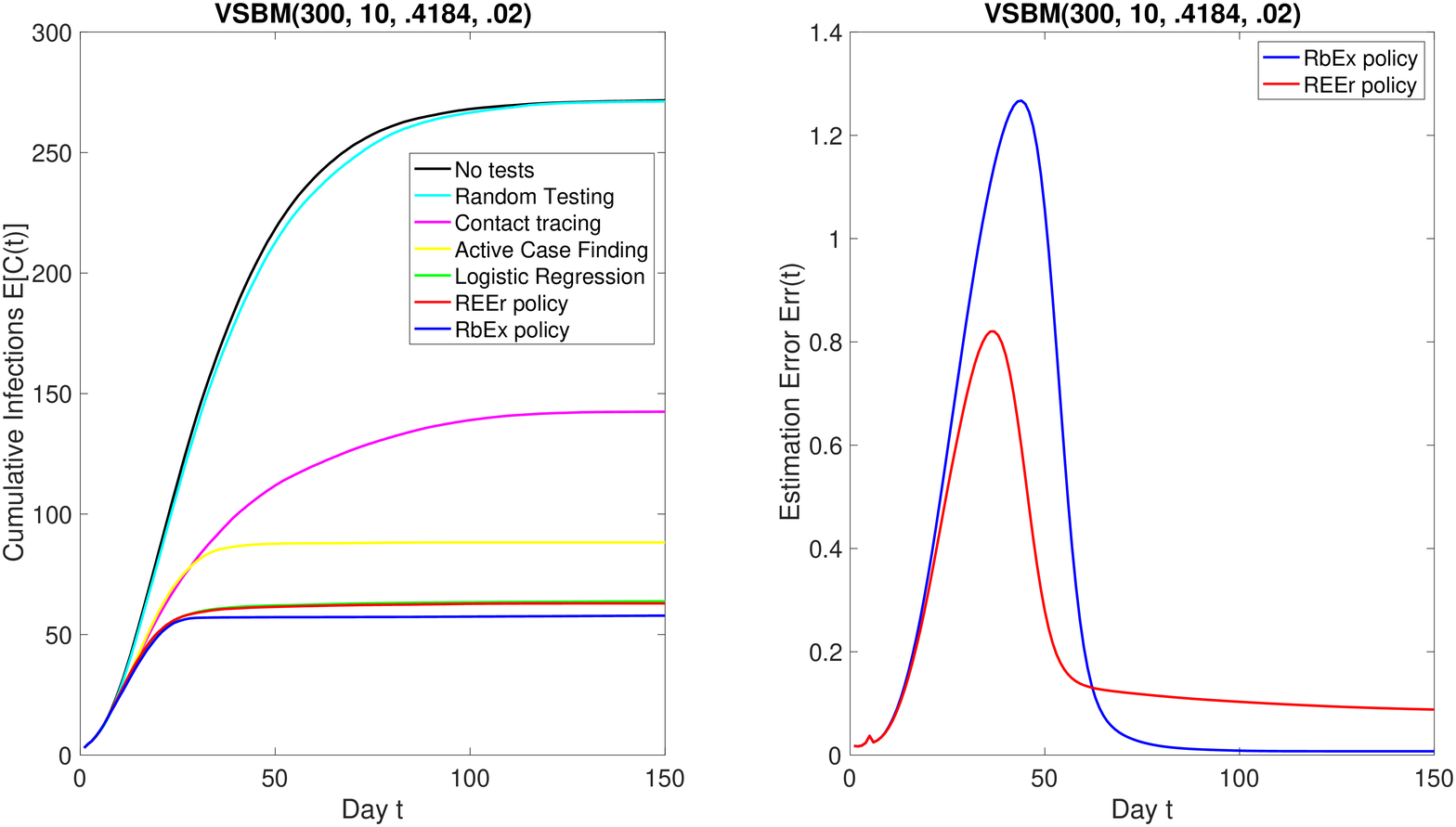}
	\caption{Performances and estimation errors of different policies in VSBM(300, 10, .4184, .02) when $\ell=5$.}
	\label{VSBM_1}
\end{figure}

\subsubsection{Impact of Network Parameters}\label{sec: Impact of Network Parameters}
	
In this subsection, we consider the impact of network parameters on the tradeoff between exploration and exploitation.

\paragraph{Impact of $\ell$.}  We first investigate the impact of the unregulated delay, $\ell$. Specifically, from Table~\ref{tabla: WS ell}, Table~\ref{tabla: SF ell}, Table~\ref{tabla: SBM ell}, and Table~\ref{tabla: VSBM ell}, as $\ell$ increases, so does ${\tt Ratio}$ and $\Delta_{\tt Err}$, implying that exploration becomes more effective.  With increase in $\ell$, the infection continues in the network for longer, there are greater number of infectious nodes in the network and they are scattered throughout the network, thus exploration is better suited to locate them. Thus, the REEr policy can contain the spread of the disease faster.

In particular, the REEr policy is always better in WS networks. This is because  exploitation may confine the tests in neighborhoods of some infected nodes. While in the SBM networks, the RbEx policy always outperforms the REEr policy. In both the SF and V-SBM networks, the RbEx policy is better when $\ell$ is small, and the REEr policy is better when $\ell$ is large. One interesting observation is that in the V-SBM networks, the REEr policy performs better when $\ell$ is large ($= 11, 13$), but the corresponding estimation errors are larger than those in the RbEx policy. In this specific network topology, it appears that smaller estimation error does not always correspond to better cumulative infections. One potential reason is that the REEr policy is sensitive to $\ell$ in this topology, i.e., we can achieve smaller cumulative infections under the REEr policy even if the estimation error is larger.

\begin{table}[!htbp]
	\centering
	\begin{tabular}{|c|c|c|c|c|c|}
		\hline  
		WS, $\ell$&  $3$ & $5$ & $7$ & $9$ & $11$ \\
		\hline
		${\tt Ratio}$ & $0.097$ & $0.128$ & $0.177$ & $0.207$ & $0.297$\\ \hline
		$\Delta_{\tt Err}$ & $0.553$ &  $0.814$ & $1.092$ & $1.197$ & $1.449$\\
		\hline
	\end{tabular}
	\caption{Role of the unregulated delay $\ell$ when $\delta=0.03$.}\label{tabla: WS ell}
\end{table}

\begin{table}[!htbp]
	\centering
	\begin{tabular}{|c|c|c|c|c|c|}
		\hline  
		SF, $\ell$&  $3$ & $5$ & $7$ & $9$ & $11$ \\
		\hline
		${\tt Ratio}$ & $-0.0009$ & $0.0026$ & $0.0033$ & $0.0042$ & $0.0059$\\ \hline
		$\Delta_{\tt Err}$ & $-0.0014$ &  $0.0237$ & $0.0334$ & $0.0434$ & $0.1212$\\
		\hline
	\end{tabular}
	\caption{Role of the unregulated delay $\ell$ when $\alpha=2.1$.}\label{tabla: SF ell}
\end{table}

\begin{table}[!htbp]
	\centering
	\begin{tabular}{|c|c|c|c|c|c|}
		\hline  
		SBM, $\ell$&  $5$ & $7$ & $9$ & $11$ & $13$ \\
		\hline
		${\tt Ratio}$ & $-0.092$ & $-0.079$ & $-0.042$ & $-0.035$ & $-0.025$\\		\hline
		$\Delta_{\tt Err}$ & $-0.026$ &  $-0.015$ & $-0.010$ & $-0.009$ & $-0.009$\\
		\hline
	\end{tabular}
	\caption{Role of the unregulated delay $\ell$ when $(p_1, p_2)=(.274, .02)$.}\label{tabla: SBM ell}
\end{table}

\begin{table}[!htbp]
	\centering
	\begin{tabular}{|c|c|c|c|c|c|}
		\hline  
		V-SBM, $\ell$&  $5$ & $7$ & $9$ & $11$ & $13$ \\
		\hline
		${\tt Ratio}$ & $-0.022$ & $-0.016$ & $-0.007$ & $0.011$ & $0.019$\\		\hline
		$\Delta_{\tt Err}$ & $-0.081$ &  $-0.066$ & $-0.046$ & $-0.033$ & $-0.025$\\
		\hline
	\end{tabular}
	\caption{Role of the unregulated delay $\ell$ when $(p_1, p_2)=(.418, .02)$.}\label{tabla: VSBM ell}
\end{table}

\paragraph{Impact of $\gamma_c$ and $L_p$.}

Then, we investigate the impact of the global clustering coefficient, i.e., $\gamma_c$, and the average shortest path-length, i.e., $L_p$.   In Table~\ref{tabla: WS clustering coefficient},  both $\gamma_c$ and $L_p$ decrease as $\delta$ increases. In  Table~\ref{tabla: SF clustering coefficient}, $\gamma_c$ decreases as $\alpha$ increases. For the SF networks, the graphs are often disconnected, so we only calculate $\gamma_c$ in Table~\ref{tabla: SF clustering coefficient}.  In Table~\ref{tabla: SBM clustering coefficient} and Table~\ref{tabla: VSBM clustering coefficient}, both $\gamma_c$ and $L_p$ decrease as $p_2$ increases.

From these tables,  as $L_p$ or $\gamma_c$ decreases, the benefits of exploration compared to exploitation decrease as well. This confirms the intuition that  exploration is particularly helpful in clustered networks with larger path lengths where undetected infection can spread without any intervention as exploitation largely confines the tests in neighborhoods of the  infections that were previously detected. This is also supported by the fact that exploration lowers estimation error in such scenarios, as  shown in Table~\ref{tabla: WS clustering coefficient}, Table~\ref{tabla: SF clustering coefficient}, Table~\ref{tabla: SBM clustering coefficient}, and Table~\ref{tabla: VSBM clustering coefficient}. Furthermore, we investigate the role of $\gamma_c$ and $L_p$ individually in Appendix~\ref{App: The impact of parameters individually}.

\vspace{-0.8cm}
\begin{table}
	\centering
	\begin{tabular}{|c|c|c|c|c|}
		\hline  
		WS, $\delta$& $\gamma_c$ &$L_p$ & ${\tt Ratio}$ & $\Delta_{\tt Err}$ \\
		\hline
		$0$ & $.5$&$62.876$ &$0.191$ & $1.153$ \\
		\hline
		$.0075$ & $.489$ & $21.264$ & $0.182$ & $1.423$\\
		\hline  
		 $.015$ & $.473$ & $14.253$  & $0.174$ &$0.991$ \\
		\hline
		$.0225$ & $.467$ &  $12.171$& $0.126$ & $0.779$\\
		\hline
		$.03$  & $.456$ &  $10.81$ & $0.097$ & $0.554$\\
		\hline
	\end{tabular}
	\caption{Role of clustering coefficient and path length when $\ell=3$.}\label{tabla: WS clustering coefficient}
\end{table}

\begin{table}
	\centering
	\begin{tabular}{|c|c|c|c|}
		\hline  
		SF, $\alpha$& $\gamma_c$ & ${\tt Ratio}$ & $\Delta_{\tt Err}$ \\
		\hline
		$2.1$  & $.5017$  & $0.0080$ & $0.0334$\\
		\hline  
		$2.3$ & $.3374$ & $0.0057$ & $0.0253$\\
		\hline
		$2.5$ & $.2348$ & $0.0032$ &$0.0177$ \\
		\hline
		$2.7$ & $.1496$ & $-0.0019$ & $0.0124$\\
		\hline
		$2.9$ & $.0219$&$-0.0064$ & $0.0081$ \\
		\hline
	\end{tabular}
	\caption{Role of clustering coefficient and path length $\ell=3$.}\label{tabla: SF clustering coefficient}
\end{table}

\begin{table}
	\centering
	\begin{tabular}{|c|c|c|c|c|}
		\hline  
		SBM, $(p_1,p_2)$& $\gamma_c$& $L_p$ & ${\tt Ratio}$ & $\Delta_{\tt Err}$  \\
		\hline
		$(0.274, 0.02)$& $0.111$ & $2.573$ & $-0.092$ & $-0.026$\\
		\hline
		$(0.214, 0.026)$ & $0.075$ & $2.518$ & $-0.103$ & $-0.023$\\
		\hline  		
		$(0.159, 0.032)$ & $0.056$ & $2.492$ & $-0.113$ & $-0.026$\\
		\hline
		$(0.102, 0.039)$& $0.048$ & $2.480$ & $-0.118$ & $-0.023$\\
		\hline
		$(0.045, 0.045)$ & $0.043$ & $2.455$ & $-0.124$ & $-0.027$\\
		\hline
	\end{tabular}
	\caption{Role of clustering coefficient and path length $\ell=5$.}\label{tabla: SBM clustering coefficient}
\end{table}

\begin{table}
	\centering
	\begin{tabular}{|c|c|c|c|c|}
		\hline  
		V-SBM, $(p_1,p_2)$& $\gamma_c$ &$L_p$ & ${\tt Ratio}$ & $\Delta_{\tt Err}$\\
		\hline
		$(0.418,0.020)$& $0.3557$ & $4.4264$ & $-0.022$ & $-0.081$\\
		\hline
		$(0.351, 0.052)$ & $0.2365$ & $3.6584$ &$-0.091$ & $-0.045$\\
		\hline  		
		$(0.284, 0.085)$ &$0.1769$ & $3.307$ & $-0.104$ & $-0.055$ \\
		\hline
		$(0.217, 0.085)$& $0.1385$ & $3.1562$ & $-0.112$ & $-0.041$\\
		\hline
		$(0.150,0.0150)$ & $0.1170$ & $3.0563$ & $-0.123$ & $-0.042$\\
		\hline
	\end{tabular}
	\caption{Role of clustering coefficient and path length $\ell=5$.}\label{tabla: VSBM clustering coefficient}
\end{table}

\subsection{Simulation Results in Real-data Networks}

\begin{figure}
\centering
\includegraphics[width=0.9\textwidth]{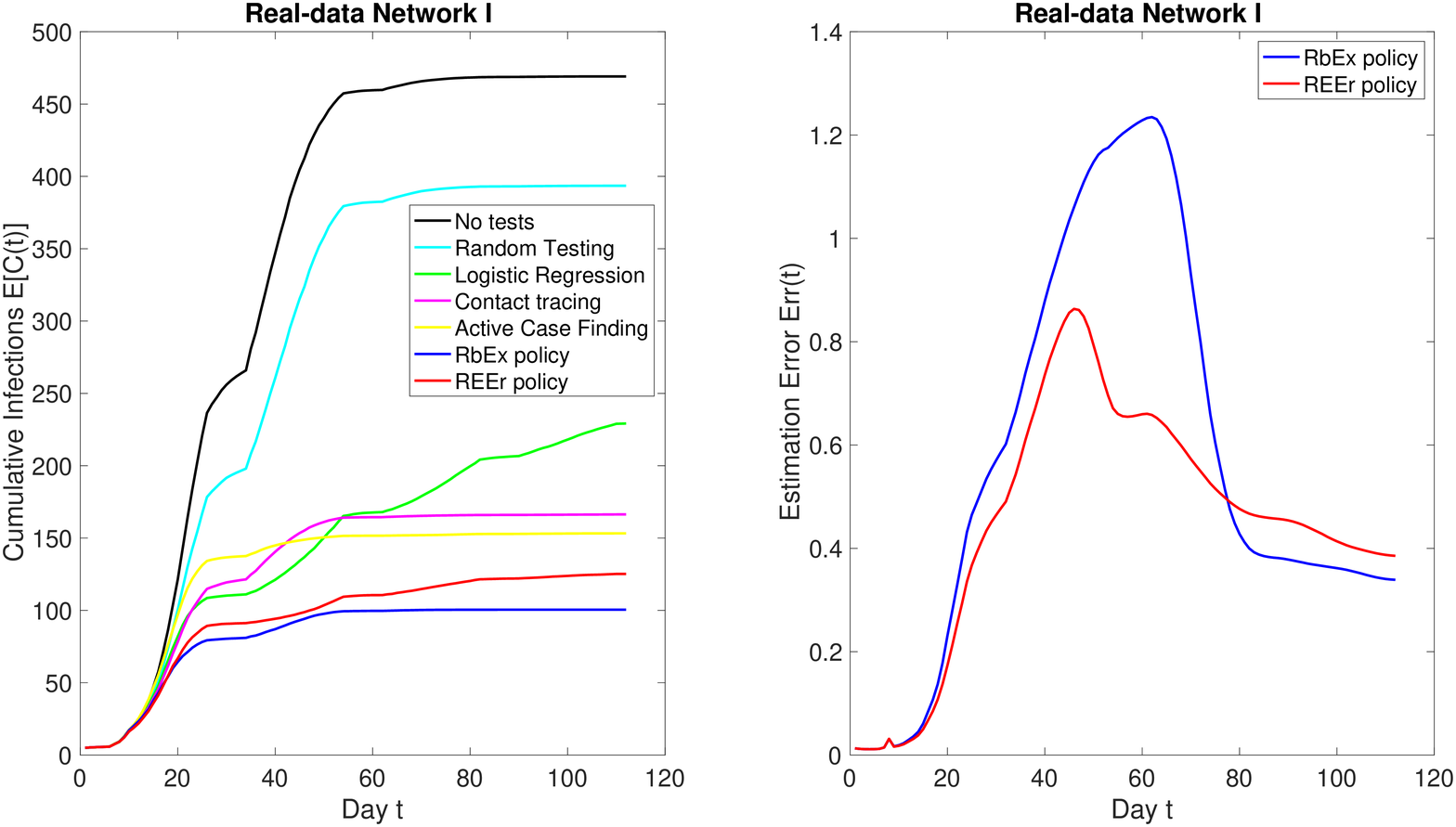}
\caption{Performances and estimation errors of different policies in the real-data network I when $\ell=8$.}
\label{realdata}
	\end{figure}
	
		\begin{table}
		\centering
		\begin{tabular}{|c|c|c|c|}
			\hline  
			Real-data Network I, $\ell$& $5$ & $8$ & $11$  \\
			\hline
			${\tt Ratio}$ & $-0.0559$ & $-0.0255$ & $0.009$\\
			\hline   
			$\Delta_{\tt Err}$ & $-0.061$ & $-0.030$ & $0.035$  \\
			\hline
		\end{tabular}
		\caption{Role of the unregulated delay $\ell$}\label{tabla: real data 1}
	\end{table}

\begin{figure}
	\centering
	\includegraphics[width=0.9\textwidth]{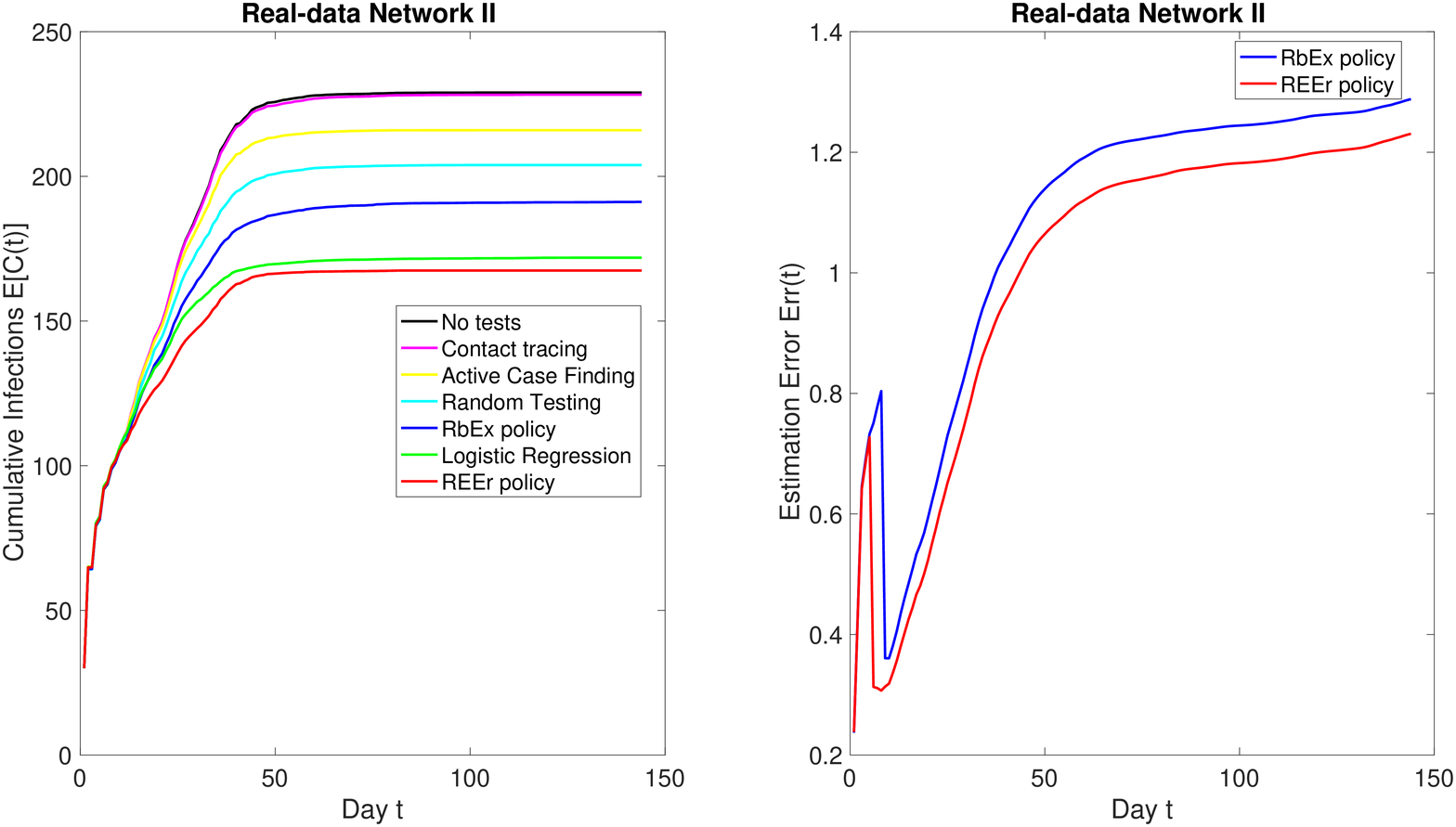}
	\caption{Performances and estimation errors of different policies in the real-data network~II when $\ell=8$.}
	\label{realdata2}
\end{figure}

	\begin{table}
		\centering
		\begin{tabular}{|c|c|c|c|}
			\hline  
			Real-data Network II, $\ell$& $5$ & $8$ & $11$  \\
			\hline
			${\tt Ratio}$ & $0.0808$ & $0.1039$ & $0.1208$\\
			\hline   
			$\Delta_{\tt Err}$ & $0.0317$ & $0.0535$ & $0.0615$  \\
			\hline
		\end{tabular}
		\caption{Role of the unregulated delay $\ell$}\label{tabla: real data 2}
	\end{table}
	
 In this section, we verify our proposed policies in real data networks (Real-data Network I and Real-data Network II). In Figure~\ref{realdata}, our proposed policies outperform the baselines, and the RbEx policy outperforms the RREr policy. In Figure~\ref{realdata2}, the REEr policy can contain the spread and outperform other baselines and RbEx, while the Logistic Regression policies outperforms RbEx. Comparing Figure~\ref{realdata} and Figure~\ref{realdata2}, we find that the RbEx policy performs well in Real-data Network I (better than the REEr policy), but performs not well in Real-data Network II (much worse than the REEr policy). In Figure~\ref{realdatanetwork12}~(left), we calculate the average edges per node on every day, and in Figure~\ref{realdatanetwork12}~(right), we calculate the number of components on every day. From Figure~\ref{realdatanetwork12}~(left), the Real-data Network I is denser than the Real-data Network II. However, from Figure~\ref{realdatanetwork12}~(right), the Real-data Network II often has more components (subgraphs) than the Real-data Network I. Thus, exploitation may become confined within some components (subgraphs), and fail to locate infectious nodes elsewhere, and exploration becomes more effective in presence of a large number of components. This explains the relative performances of REEr and RBEx in these. Contact tracing policy employs only exploitation, while active case finding policy uses most of its test budget for exploitation (and the small amount of the residual test budget for exploration). From Figure~\ref{realdata} and Figure~\ref{realdata2}, the contact tracing  and the active case finding policies perform relatively poorly  in the Real-data Network II compared to that in the Real-data Network I; this may again be attributed to the presence of a large number of components in the former.

\begin{figure}[t!]
	\centering
	\includegraphics[width=0.44\textwidth]{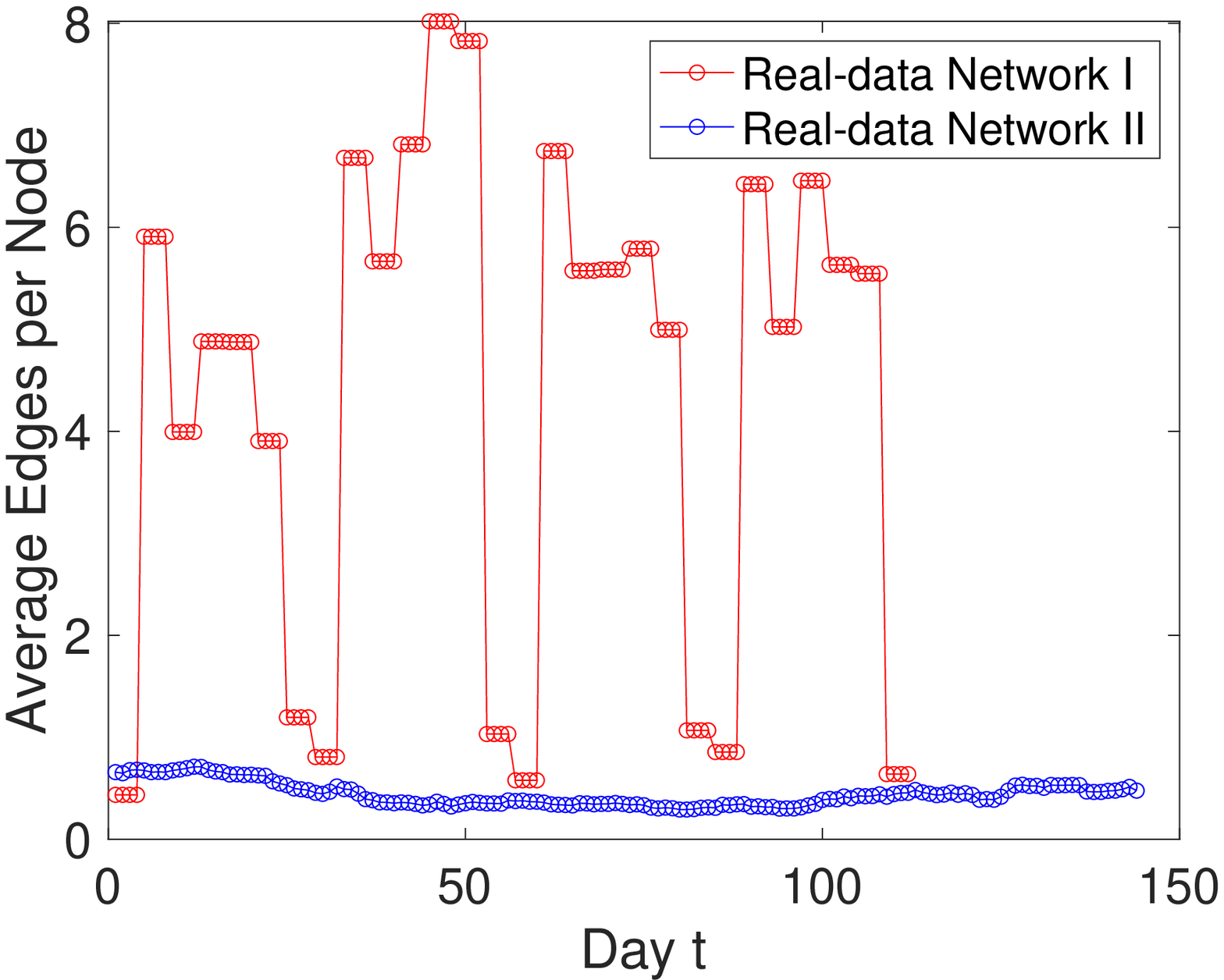}\quad
	\includegraphics[width=0.44\textwidth]{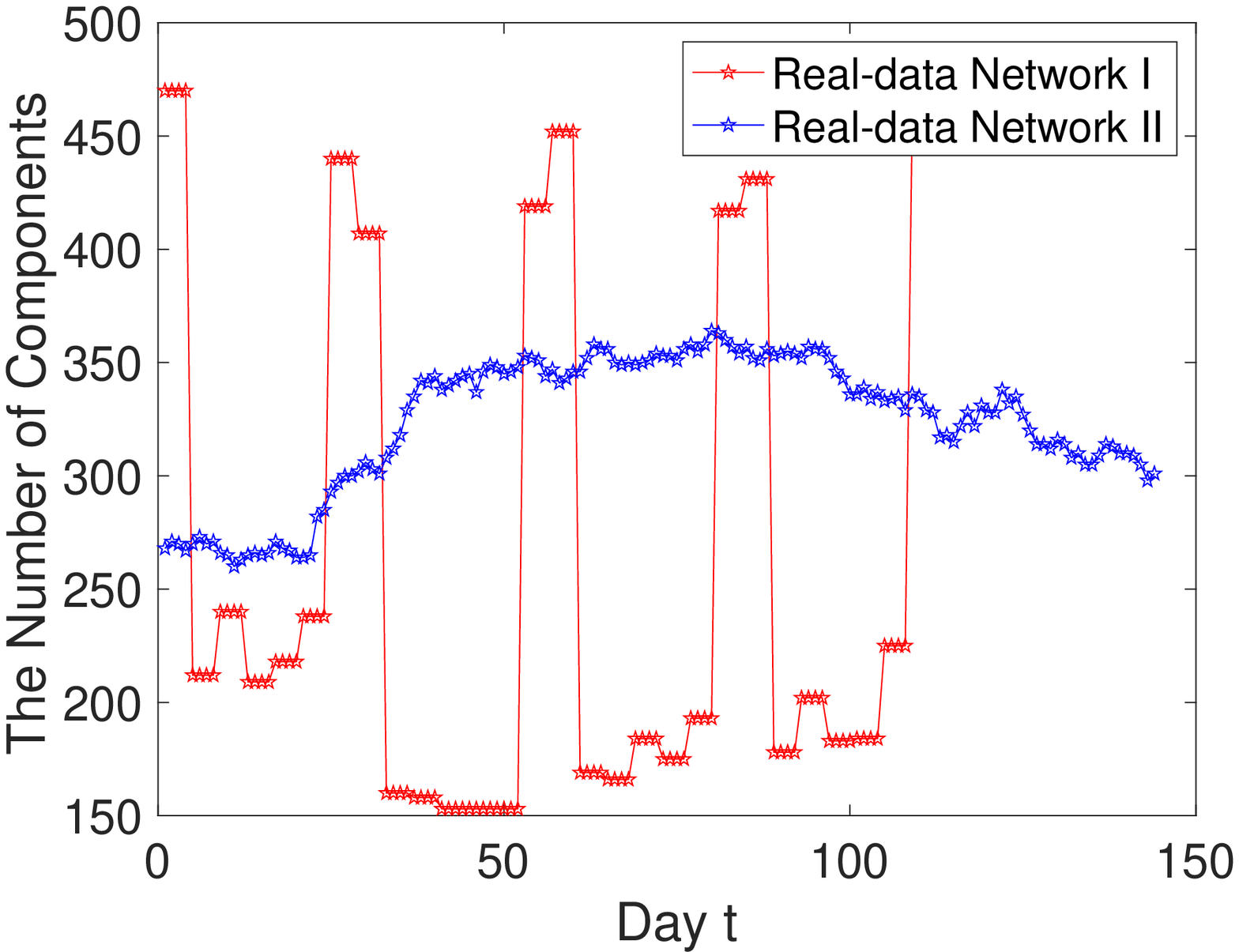}   
	\caption{Left: The average number of edges per node on each day. Right: The number of components on each day.}
	\label{realdatanetwork12}
\end{figure}

As $\ell$ increases, as we show in Table~\ref{tabla: real data 1} and Table~\ref{tabla: real data 2} that the benefit of exploitation decreases. In Table~\ref{tabla: real data 2}, because of a large number of components, exploration always outperforms exploitation. However, in Table~\ref{tabla: real data 1}, we observe that exploration outperforms exploitation only for larger  values of $\ell$. Our results are thus consistent with synthetic networks.

	\section{Conclusions and Future Work}\label{sec: conclusions}
	In this paper, we studied the problem of containing a spread process (e.g. an infectious disease such as COVID-19) through sequential testing and isolation. We modeled the spread process by a compartmental model that evolves in time and stochastically spreads over a given contact network.  Given a daily test budget, we aimed to minimize the cumulative infections. Under mild conditions, we proved that the problem can be cast as  minimizing  a supermodular function expressed in terms of nodes' probabilities of infection and proposed a greedy testing policy that attains a constant factor approximation ratio. We subsequently designed a computationally tractable reward-based policy that preferentially tests nodes that have  higher rewards, where the reward of a node is defined as the expected number of new infections it induces in the next time slot. We  showed that this policy effectively minimizes an upper bound on the cumulative infections. 
	
	These policies, however, need knowledge about  nodes' infection probabilities which are unknown and evolving. Thus, they  have to be actively learned by testing. We discussed how testing has a dual role in this problem: (i) identifying the infected nodes and isolating them in order to contain the spread, and (ii) providing better estimates for the nodes' infection probabilities. We proved that this dual role of testing makes decision making more challenging. In particular, we showed that  reward based policies that make decisions based on nodes' estimated infection probabilities can be arbitrarily sub-optimal while incorporating simple forms of exploration can boost their performance by a constant factor. Motivated by this finding, we devised exploration policies that probabilistically test nodes according to their rewards and  numerically showed that when (i) the unregulated delay, (ii) the global clustering coefficient, or  (iii) the average shortest path length increase, exploration becomes more beneficial  as it provides better estimates of the nodes' probabilities of infection.  
	
	Given the history of observations, computing nodes' estimated probabilities of infection  is itself a core challenge in our problem. We developed a message-passing framework to estimate these probabilities utilizing the observations in form of the test results. This framework passes messages back and forth in time to iteratively  predict the probabilities in future and correct the errors in the estimates in prior time instants. This framework can also be of independent interest.
	
	We showed novel tradeoffs between exploration and exploitation, different from the ones commonly observed in multi-armed bandit settings: (i) in our setting, the number of arms is time-variant and actions cannot be repeated; (ii) the tradeoffs in our setting are not due to lack of knowledge about the network or the process model, but rather due to lack of knowledge about the  time-evolving unknown set of infected nodes.
	
	We now describe directions for future research. 
	
	Our framework can be extended to incorporate delay and/or error in test results in a relatively straightforward manner (an outline of the extension incorporating a delay is given in  Appendix~\ref{App: Delay of Testing Results}), but generalizing the  performance guarantees for the proposed policies in these cases forms a direction of future research. This includes establishing fundamental lower bounds using genie-aided myopic policies. 
	
\subsection{Impact Statements}\label{sec: limitations}
		
We have made several assumptions for the purpose of analytical and computational tractability which do not hold in practice: (1) the infections from different nodes are independent (2) given the entire history of testing results the states of nodes on the truncation day are independent (Assumption~\ref{assu: independence in terminate}), (3) the symptoms need not be considered in deciding who should be tested and (4) the public health authority knows the entire network topology and uses it to determine who should be tested (5) independence of states of nodes (in one step). The first two assumptions were used to derive the message passing framework and to prove that the objective function is super-modular which in turn led to a myopic testing strategy which is also optimal. The first assumption is reasonable as specific actions of infected individuals, eg, coughing, touching,  spread the infection, which are undertaken independently.

		We now consider the second assumption, ie, Assumption~\ref{assu: independence in terminate}, in which  we assume that the nodes' states $\zeta(t-g)$ (in the posterior probability space on day $t-g$)  are independent. Note that $g$ is the truncation time for each backward step, that is, once we get the observations $\underline{Y}(t)$, we do the backward step and truncate at time $t-g$.  This assumption does not impose independence on the state of the nodes, but only in the posterior space at a specific time. That is, in the process of propagating information back to time $t-g$, we are assuming that there is no further correlation between time $t-g-1$ and time $t-g$ worthwhile to exploit given observations at time $t$.  Naturally, as $g$ gets larger and larger, our framework and calculations  become more precise, as the impact of the testing results  at time $t$ in inferring about the nodes' probabilities at time $t-g$ vanishes as $g$ gets large. But increase in $g$ significantly increases the computation time. Therefore, for computational tractability, of the backward update equations, we use $g=1$. In principle the derivations of the backward update equations can be generalized in a straightforward manner to $g > 1.$ But designing approximation strategies that ensure computational tractability  for larger $g$ constitutes a direction of future research. 
		
		Consider the third assumption. We have not considered symptoms in determining who to test. But for some infectious diseases, symptoms are a reliable manifestation of the disease (e.g., Ebola). In principle our testing framework can be generalized in a straightforward manner to consider symptoms by introducing additional states in the compartmental model for evolution of the disease. But introduction of additional states significantly increases the computation time, for example of the forward and backward updates of the probabilities that individuals have the disease, which renders implementation of our framework challenging. Considering symptoms while retaining computational tractability constitutes a direction of future research.

Next, consider the fourth assumption. In practice, public health authorities will not typically know contact networks in their entirety particularly when they are large, for example, as in large cities. However, small network topologies, for example, contact networks within a community, may be observed by the public health authority. As a specific example, the Government of China fully detected contact networks in many communities in Wuhan and tracked paths traversed by every individual \cite{YXLN2020}. This tracking may also generate concerns about privacy which is beyond the scope of this paper. Nonetheless, the technology for learning contact networks in their entirety for small communities exists and our framework can be utilized for those. Generalizing our framework to obtain approximation guarantees when contact networks can only be partially observed constitutes a direction of future research.

Finally consider the last assumption. Note that it is a strong assumption and clearly does not hold in general  but it has been resorted to for only one step in the entire framework. Specifically to obtain Equation (\ref{eq: fEF}) we have assumed that the state of the nodes are independent. This allows us to obtain a simple expression in (\ref{eq: fEF}) in terms of the infection probabilities. We do not use this independence assumption in the rest of the paper.

\subsection*{Acknowledgments} This work was supported by NSF CAREER Award 2047482,  NSF Award 1909186, NSF Award 1910594, and NSF Award 2008284.
	
	\bibliographystyle{unsrt}
	\bibliography{tmlr}

\begin{thebibliography}{10}

\bibitem{BSLSZA1998}
{B. Shulgin, L. Stone and Z. Agur}.
\newblock {Pulse vaccination strategy in the SIR epidemic model}.
\newblock {\em {Bulletin of Mathematical Biology}}, 60:1123 -- 1148, 1998.

\bibitem{PTJC1996}
{P. Tapaswi and J. Chattopadhyay}.
\newblock {Global stability results of a
  ''susceptible-infective-immune-susceptible'' (SIRS) epidemic model}.
\newblock {\em {Ecological Modelling}}, 87(223 - 226), 1996.

\bibitem{LSBSZA2000}
{L.Stone, B.Shulgin and Z.Agur}.
\newblock {Theoretical examination of the pulse vaccination policy in the SIR
  epidemic model}.
\newblock {\em {Mathematical and Computer Modelling}}, 31:207 -- 215, 2000.

\bibitem{YTWMEB2000}
{Y. Takeuchi, W. Ma and E. Beretta}.
\newblock {Global asymptotic properties of a delay SIR epidemic model with
  finite incubation times}.
\newblock {\em {Nonlinear Analysis: Theory, Methods \& Applications}}, 42:931
  -- 947, 2000.

\bibitem{JLA1988}
{J. Aron}.
\newblock {Acquired immunity dependent upon exposure in an SIRS epidemic
  model}.
\newblock {\em {Mathematical Biosciences}}, 88:37 -- 47, 1988.

\bibitem{LJSA1994}
{L. Allen}.
\newblock {Some discrete-time SI, SIR, and SIS epidemic models}.
\newblock {\em {Mathematical Biosciences}}, 124:83 -- 105, 1994.

\bibitem{AMMAB2021}
{A. M. Ramos, M. R. Ferrandez, M. Vela-Perez and et al}.
\newblock {A simple but complex enough $\theta$-SIR type model to be used with
  COVID-19 real data. Application to the case of Italy}.
\newblock {\em {Physica D}}, 421-132839, 2021.

\bibitem{AGG2020}
{A. G. M. Neves and G. Guerrero}.
\newblock {Predicting the evolution of the COVID-19 epidemic with the A-SIR
  model: Lombardy, Italy and Sao Paulo state, Brazil}.
\newblock {\em {Physica D}}, {413-132693}, 2020.

\bibitem{ASRPSN2020}
{A. Simha, R. Prasad and S. Narayana}.
\newblock {A simple Stochastic SIR model for COVID-19 Infection Dynamics for
  Karnataka after interventions -- Learning from European Trends}.
\newblock {arXiv: 2003.11920}, 2020.

\bibitem{BMNLT2020}
{B. Ndiaye, L. Tendeng and D. Seck}.
\newblock {Analysis of the COVID-19 pandemic by SIR model and machine learning
  technics for forecasting}.
\newblock {arXiv: 2004.01574}, 2020.

\bibitem{ZJS2020}
{J. Zhu, P. Ge, C. Jiang and et al}.
\newblock {Deep‐learning artificial intelligence analysis of clinical
  variables predicts mortality in COVID‐19 patients}.
\newblock {\em {Journal of the American College of Emergency Physicians Open}},
  1(6):1364--1373, 2020.

\bibitem{CMRKBKM2020}
{C. Mahanty, R. Kumar, B. K. Mishra, and et al}.
\newblock {Prediction of COVID-19 active cases using exponential and non-linear
  growth models}.
\newblock {\em {Expert Systems}}, 39(3), 2020.

\bibitem{EBP2020}
{E. B. Postnikov}.
\newblock {Estimation of COVID-19 dynamics ``on a back-of-envelope'': Does the
  simplest SIR model provide quantitative parameters and predictions?}
\newblock {\em {Chaos, Solitons and Fractals}}, volume = {135-109841}, 2020.

\bibitem{BMNLTDS2020}
{B. Ndiaye, L. Tendeng and D. Seck}.
\newblock {Comparative prediction of confirmed cases with COVID-19 pandemic by
  machine learning, deterministic and stochastic SIR models}.
\newblock {arXiv: 2004.13489}, 2020.

\bibitem{IRAGPAFC2021}
{I. Rahimi, A. H. Gandomi, P. G. Asteris and et al}.
\newblock {Analysis and Prediction of COVID-19 Using SIR, SEIQR, and Machine
  Learning Models: Australia, Italy, and UK Cases}.
\newblock {\em {Information}}, 12(109), 2021.

\bibitem{GHJG2021}
{G. Hu and J. Geng}.
\newblock {Heterogeneity learning for SIRS model: an application to the
  COVID-19}.
\newblock {\em {Statistics and Its Interface}}, 14:73 -- 81, 2021.

\bibitem{RVLFRG2022}
{R. Vega, L. Flores and R. Greiner}.
\newblock {SIMLR: Machine Learning inside the SIR Model for COVID-19
  Forecasting}.
\newblock {\em {Forecasting}}, 4(1):72 -- 94, 2022.

\bibitem{HBKDVG2021}
{H. Bastani, K. Drakopoulos, V. Gupta and et al}.
\newblock {Efficient and targeted COVID-19 border testing via reinforcement
  learning}.
\newblock {\em {Nature}}, 599:108 -- 113, 2021.

\bibitem{SAAMMK2020}
{S. A. Alanazi, M. M. Kamruzzaman, M. Alruwaili and et al}.
\newblock {Measuring and Preventing COVID-19 Using the SIR Model and Machine
  Learning in Smart Health Care}.
\newblock {\em {Journal of Healthcare Engineering}}, {2020-8857346}, 2020.

\bibitem{GPDSPSL2020}
{G. Perakis, D. Singhvi, O. S. Lami, and et al}.
\newblock {COVID-19: A multiwave SIR-based model for learning waves}.
\newblock {\em {Production and Operations Management}}, (13681), 2022.

\bibitem{SCSRIC2021}
{S. Chowdhury, S. Roychowdhury and I. Chaudhuri}.
\newblock {Universality and herd immunity threshold : Revisiting the SIR model
  for COVID-19}.
\newblock {\em {International Journal of Modern Physics C}}, 3(6), 2021.

\bibitem{WCES2021}
{W. Choi and E. Shim}.
\newblock {Optimal strategies for social distancing and testing to control
  COVID-19}.
\newblock {\em {Journal of Theoretical Biology}}, 512(110568), 2021.

\bibitem{AD2021}
{D. Acemoglu, A. Fallah, A. Giometto and et al}.
\newblock {Optimal adaptive testing for epidemic control: combining molecular
  and serology tests}.
\newblock {arXiv:2101.00773}, 2021.

\bibitem{ALGB2020}
{L. Abraham, G. Becigneul and B. Scholkopf}.
\newblock {Crackovid: Optimizing Group Testing}.
\newblock {arXiv:2005.06413}, 2020.

\bibitem{tsay2020modeling}
{C. Tsay, F. Lejarza, M. Stadtherr and et al}.
\newblock {Modeling, state estimation, and optimal control for the US COVID-19
  outbreak}.
\newblock {\em {Scientific reports}}, 10(10711), 2020.

\bibitem{piguillem2020optimal}
{F. Piguillem and L. Shi}.
\newblock {Optimal COVID-19 quarantine and testing policies}.
\newblock {\em {Nature Communications}}, 12(356), 2021.

\bibitem{kuga2021pair}
M.~Tanaka K.~Kuga and J.~Tanimoto.
\newblock Pair approximation model for the vaccination game: predicting the
  dynamic process of epidemic spread and individual actions against contagion.
\newblock {\em Proceedings of the Royal Society A}, 477(2246):20200769, 2021.

\bibitem{kabir2020impact}
K.~Kuga K.~Kabir and J.~Tanimoto.
\newblock The impact of information spreading on epidemic vaccination game
  dynamics in a heterogeneous complex network-a theoretical approach.
\newblock {\em Chaos, Solitons \& Fractals}, 132:109548, 2020.

\bibitem{kabir2019evolutionary}
K.~Kabir and J.~Tanimoto.
\newblock Evolutionary vaccination game approach in metapopulation migration
  model with information spreading on different graphs.
\newblock {\em Chaos, Solitons \& Fractals}, 120:41--55, 2019.

\bibitem{LWNH2021}
{L. Willem, S. Abrams, P. J. K. Libin and et al}.
\newblock {The impact of contact tracing and household bubbles on deconfinement
  strategies for COVID-19}.
\newblock {\em {Nature Communications}}, 12(1524), 2021.

\bibitem{JKXCSSBSS}
{J. Kim, X. Chen, H. Nikpey and et al}.
\newblock {Tracing and testing multiple generations of contacts to COVID-19
  cases: cost-benefit tradeoffs}.
\newblock {\em {Royal Society Open Science}}, 9(10):1 -- 20, 2022.

\bibitem{aleta2020modelling}
{A. Aleta, D. Martin-Corral, A. Piontti and et al}.
\newblock {Modelling the impact of testing, contact tracing and household
  quarantine on second waves of COVID-19}.
\newblock {\em {Nature Human Behaviour}}, 4:964--971, 2020.

\bibitem{hellewell2020feasibility}
{J. Hellewell, S. Abbott, A. Gimma and et al}.
\newblock {Feasibility of controlling COVID-19 outbreaks by isolation of cases
  and contacts}.
\newblock {\em {The Lancet Global Health}}, 8(4), 2020.

\bibitem{kucharski2020effectiveness}
{A. Kucharski, P. Klepac, A. Conlan and et al}.
\newblock {Effectiveness of isolation, testing, contact tracing, and physical
  distancing on reducing transmission of SARS-CoV-2 in different settings: a
  mathematical modelling study}.
\newblock {\em {The Lancet Infectious Diseases}}, 20(10), 2020.

\bibitem{KSHDLM2021}
{S. Kojaku, L. Hebert-Dufresne, E. Mones, and et al}.
\newblock {The effectiveness of backward contact tracing in networks}.
\newblock {\em {Nature Physics}}, 17:652 -- 658, 2021.

\bibitem{Lacent2020}
{A. J. Kucharski, A. J. K. Conlan, S. M. Kissler, etc}.
\newblock {Effectiveness of isolation, testing, contact tracing, and physical
  distancing on reducing transmission of SARS-CoV-2 in different settings: a
  mathematical modelling study}.
\newblock {\em {The Lancet Infectious Diseases}}, 20(10):1151 -- 1160, 2020.

\bibitem{APMCJGM2020}
{A. Perrault, M. Charpignon, J. Gruber, etc}.
\newblock {Designing Efficient Contact Tracing Through Risk-Based
  Quarantining}.
\newblock {Working Paper}, {National Bureau of Economic Research}, Nov. 2020.

\bibitem{OU2020}
{H. Ou, A. Sinha, S. Suen, etc}.
\newblock {Who and when to screen: Multi-round active screening for network
  recurrent infectious diseases under uncertainty}.
\newblock In {\em {Proceedings of 19th International Conference on Autonomous
  Agents and Multiagent Systems (AAMAS)}}, 2020.

\bibitem{PANCBPF2002}
{P. Auer, N. Cesa-Bianchi and P. Fischer}.
\newblock {Finite-time Analysis of the Multiarmed Bandit Problem}.
\newblock {\em {Machine Learning}}, 47:235--256, 2002.

\bibitem{SANG2011}
{S. Agrawal and N. Goyal}.
\newblock {Analysis of Thompson sampling for the multi-armed bandit problem}.
\newblock In {\em {Proceedings of the 25th Annual Conference on Learning
  Theory}}, volume~23, pages 1--26, 2012.

\bibitem{SBNCB2012}
{S. Agrawal and N. Goyal}.
\newblock {Regret analysis of stochastic and nonstochastic multi-armed bandit
  problems}.
\newblock {\em {Foundations and Trends in Machine Learning}}, 5(1):1--122,
  2012.

\bibitem{KMTM2019}
{K. Madhama and T. Murata}.
\newblock {A multi-armed bandit approch for exploring partially observed
  networks}.
\newblock {\em {Applied Network Science}}, 4(26):1--18, 2019.

\bibitem{MBLMLG2010}
{M. Bilgic, L. Mihalkova and L. Getoor}.
\newblock {Active learning for networked data}.
\newblock In {\em {Proceedings of the 27th International Conference on
  International Conference on Machine Learning}}, pages 79--86, 2010.

\bibitem{XWRGJS2013}
{X. Wang and R. Garnett and J. Schneider}.
\newblock {Active search on graphs}.
\newblock In {\em {Proceedings of the 19th ACM SIGKDD international conference
  on Knowledge discovery and data mining}}, pages 731--738, 2013.

\bibitem{YMTKHJS2015}
{Y. Ma and T. K. Huang and J. Schneider}.
\newblock {Active search and bandits on graphs using sigma-optimality}.
\newblock In {\em {Proceedings of the Thirty-First Conference on Uncertainty in
  Artificial Intelligence}}, pages 542 -- 551, 2015.

\bibitem{RGYKDWJS2011}
{R. Garnett, Y. Krishnamurthy, D. Wang and et al}.
\newblock {Bayesian optimal active search on graphs}.
\newblock In {\em {Proceedings of the 29th International Coference on
  International Conference on Machine Learning}}, pages 843--850, 2011.

\bibitem{DZDCXT2018}
{D. Zhao, J. Liu, R. Wu and et al}.
\newblock {Data-Efficient Reinforcement Learning Using Active Exploration
  Method}.
\newblock In {\em {International Conference on Neural Information Processing}},
  pages 265--276, 2018.

\bibitem{BY18}
{Y. Burda, H. Edwards, A. Storkey and et al}.
\newblock {Exploration by random network distillation}.
\newblock In {\em {International Conference on Learning Representations}},
  2019.

\bibitem{BM16}
{M. Bellemare, S.Srinivasan, G. Ostrovski and et al}.
\newblock {Unifying count-based exploration and intrinsic motivation}.
\newblock In {\em {Proceedings of the 30th International Conference on Neural
  Information Processing Systems}}, pages 1479--1487, 2016.

\bibitem{RFNS2020}
{R. Singh, F. Liu and N. B. Shroff}.
\newblock {A Partially Observable MDP Approach for Sequential Testing for
  Infectious Diseases such as COVID-19}.
\newblock {arXiv:2007.13023}, 2020.

\bibitem{GH2020}
{H. Grushka-Cohen, R. Cohen, B. Shapira and et al}.
\newblock {A framework for optimizing COVID-19 testing policy using a Multi
  Armed Bandit approach}.
\newblock {arXiv:2007.14805}, 2020.

\bibitem{EMHMSMGC2021}
{E. Meirom, H. Maron, S. Mannor, and G. Chechik}.
\newblock {Controlling Graph Dynamics with Reinforcement Learning and Graph
  Neural Networks}.
\newblock In {\em Proceedings of the 38th International Conference on Machine
  Learning}, number 139, pages 7565 -- 7577, 2021.

\bibitem{KLM1998}
{L. Kaelbling and M. Littman and A. Cassandra}.
\newblock {Planning and acting in partially observable stochastic domains}.
\newblock {\em {Artificial intelligence}}, 101(1-2):99--134, 1998.

\bibitem{MG82}
{G. Monahan}.
\newblock {State of the Art---A Survey of Partially Observable Markov Decision
  Processes: Theory, Models, and Algorithms}.
\newblock {\em {Management Science}}, 28(1):1--16, 1982.

\bibitem{compartmental}
{G. Walter and M. Contreras}.
\newblock {\em {Compartmental Modeling with Networks}}.
\newblock {Birkhauser, Boston, MA}, 1999.

\bibitem{MSZY2020}
{S. Ma, J. Zhang, M. Zeng and et al.}
\newblock {Epidemiological parameters of coronavirus disease 2019: a pooled
  analysis of publicly reported individual data of 1155 cases from seven
  countries}.
\newblock medRxiv: https://doi.org/10.1101/2020.03.21.20040329, Feb 2020.

\bibitem{ABSJ2020}
{A. Byrne, D. McEvoy, A. Collins and et al}.
\newblock {Inferred duration of infectious period of SARS-CoV-2: rapid scoping
  review and analysis of available evidence for asymptomatic and symptomatic
  COVID-19 cases}.
\newblock {\em {BMJ Open}}, 10, 2020.

\bibitem{contact}
{CDC}.
\newblock {The U.S. Centers for Disease Control and Prevention (CDC)}.
\newblock
  {https://www.cdc.gov/coronavirus/2019-ncov/php/contact579tracing/contact-tracing-plan/contact-tracing.html},
  Sept 2020.

\bibitem{submodularbook}
{D. Topkis}.
\newblock {\em {Supermodularity and Complementarity}}.
\newblock {Princeton University Press}, 1998.

\bibitem{V2001}
{V. Ilev}.
\newblock {An approximation guarantee of the greedy descent algorithm for
  minimizing a supermodular set function}.
\newblock {\em {Discrete Applied Mathematics}}, 114:131--146, 2001.

\bibitem{HANSEN202031}
{D. L. Hansen, B. Shneiderman, M. A. Smith and et al}.
\newblock {\em {Analyzing Social Media Networks with NodeXL (Second Edition)}}.
\newblock {Morgan Kaufmann}, 2020.

\bibitem{DWSS1998}
{D. J. Watts and S. H. Strogatz}.
\newblock {Collective dynamics of small-world networks}.
\newblock {\em {Nature}}, 393(4):440--442, 1998.

\bibitem{ADBAC2019}
{A. D. Broido and A. Clauset}.
\newblock {Scale-free networks are rare}.
\newblock {\em {Nature Communications}}, 10(1017):1 -- 10, 2019.

\bibitem{CLDW2019}
{C. Lee and D. J. Wilkinson}.
\newblock {A review of stochastic block models and extensions for graph
  clustering}.
\newblock {\em {Applied Network Science}}, 4(122), 2019.

\bibitem{PSASDDLSL2019}
{P. Sapiezynski, A. Stopczynski, D. D. Lassen and et al}.
\newblock {Interaction data from the Copenhagen Networks Study}.
\newblock {\em {Scientific Data}}, 6(1):1--10, 2019.

\bibitem{SKPK2020}
{S. M. Kissler, P. Klepac, M. Tang, etc}.
\newblock {Sparking ``The BBC Four Pandemic'': Leveraging citizen science and
  mobile phones to model the spread of disease}.
\newblock bioRxiv https://doi.org/10.1101/479154., 2018.

\bibitem{PKSK2018}
{P. Klepac, S. Kissler and J. Gog}.
\newblock {Contagion! The BBC Four Pandemic -- the model behind the
  documentary.}
\newblock {\em {Epidemics}}, 24:49 -- 59, 2018.

\bibitem{JAFJH2020}
{J. A. Firth, J. Hellewell, P. Klepac, etc}.
\newblock {Using a real-world network to model localized COVID-19 control
  strategies}.
\newblock {\em {Nature Medicine}}, 26:1616 -- 1622, 2020.

\bibitem{YXLN2020}
{X. Yu and N. Li}.
\newblock {How Did Chinese Government Implement Unconventional Measures Against
  COVID-19 Pneumonia.}
\newblock {\em {Risk Manag Healthc Policy}}, 13:491 -- 499, 2020.

\bibitem{RDS2013}
{R. D. Shachter}.
\newblock {Bayes-Ball: The Rational Pastime (for Determining Irrelevance and
  Requisite Information in Belief Networks and Influence Diagrams)}.
\newblock {arXiv: 1301.7412}, 2013.

\bibitem{MJBall}
{M. Jordan}.
\newblock {An Introduction to Probabilistic Graphical Models}.
\newblock {https://people.eecs.berkeley.edu/~jordan/prelims}, 2003.

\end{thebibliography}
	
	\appendix

	\section{Proof of Lemma~\ref{lem: telescopic sum}}\label{App: lemma telescopic sum}
	Note that a node is counted in $C^\pi(t)$ once it has been infected. Then, on day $t+1$, $C^\pi(t+1)$ increases (comparing to $C^\pi(t)$) only because some susceptible nodes are infected by  infectious nodes and are in the latent state for the first time. 
	
	After testing, positive nodes in $\mathcal{K}^\pi(t)$ would not infect others because they are quarantined, and negative nodes would not infect others due to the model assumptions. Hence
	\begin{align*}
		C^\pi(t+1) = C^\pi(t) + \sum_{i\in\mathcal{V}(t)}F_i(\mathcal{V}(t)\backslash\mathcal{K}^\pi(t);t).
	\end{align*}
	Taking the expectation on both sides, we obtain the desired result.

	\section{Proof of Theorem~\ref{thm: supmodular}}\label{App: proof of supmodular}
	To show $S\big(\mathcal{K}^\pi(t); t\big)$ defined in (\ref{eq: supmodular}) is a supermodular function. It suffices to show that for any $\mathcal{A}\subset\mathcal{B}\subset\mathcal{V}(t)$, and for $x\in\mathcal{V}(t)\backslash\mathcal{B}$, we have
	\begin{align}\label{eq: sup1}
		S\big(\mathcal{A}\cup\{x\}; t\big) - S\big(\mathcal{A}; t\big) \leq S\big(\mathcal{B}\cup\{x\}; t\big) - S\big(\mathcal{B}; t\big).
	\end{align}
	Then, it suffices to show for any $i\in\mathcal{V}(t)$,
	\begin{align}\label{eq: sup2}
		\mathbb{E}[F_i\big(\mathcal{A}\cup\{x\};t\big)] - \mathbb{E}[F_i\big(\mathcal{A};t\big)] \leq \mathbb{E}[F_i\big(\mathcal{B}\cup\{x\};t\big)] - \mathbb{E}[F_i\big(\mathcal{B};t\big)].
	\end{align}
	Now, we consider  three cases.
	
	\noindent{\bf Case 1}. If $\mathcal{A}\cap\partial_i(t) = \mathcal{B}\cap\partial_i(t)$, then from (\ref{eq: fEF}), the LHS and RHS in (\ref{eq: sup2}) are exactly the same. Hence, (\ref{eq: sup2}) holds.
	
	\noindent{\bf Case 2}. If $\mathcal{A}\cap\partial_i(t) \subset \mathcal{B}\cap\partial_i(t)$, and $x\notin\partial_i$, then from (\ref{eq: fEF}), $f_i(\mathcal{A}\cup\{x\}) = f_i(\mathcal{A})$ and $f_i(\mathcal{B}\cup\{x\}) = f_i(\mathcal{B})$. Hence  (\ref{eq: sup2}) holds.

	\noindent{\bf Case 3}. If $\mathcal{A}\cap\partial_i(t) \subset \mathcal{B}\cap\partial_i(t)$, and $x\in\partial_i(t)$, let $\mathcal{Y} = \big(\mathcal{B}\cap\partial_i(t)\big)\backslash\big(\mathcal{A}\cap\partial_i(t)\big)$. Here $x\notin\mathcal{Y}$. From (\ref{eq: fEF}), we can compute
	\begin{align*}
		&\mathbb{E}[F_i(\mathcal{A}\cup\{x\};t)] - \mathbb{E}[F_i(\mathcal{A};t)]\\
		&=v_S^{(i)}(t)\prod_{j\in\partial_{i}(t)\backslash(\mathcal{A}\cup\{x\})}\big(1-\beta v_I^{(j)}(t)\big)\\
		&\times\Big(1 - \prod_{j\in\partial_i(t)\cap(\mathcal{A}\cup\{x\})}(1-\beta v_I^{(j)}(t))-(1-\beta v_I^{(x)}(t))\big(1 - \prod_{j\in\partial_i(t)\cap\mathcal{A}}(1-\beta v_I^{(j)}(t))\big)\Big),
	\end{align*}
	which implies
	\begin{align*}
		&\mathbb{E}[F_i(\mathcal{A}\cup\{x\};t)] - \mathbb{E}[F_i(\mathcal{A};t)]\\
		&=v_S^{(i)}(t)\prod_{j\in\partial_{i}(t)\backslash(\mathcal{A}\cup\{x\})}\big(1-\beta v_I^{(j)}(t)\big)\\
		&\times\Big(\beta v_I^{(x)}(t) - \prod_{j\in\partial_i(t)\cap(\mathcal{A}\cup\{x\})}(1-\beta v_I^{(j)}(t))+ \prod_{j\in\partial_i(t)\cap(\mathcal{A}\cup\{x\})}(1-\beta v_I^{(j)}(t))\big)\Big)\\
		=&v_S^{(i)}(t)\prod_{j\in\partial_{i}(t)\backslash(\mathcal{A}\cup\{x\})}\big(1-\beta v_I^{(j)}(t)\big)\beta v_I^{(x)}(t).
	\end{align*}
	Similarly, note that $\big(\mathcal{B}\cap\partial_i(t)\big)=\big(\mathcal{A}\cap\partial_i(t)\big)\cup\mathcal{Y}$. We have
	\begin{align*}
		&\mathbb{E}[F_i(\mathcal{B}\cup\{x\};t)] - \mathbb{E}[F_i(\mathcal{B};t)]\\
		=&v_S^{(i)}(t)\prod_{j\in\partial_{i}(t)\backslash(\mathcal{A}\cup(\{x\}\cup\mathcal{Y}))}\big(1-\beta v_I^{(j)}(t)\big)\beta v_I^{(x)}(t).
	\end{align*}
	Thus,
	\begin{align*}
		\frac{\mathbb{E}[F_i(\mathcal{A}\cup\{x\};t)] - \mathbb{E}[F_i(\mathcal{A};t)]}{\mathbb{E}[F_i(\mathcal{B}\cup\{x\};t)] - \mathbb{E}[F_i(\mathcal{B};t)]} = \prod_{y\in\mathcal{Y}}\big(1-\beta v_I^{(y)}(t)\big)\leq 1,
	\end{align*}
	which implies $S\big(\mathcal{TP}^\pi(t)\big)$ is supmodular. 
	
	To show $S\big(\mathcal{K}^\pi(t); t\big)$ is an increasing monotone function on $\mathcal{K}^\pi(t)$, it suffices to show $\mathbb{E}[F_i\big(\mathcal{K}^\pi(t);t\big)]$ is an increasing monotone function on $\mathcal{K}^\pi(t)$ for any $i$.
	
	For $\mathcal{A}\subset\mathcal{B}$, we have $\partial_{i}(t)\backslash\mathcal{B}\subset \partial_{i}(t)\backslash\mathcal{A}$, and $\mathcal{A}\cap\partial_{i}(t)\subset \mathcal{B}\cap\partial_{i}(t)$. Then
	\begin{align*}
		&\prod_{j\in\partial_i(t)\backslash\mathcal{B}}\big(1-\beta v_I^{(j)}(t)\big)\geq \prod_{j\in\partial_i(t)\backslash\mathcal{A}}\big(1-\beta v_I^{(j)}(t)\big)\\
		&\prod_{j\in\mathcal{B}\cap\partial_i(t)}(1-\beta v_I^{(j)}(t))\leq \prod_{j\in\mathcal{A}\cap\partial_i(t)}(1-\beta v_I^{(j)}(t)),
	\end{align*}
	and thus, from (\ref{eq: fEF}), we have $\mathbb{E}[F_i\big(\mathcal{A};t\big)]\leq \mathbb{E}[F_i\big(\mathcal{B};t\big)]$.

	\section{Complexity of Algorithm~\ref{alg: Greedy Algorithm}}\label{App: complexity of greedy}
	First of all, we consider the complexity of (\ref{eq: fEF}). Suppose $\{\underline{v}_i(t)\}_{i\in\mathcal{V}(t)}$ is given for every day~$t$. For any $\mathcal{K}^\pi(t)$, the complexity of computing (\ref{eq: fEF}) is 
	\begin{align*}
		1 + |\partial_i(t)\backslash\mathcal{K}^\pi(t)| - 1 + 1 + |\partial_i(t)\backslash\mathcal{K}^\pi(t)| + |\partial_i(t)\cap\mathcal{K}^\pi(t)|-1 + |\partial_i(t)\cap\mathcal{K}^\pi(t)| =2|\partial_{i}(t)|.
	\end{align*}
	Then, for any $\mathcal{K}^\pi(t)$, the complexity of computing $S\big(\mathcal{K}^\pi(t); t\big)$ is
	\begin{align*}
		2\sum_{j\in\mathcal{V}(t)}|\partial_j(t)|.
	\end{align*}
	From Algorithm~\ref{alg: Greedy Algorithm}, in step~$i$, the complexity is
	\begin{align*}
		\big(N(t)-i+1\big)\times 2\sum_{j\in\mathcal{V}(t)}|\partial_{j}(t)|.
	\end{align*}
	And in total we have $\big(N(t)-|\mathcal{K}^\pi(t)|\big)$ steps, therefore, 
	on day $t$ the complexity of Algorithm~\ref{alg: Greedy Algorithm} is
	\begin{align*}
		\sum_{i=0}^{N(t)-|\mathcal{K}^\pi(t)|}2\big(N(t)-i+1\big)\sum_{j\in\mathcal{V}(t)}|\partial_{j}(t)|.
	\end{align*}
	Recall that the time horizon is $T$, then the total complexity of Algorithm~\ref{alg: Greedy Algorithm} is
	\begin{align*}
		\sum_{t=0}^{T-1}\sum_{i=0}^{N(t)-|\mathcal{K}^\pi(t)|}2\big(N(t)-i+1\big)\sum_{j\in\mathcal{V}(t)}|\partial_{j}(t)|.
	\end{align*}
	Note that 
	\begin{align*}
		\sum_{i=1}^{N(t)-|\mathcal{K}^\pi(t)|}2\big(N(t)-i+1\big) \leq & O\big(N^2(t)\big)\\
		\sum_{i\in\mathcal{V}(t)}|\partial_{i}(t)| \leq & O\big(N^2(t)\big).
	\end{align*}
	Then, the total complexity is bounded by
	\begin{align*}
		O\Big(\sum_{t=0}^{T-1} N^4(t)\Big).
	\end{align*}

	\section{Proof of Lemma~\ref{lem: upperbound}}\label{App: supermodularity}
	As defined in  \cite{V2001} (an equivalent definition of footnote~3), consider a finite set $I$, $f: 2^I\to\mathbb{R}$ is a supermodular function if for all $X, Y\subset I$,
	\begin{align}\label{eq: XYsupermodular}
		f(X\cup Y) + f(X\cap Y) \geq f(X) + f(Y).
	\end{align}

	Following the supermodularity of function $S(\cdot)$ as shown in Theorem~\ref{thm: supmodular}, set $X = \mathcal{V}(t)\backslash\mathcal{K}^\pi(t)$ and $Y = \mathcal{K}^\pi(t)$ in (\ref{eq: XYsupermodular}), we have
	\begin{align}\label{eq: supmodular1}
		S\big(\mathcal{V}(t)\backslash\mathcal{K}^\pi(t); t\big) \leq S\big(\mathcal{V}(t); t\big) - S\big(\mathcal{K}^\pi(t); t\big).
	\end{align} 
	Again, set $X = \mathcal{K}^\pi(t)\backslash\{i\}$ and $Y = \{i\}$ in (\ref{eq: XYsupermodular}), and use (\ref{eq: XYsupermodular}) repeatedly to obtain:
	\begin{align}
		&S\big(\mathcal{K}^\pi(t); t\big)\geq\sum_{i\in\mathcal{K}^\pi(t)}S\big(\{i\}; t\big)=\sum_{i\in\mathcal{K}^\pi(t)}r_i(t).\label{eq: upperbound}
	\end{align} 
	
	Substituting  (\ref{eq: upperbound}) in (\ref{eq: supmodular1}), we obtain 
	\begin{align*} S\big(\mathcal{V}(t)\backslash\mathcal{K}^\pi(t); t\big)\leq S\big(\mathcal{V}(t); t\big)-\sum_{i\in\mathcal{K}^\pi(t)}r_i(t). 
	\end{align*}

	\section{Local Transition Equations}\label{App: Local Transition Equations}
	
	In this section, we will describe the local transition matrix ${\tt P}_i\big(\{\underline{v}_j(t)\}_{j\in\partial^+_i(t)}\big)$ used in (\ref{eq: matrix form}). The state of each node evolves as follows: (i) if node $i$ is susceptible on day $t$, then it might be infected by its neighbors in $\partial_i(t)$; (ii) an infectious  node remains in the latent state with probability $1-\lambda$, and changes state to the infectious state ($I$) with probability $\lambda$; (iv) if node $i$ is in state $I$, it will recover after a geometric distribution with parameter $\gamma$.   Let $\xi_i(t)=1-\prod_{m\in \partial_i(t)}\big(1\! -\! v_{I}^{(m)}(t)\beta\big)$. 
	In particular, define $\xi_i(t)=0$ if $\partial_i(t)=\varnothing$.
	Then, the probabilities of nodes being in different states evolve in time as follows:
	\begin{align}
		v_{I}^{(i)}(t+1) = & v_{I}^{(i)}(t)(1-\gamma) + v_{L}^{(i)}(t)\lambda\label{eq: uI_a}\\
		v_{L}^{(i)}(t+1) =&v_{L}^{(i)}(t)(1-\lambda) + v_{S}^{(i)}(t)\xi_i(t)\label{eq: uL}\\
		v_{R}^{(i)}(t+1) = & v_{R}^{(i)}(t) + v_{I}^{(i)}(t)\gamma\label{eq: uR}\\
		v_{S}^{(i)}(t+1) =& v_{S}^{(i)}(t)\big(1-\xi_i(t)\big)\label{eq: uS}.
	\end{align}
	Note that row vector $\underline{v}_i(t)$ is defined in (\ref{eq: rowvectorv}).	Collecting 	(\ref{eq: uI_a}) - (\ref{eq: uS}), we define the local transition probability matrix as given below:
	\begin{equation}\label{eq: transition probability matrix}
		\begin{aligned}
			&{\tt P}_i\big(\{\underline{v}_j(t)\}_{j\in\partial^+_i(t)}\big)\! =\! \left[\begin{matrix}
				(1\!-\!\gamma)&0&\gamma&0\\
				\lambda &1\!-\!\lambda&0&0\\
				0&0&1&0\\
				0&\xi_i(t)&0&1-\xi_i(t)
			\end{matrix}\right].
		\end{aligned} 
	\end{equation}
	and we obtain (\ref{eq: matrix form}).

	\section{Proofs of (\ref{eq: estimate matrix form1}) and (\ref{eq: estimate matrix form2})}\label{App: Proof two equations}
	First of all, we give the following definition.
	\begin{definition}\label{def: condition random variable}
		Let $X$ be a random variable and $\mathcal{B}$ be an event. Define $X|_{\mathcal{B}}$ as the random variable $X$ given $\mathcal{B}$;~i.e.,
		
		\begin{align}\label{eq: condition random variable}
			\Pr\big(X|_{\mathcal{B}}=x\big) = \Pr\big(X=x|\mathcal{B}\big).
		\end{align}
	\end{definition}
	
	For brevity, let us define 
	\begin{align*}
		\theta_i(t) = \sigma_i(t)|_{\{\underline{Y}(\tau)\}_{\tau=1}^{t-1}},\quad \zeta_i(t) = \sigma_i(t)|_{\{\underline{Y}(\tau)\}_{\tau=1}^{t}}.
	\end{align*}
	We thus have
	\begin{align*}
		u_x^{(i)}(t) = \Pr\big(\theta_i(t)=x\big),\quad w_x^{(i)}(t) = \Pr\big(\zeta_i(t)=x\big).
	\end{align*}
	
	Recall that 
	\begin{align*}
		\underline{v}_i(t) =&  \big[v_x^{(i)}(t)\big]_{ x\in\mathcal{X}},\,\, v_{x}^{(i)}(t)= \Pr\big(\sigma_i(t) = x\big).
	\end{align*}
	Then, \eqref{eq: matrix form} can be re-written as
	\begin{align}\label{eq: statechangeeq}
		\Pr\big(\sigma_i(t+1)=x'_i\big) = \Pr\big(\sigma_i(t)=x_i\big){\tt P}_i\Big(\{\sigma_j(t)\}_{j\in\partial^+_i(t)} = \{x_j\}_{j\in\partial^+_i(t)}\Big),
	\end{align}
	where $x_i', \{x_j\}_{j\in\partial^+_i(t)}\in\mathcal{X}$.
	Conditioning both sides of (\ref{eq: statechangeeq}) on $\{\underline{Y}(\tau)\}_{\tau=1}^{t-1}$, state variables 
	$\sigma_i(t)$ and $\sigma_i(t-1)$ in (\ref{eq: statechangeeq}) can be replaced by $\theta_i(t)$ and $\zeta_i(t-1)$, respectively, to obtain
	\begin{align}\label{eq: estimate matrix form}
		\underline{u}_i(t) = \underline{w}_i(t-1)\times{\tt P}_i\big(\{\underline{w}_j(t-1)\}_{j\in\partial^+_i(t-1)}\big),
	\end{align}
	which gives (\ref{eq: estimate matrix form1}).
	In addition, define 
	\begin{align*}
		\phi_i(t) =  \sigma_i(t)|_{\{\underline{Y}(\tau)\}_{\tau=1}^{t+1}},
	\end{align*}
	and
	\begin{equation}\label{eq: e def}
		\begin{aligned}
			&\underline{e}_i(t-1) = (e_x^{(i)}(t-1), x\in\mathcal{X}),\\
			&e_x^{(i)}(t-1) = \Pr\big(\phi_i(t-1) = x\big).
		\end{aligned}
	\end{equation}
	
	This notation implies 
	\begin{align}\label{eq: thetait}
		\phi_i(t-1) = \theta_i(t-1)|_{\underline{Y}(t)}.    
	\end{align}
	
	Similarly, conditioning both sides of  (\ref{eq: estimate matrix form}) on $\underline{Y}(t)$, we find
	\begin{align}
		\label{eq:we}
		\underline{w}_i(t) = \underline{e}_i(t-1)\times\tilde{{\tt P}}_i\big(\{\underline{e}_j(t-1)\}_{j\in\partial^+_i(t-1)}\big),
	\end{align}
	which gives (\ref{eq: estimate matrix form2}). $\tilde{P}_i(\{\underline{e}_j(t-1)\}_{j\in\partial^+_i(t-1)})$ is obtained in the following subsection.
	\subsection{Computing the transition probability matrix $\tilde{P}_i(\{\underline{e}_j(t-1)\}_{j\in\partial^+_i(t-1)})$}
	
	Note that $\tilde{P}_i(\{\underline{e}_j(t-1)\}_{j\in\partial^+_i(t-1)})$ is not the same as $P_i(\{\underline{w}_j(t)\}_{j\in\partial^+_i(t)})$. {\it This is because ``future'' observations were available in $\tilde{P}_i(\{\underline{e}_j(t-1)\}_{j\in\partial^+_i(t-1)})$}.
	To get $\tilde{{\tt P}}_i\big(\{\underline{e}_j(t-1)\}_{j\in\partial^+_i(t-1)}\big)$,
	we  split the nodes $\mathcal{V}(t)$ into two classes of nodes: (i)  nodes who do not get new observations and (ii)  nodes who get new observations. $\tilde{P}_i(\{\underline{e}_j(t-1)\}_{j\in\partial^+_i(t-1)})$ is obtained by the following rules. For the first class of nodes, the local transition matrix in (\ref{eq:we}), i.e., $\tilde{P}_i(\{\underline{e}_j(t-1)\}_{j\in\partial^+_i(t-1)})$, is the same as that in (\ref{eq: estimate matrix form}). However, for the second class of nodes, the local transition matrices are changed accordingly because of the new observations. 
	Let $[A]_{\{i,:\}}$ be the $i^{th}$ row of matrix $A$, and $q_i$ be a $1\times 4$ vector with the $i^{th}$ element being one and the rest zero. For brevity, denote the local transition matrices in (\ref{eq: estimate matrix form}) and (\ref{eq:we}) by ${\tt P}_i(t-1)$ and $\tilde{{\tt P}}_i(t-1)$, respectively. We have the following three cases:
	\begin{itemize}
		\item [(i)] If node $i$ is not observed, then node $i$ does not have new observation and we have 
		\begin{align}\label{eq: Ptilde1}
			\tilde{P}_i(t-1) = P_i(t-1).
		\end{align}
		\item [(ii)] If $Y_i(t)=0$,  then node $i$ is not infectious in day $t$ with probability $1$. The local transition matrix is changed to
		\begin{align}\label{eq: Ptilde2}
			[\tilde{{\tt P}}_i(t-1)]_{\{j,:\}} = \left\{
			\begin{aligned}
				&q_3&&j = 1\\
				&q_2&&j = 2\\
				&[{\tt P}_i(t-1)]_{\{j,:\}}&&\text{otherwise}
			\end{aligned}
			\right..
		\end{align}
		\item [(iii)] If $Y_i(t)=1$,  then node $i$ is infectious in day $t$ with probability $1$. The local transition matrix is changed to
		\begin{align}\label{eq: Ptilde3}
			[\tilde{{\tt P}}_i(t-1)]_{\{j,:\}} = \left\{
			\begin{aligned}
				&q_1&&j = 1\\
				&q_1&&j=2\\
				&[{\tt P}_i(t-1)]_{\{j,:\}}&&\text{otherwise}
			\end{aligned}
			\right..
		\end{align}
		
	\end{itemize}

	\section{Proofs of (\ref{eq: chain rule of e2}) and (\ref{eq: fenzi34})}\label{App: proof for e}\label{AppendixG}

	Using new observations, we aim to move backward in time and update  our belief (posterior probability) in previous time slots. Define a {\it truncation number} $g$ and suppose  
	that $\{\underline{Y}(t)\}$ affects the posterior probabilities from day $t$ to day $t-g$. 
	We call day $t-g$ the {\it truncation day} associated with day $t$. To get  accurate posterior probabilities in every day, we need to set $g=t$ on every day $t$ and track back to the initial time.
	However, the influence weakens as time elapses backwards, and for computation tractability, we continue under the following assumption where $g=1$. Recall that $\zeta_i(t)=\sigma_i(t)|_{\{\underline{Y}(\tau)\}_{\tau=1}^{t}}$.
	\begin{assumption}\label{assu: independence in terminate}
		On the truncation day $(t-g)$, $\{\zeta_i(t-g)\}_i$ are independent over $i$. In the following, the truncation number is assumed to be $g=1$.
	\end{assumption}
	\begin{remark}
		In Assumption~\ref{assu: independence in terminate}, we assume that the nodes' states $\zeta(t-g)$ (in the posterior probability space on day $t-g$)  are independent. This assumption is only used at time $t$ of our probability update in a moving window kind of way. 
		It provides us with a truncation time for each backward step. In particular, under Assumption~\ref{assu: independence in terminate}, once we get the observations $\underline{Y}(t)$, we do the backward step and truncate at time $t-g$. 
		For example, in the trivial case of  $g=t$, the assumption holds. This assumption does not impose independence on the state of the nodes, but only in the posterior space at a specific time. In a sense, in the process of propagating information back to time $t-g$, we are assuming that there is no further correlation between time $t-g-1$ and time $t-g$ worthwhile to exploit given observations at time $t$.  Naturally, as $g$ gets larger and larger, our framework and calculations  become more precise but this comes at a huge computational cost.
		The idea behind truncating the backward step lies in the observation that the impact of the testing results  at time $t$ in inferring about the nodes' probabilities at time $t-g$ vanishes as $g$ gets large. For simplicity of derivations and to have manageable complexity, we set $g=1$. The idea and the derivations can be generalized in a straightforward manner to larger $g$. 
	\end{remark}

	Note that the posterior probabilities on day $t-1$, $\underline{w}_i(t-1),\,i\in\mathcal{V}(t-1)$,  are assumed known (and are conditioned on the history of observations $\{\underline{Y}(\tau)\}_{\tau=1}^{t-1}$). The probability vector $\underline{e}_i(t-1)$ is the new posterior probability at time $t-1$ which is updated (from $\underline{w}_i(t-1)$) based on new observations $\underline{Y}(t)$. In other words, we infer about the previous state of the nodes given new observations at present time.

	To obtain $\{\underline{w}_i(t)\}_{i\in\mathcal{V}(t)}$, it suffices to obtain $\underline{e}_i(t-1)$ and the corresponding local transition matrix $\tilde{{\tt P}}_i\big(\{\underline{e}_j(t-1)\}_{j\in\partial^+_i(t-1)}\big)$, see (\ref{eq: estimate matrix form2}).
	Note that the posterior probabilities  $\underline{w}_i(t-1),\,i\in\mathcal{V}(t-1)$, which are calculated based on $\underline{Y}(t-1)$, are known. The vector $\underline{e}_i(t-1)$ is the new posterior probability which is updated based on $\underline{Y}(t)$ and $\underline{w}_i(t-1)$.

	Equation (\ref{eq: chain rule of e2}), which we aim to prove, simply follows from Definition~\ref{def: condition random variable}, (\ref{eq: e def})-(\ref{eq: thetait}), and Bayes rule:
	\begin{equation}\label{eq: chain rule of e4}
		\begin{aligned}
			e_x^{(i)}(t-1) = \Pr\big(\zeta_i(t-1) = x|\underline{Y}(t)\big)=\frac{\Pr\big(\underline{Y}(t)|\zeta_i(t-1) = x\big)w_x^{(i)}(t-1)}{\Pr\big(\underline{Y}(t)\big)}.
		\end{aligned}
	\end{equation}
	To find  $\Pr\big(\underline{Y}(t)|\zeta_i(t-1) = x\big)$, and establish (\ref{eq: fenzi34}), we now proceed as follows.
	We introduce $\{\theta_j(t)\}_{j\in\mathcal{O}(t)}$ into (\ref{eq: chain rule of e4}). In particular, we have
	\begin{equation*}
		\begin{aligned}
			&\Pr\big(\underline{Y}(t)|\zeta_i(t-1) = x\big)=\sum_{\theta_j(t),\ j\in\mathcal{O}(t)}\Pr\big(\{\theta_j(t)\}_{j\in\mathcal{O}(t)}, \underline{Y}(t)|\zeta_i(t-1) = x\big)
		\end{aligned}
	\end{equation*}
	By the chain rule of conditional probability,
	\begin{equation*}
		\begin{aligned}
			&\Pr\big(\underline{Y}(t)|\zeta_i(t-1) = x\big)\\
			=&\sum_{\theta_j(t),\ j\in\mathcal{O}(t)}\Pr\big(\underline{Y}(t)|\{\theta_j(t)\}_{j\in\mathcal{O}(t)},\zeta_i(t-1) = x\big)\times\Pr\big(\{\theta_j(t)\}_{j\in\mathcal{O}(t)}|\zeta_i(t-1) = x\big).
		\end{aligned}
	\end{equation*}
	
	From (\ref{eq: zetathetat}),  $\{\zeta_j(t)\}_{j\in\mathcal{V}(t)}$ and $\{\theta_j(t)\}_{j\in\mathcal{V}(t)}$ are variables defined by $\{\sigma_j(t)\}_{j\in\mathcal{V}(t)}$ in posterior spaces of $\{\underline{Y}(\tau)\}_{\tau=1}^{t}$ and $\{\underline{Y}(\tau)\}_{\tau=1}^{t-1}$, respectively. Since $\underline{Y}(t)$ is a deterministic function of  $\{\sigma_j(t)\}_{j\in\mathcal{O}(t)}$, and hence  $\{\theta_j(t)\}_{j\in\mathcal{O}(t)}$, then $\underline{Y}(t)$ is independent of $\zeta_i(t-1)$ given $\{\theta_j(t)\}_{j\in\mathcal{O}(t)}$. In addition, the testing result $Y_j(t)$ (on day $t$) of  node $j$ only depends on its state, i.e., given $\theta_j(t)$, the testing results are determined.  Therefore, we have
	\begin{align*}
	\Pr\big(\underline{Y}(t)|\{\theta_j(t)\}_{j\in\mathcal{O}(t)},\zeta_i(t-1) = x\big)=\Pr\big(\underline{Y}(t)|\{\theta_j(t)\}_{j\in\mathcal{O}(t)}\big)=\prod_{j\in \mathcal{O}(t)}\Pr\big(Y_j(t)|\theta_j(t)\big).
	\end{align*}
	The product above is an indicator which takes values on $\{0,1\}$. We can thus re-write it as follows: 
	\begin{align*}
		&\Pr\big(\underline{Y}(t)|\{\theta_j(t)\}_{j\in\mathcal{O}(t)},\zeta_i(t-1) = x\big)\triangleq\delta({\{Y_j(t), \theta_j(t)\}_{j\in\mathcal{O}(t)}}).
	\end{align*}
	where $$\delta({\{Y_j(t), \theta_j(t)\}_{j\in\mathcal{O}(t)}})=1$$ if the pairs $\{Y_j(t), \theta_j(t)\}_{j\in\mathcal{O}(t)}$ are consistent,  and  $$\delta({\{Y_j(t), \theta_j(t)\}_{j\in\mathcal{O}(t)}})=0$$ otherwise.
	
	Next, define 
	\begin{align*}
		\Theta_i(t) = \{j| j\in\partial^+_k(t-1), k\in\mathcal{O}(t)\}\backslash \{i\}
	\end{align*}
	to represent the neighbors (in day $t-1$) of nodes in $\mathcal{O}(t)$ excluding node $i$. Then,
	\begin{equation}\label{eq: Theta Phi}
		\begin{aligned}
			&\Pr\big(\{\theta_j(t)\}_{j\in\mathcal{O}(t)}|\zeta_i(t-1) = x\big)=\!\!\!\!\sum_{\zeta_l(t\!-\!1),\  l\in\Theta_i(t)}\!\!\!\!\!\!\!\!\!\!\Pr\big(\{\theta_j(t)\}_{j\in\mathcal{O}(t)},\{\zeta_l(t\!-\!1)\}_{l\in\Theta_i(t)}|\zeta_i(t-1) = x\big).
		\end{aligned}
	\end{equation}
	By the chain rule of conditional probability,
	\begin{equation*}
		\begin{aligned}
			&\Pr\big(\{\theta_j(t)\}_{j\in\mathcal{O}(t)}|\zeta_i(t-1) = x\big)\\
			=&\!\!\!\!\sum_{\zeta_l(t-1),\  l\in\Theta_i(t)}\!\!\!\!\!\!\!\!\!\!\Pr\big(\{\theta_j(t)\}_{j\in\mathcal{O}(t)}|\{\zeta_l(t\!-\!1)\}_{l\in\Theta_i(t)},\zeta_i(t-1) = x\big)\times\Pr\big(\{\zeta_l(t\!-\!1)\}_{l\in\Theta_i(t)}|\zeta_i(t-1) = x\big).
		\end{aligned}
	\end{equation*}
	Given $\{\zeta_l(t-1)\}_{l\in\Theta_i(t)}\cup\{\zeta_i(t-1)\}$, $\{\theta_j(t)\}_{j\in\mathcal{O}(t)}$ are independent. We thus have
	\begin{align*}
		&\Pr\big(\{\theta_j(t)\}_{j\in\mathcal{O}(t)}|\{\zeta_l(t-1)\}_{l\in\Theta_i(t)},\zeta_i(t-1) = x\big)\\
		=&\prod_{j\in\mathcal{O}(t)}\Pr\big(\theta_j(t)|\{\zeta_l(t-1)\}_{l\in\Theta_i(t)},\zeta_i(t-1) = x\big)\\
		=&\prod_{j\in\mathcal{O}(t)}\Pr\big(\theta_j(t)|\{\zeta_l(t\!-\!1)\}_{l\in\partial^+_j(t-1)\backslash\{i\}},\zeta_i(t\!-\!1) = x\big).
	\end{align*}

	Based on Assumption~\ref{assu: independence in terminate},
	\begin{align*}
		\Pr\big(\{\zeta_l(t-1)\}_{l\in\Theta_i}|\zeta_i(t-1) = x\big)=\!\!\!\!\prod_{l\in\{\Theta_i(t)\}}\!\!\!\!\Pr\big(\zeta_l(t-1)\big).
	\end{align*}
	Therefore, 
	\begin{equation}\label{eq: fenzi1}
		\begin{aligned}
			&\Pr\big(\underline{Y}(t)|\zeta_i(t-1) = x\big)\\
			=&\sum_{\theta_j(t),\ j\in\mathcal{O}(t)}\delta({\{Y_j(t), \theta_j(t)\}_{j\in\mathcal{O}(t)}})\\
			\times&\sum_{\zeta_l(t-1)}\prod_{j\in\mathcal{O}(t)}\Pr\big(\theta_j(t)|\{\zeta_l(t\!-\!1)\}_{l\in\partial^+_j(t\!-\!1)}\backslash\{i\},\zeta_i(t-1) = x\big)\times\prod_{l\in\{\Theta_i(t)\}}\Pr\big(\zeta_l(t-1)\big).
		\end{aligned}
	\end{equation}
	Denote $\{x_j\}_{j\in\mathcal{O}(t)}$ as a realization of $\{\theta_j(t)\}_{j\in\mathcal{O}(t)}$ and  $\{y_l\}_{l\in\Theta_i(t)}$ as a realization of $\{\zeta_l(t-1)\}_{l\in\Theta_i(t)}$. Then,
	\begin{equation*}
		\begin{aligned}
			&\Pr\big(\underline{Y}(t)|\zeta_i(t-1) = x\big)\\
			=&\sum_{\{x_j\}_{j\in\mathcal{O}(t)}}\delta({\{Y_j(t), x_j\}_{j\in\mathcal{O}(t)}})\\
			&\times\sum_{\{y_l\}_{l\in\Theta_i(t)}}\prod_{j\in\mathcal{O}(t)}\Pr\big(x_j|\{y_l\}_{l\in\partial^+_j(t-1)\backslash\{i\}},\zeta_i(t-1) = x\big)\times\prod_{l\in\{\Theta_i(t)\}}\Pr\big(\zeta_l(t-1) = y_l\big).
		\end{aligned}
	\end{equation*}
	Denote 
	\begin{equation}\label{eq: rho}
		\begin{aligned}
			&\rho\big(\{x_j\}_{j\in\mathcal{O}(t)}, x\big) \\
			=&\sum_{\{y_l\}_{l\in\Theta_i(t)}}\prod_{j\in\mathcal{O}(t)}\Pr\big(x_j|\{y_l\}_{l\in\partial^+_j(t-1)\backslash\{i\}},\zeta_i(t-1) = x\big)\times\prod_{l\in\Theta_i(t)}\Pr\big(\zeta_l(t-1)=y_l\big).
		\end{aligned}
	\end{equation}
	Then,
	\begin{equation}\label{eq: fenzi3}
		\begin{aligned}
			&\Pr\big(\underline{Y}(t)|\zeta_i(t-1) = x\big) = \sum_{x_j\in\mathcal{X},j\in\mathcal{O}(t)}\delta({\{Y_j(t), x_j\}_{j\in\mathcal{O}(t)}})\rho\big(\{x_j\}_{j\in\mathcal{O}(t)}, x\big).
		\end{aligned}
	\end{equation}

	Based on Assumption~\ref{assu: independence in terminate}, we can further simplify (\ref{eq: fenzi3}). Consider node $i$,
	$\underline{Y}(t)$ can be split into $\underline{Y}_{i,1}(t)$ and $\underline{Y}_{i,2}(t)$, where $\underline{Y}_{i, 1}(t)$ is the observations of the set $\mathcal{O}(t)\cap\partial_i^+(t-1)$, and $\underline{Y}_{i, 2}(t)$ is the observations of the rest of the nodes. Note that $ \underline{Y}_{i, 1}(t)\cup \underline{Y}_{i, 2}(t) = \underline{Y}(t)$ and $ \underline{Y}_{i, 1}(t)\cap \underline{Y}_{i, 2}(t) = \varnothing$. 
	\begin{lemma}\label{lem: local independence}
		Conditioned on $\underline{Y}_{i, 1}(t)$, $\zeta_i(t-1)$ is independent of $\underline{Y}_{i, 2}(t)$.
	\end{lemma}
	\begin{proof}
		To show Lemma~\ref{lem: local independence}, we use the structured belief network as defined in	\cite{RDS2013}. $\zeta_j(t)$ is the random variable associated with node~$j$. Note that $Y_j(t)$ is the test result of 	$\zeta_j(t)$ on day $t$. Now, we consider $j\in\big(\mathcal{O}(t)\backslash (\mathcal{O}(t)\cap\partial_i^+(t-1)) \big)$. By \cite[Theorems~1]{RDS2013} and Bayes ball algorithm defined in \cite[Section~$2$]{MJBall}, we investigate the following two cases.  
		\begin{itemize}
			\item [(i)] For any $j\in \big(\mathcal{O}(t)\backslash (\mathcal{O}(t)\cap\partial_i^+(t-1)) \big)$ with $Y_j(t)=1$, the corresponding state  $\zeta_j(t)$ is determined (which is $I$). Then, probabilities conditioning on $Y_j(t)$ is equivalent to (equal to) probabilities conditioning on $\zeta_j(t)$. By Bayes ball algorithm \cite{RDS2013, MJBall}, the information (the ball) is blocked at $\zeta_j(t)$ when the information (the ball) reaches $\zeta_j(t)$, which implies  the information (the ball) can not reach $\zeta_i(t-1)$. 
			\item [(ii)] For any $j\in \big(\mathcal{O}(t)\backslash (\mathcal{O}(t)\cap\partial_i^+(t-1)) \big)$ with $Y_j(t)=0$, $\zeta_j(t)$ is not determined. By Bayes ball algorithm \cite{RDS2013, MJBall}, when the information (the ball) reaches $\zeta_j(t)$, it can traverse $Y_j(t)$ when blocking $Y_j(t)$ (conditioning on $Y_j(t)$). However, by Assumption~\ref{assu: independence in terminate}, $\zeta_i(t-1)$ and $\zeta_j(t-1)$ are independent, so any path between $\zeta_i(t-1)$ and $\zeta_j(t-1)$ is blocked, including the path $\zeta_j(t-1)\leftrightarrow\zeta_j(t)\leftrightarrow Y_j(t)\leftrightarrow\zeta_j(t)\leftrightarrow\zeta_i(t-1)$. Thus, the information (the ball) can not reach $\zeta_i(t-1)$.
		\end{itemize}
		A simple example is given in Figure~\ref{Fig3}: Let $Y_1(t)=0$ and $Y_2(t)=1$. Given $Y_1(t)$ and $Y_2(t)$, $Y_3(t)$ is independent of $\zeta_1(t-1)$. 
		\begin{figure}[ht]
			\centering
			\includegraphics[width=0.6\textwidth]{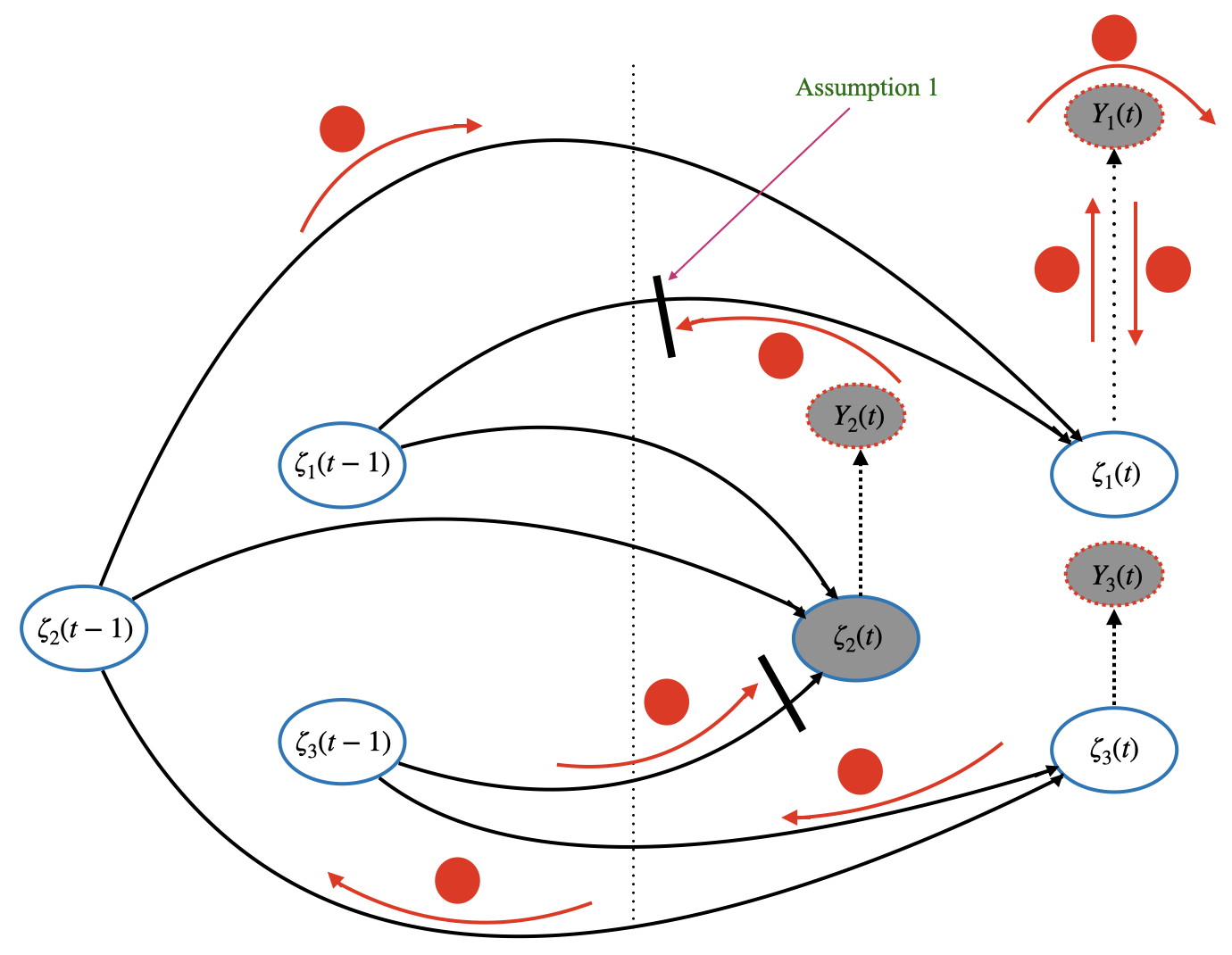} 
			\caption{Bayes ball algorithm in the network of $3$ nodes. The terms on which we have conditioning are shaded gray and are equivalently blocked.} 
			\label{Fig3} 
		\end{figure} 
	\end{proof}
	
	From Lemma~\ref{lem: local independence},
	\begin{equation}\label{eq: local independence}
		\begin{aligned}
			e_x^{(i)}(t-1) & = \Pr\big(\zeta_i(t-1) = x|\underline{Y}(t)\big)=\Pr\big(\zeta_i(t-1) = x|\underline{Y}_{i,1}(t)\big).
		\end{aligned}
	\end{equation}
	
	We simplify (\ref{eq: fenzi3}) based on Lemma~\ref{lem: local independence} or (\ref{eq: local independence}). From (\ref{eq: local independence}), denote the observations of nodes in $\partial_i^+(t-1)$ as  $\underline{Y}_{\partial^+_i}(t)$, $\underline{Y}_{\partial^+_i}(t)$ is independent of $\zeta_i(t-1)$. Denote $\Psi_i(t) = \mathcal{O}(t)\cap\partial^+_i(t-1)$. Then, We can replace $\mathcal{O}(t)$ by $\Psi_i(t)$ in (\ref{eq: chain rule of e2}). Subsequently, denote $\Phi_i(t) = \{j|j\in\partial^+_k(t-1), k\in\Psi_i(t)\}\backslash\{i\}$, and we can replace $\Theta_i(t)$ by $\Phi_i(t)$ in (\ref{eq: Theta Phi}). Thus, from (\ref{eq: rho}) and (\ref{eq: fenzi3}), we respectively have 
	\begin{equation}\label{eq: rho1}
		\begin{aligned}
		\rho\big(\{x_j\}_{j\in\Psi_i(t)}, x\big) =  \sum_{\{y_l\}_{l\in\Phi_i(t)}}\prod_{j\in\Psi_i(t)}\Pr\big(x_j|\{y_l\}_{l\in\partial^+_j(t-1)\backslash\{i\}},\zeta_i(t-1) = x\big)\times\prod_{l\in\Phi_i(t)}\Pr\big(\zeta_l(t-1)=y_l\big)
		\end{aligned}
	\end{equation}
	and
	\begin{equation}\label{eq: fenzi31}
		\begin{aligned}
			&\Pr\big(\{Y_j(t)\}_{j\in\Psi_i(t)}|\zeta_i(t-1) = x\big) = \sum_{x_j\in\mathcal{X},j\in\Psi_i(t)}\delta({\{Y_j(t), x_j\}_{j\in\Psi_i(t)}})\rho\big(\{x_j\}_{j\in\Psi_i(t)}, x\big)
		\end{aligned}
	\end{equation}
	which give the desired result (\ref{eq: fenzi34}).

	\section{A Simple Example for Algorithm~\ref{alg: Backward Update}}\label{App: simpleexample}
	
	In this section, we give a simple example to illustrate the ideas and steps of Algorithm~\ref{alg: Backward Update}. Besides, we compare our proposed algorithm (Algorithm~\ref{alg: Backward Update}) with the Naive approach discussed in Remark~\ref{remark1}. Consider a simple network with three nodes. Node $1$ has an edge with node $2$, node $2$ has an edge with node $3$ (see Fig~\ref{fig:graphical model}).  Nodes~$1$ and $3$ are symmetric and statistically identical, and node~$2$ has higher degree. 
	
	\begin{figure}[t!]
		\centering
		\includegraphics[width=.2\textwidth]{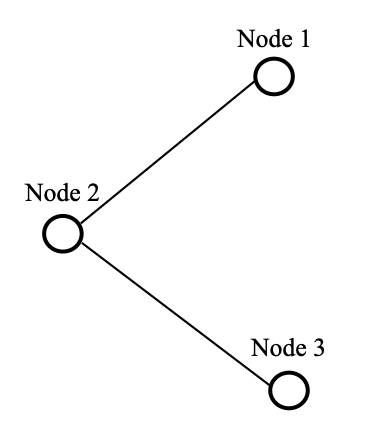}\quad\quad\quad
		\includegraphics[width=.45\textwidth]{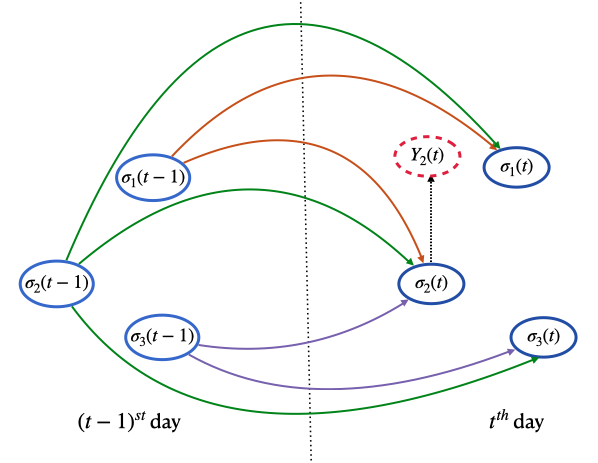}
		\caption{The original graph (left). The graphical model of  states and observations (right)}
		\label{fig:graphical model}
	\end{figure}
	
	Consider the following situation: on the initial day (day $0$), assume that nodes $1$ and $3$ are susceptible, and node $2$ is infectious. On day $1$, let node $2$ be tested.  Recall that we define the posterior probability vectors at the end of every day, and the prior probability vectors at the beginning of every day. Nodes' states change in the beginning of every day and testing is also done in the beginning of every day. Let the initial belief, i.e., the posterior probability $\underline{w}_i(0)$ and the prior probability $\underline{u}_i(0)$ on day $0$, of the nodes be
	\begin{align*}
		&\underline{w}_i(0) = [1/3, 0, 0, 2/3],\quad i = 1,2,3\\
		&\underline{u}_i(0) = [1/3, 0, 0, 2/3],\quad i=1,2,3.
	\end{align*}

	{\bf Day $0$}: No tests on day~$0$, the prior probabilities are updated by the forward step ({\bf Step~2} in Algorithm~\ref{alg: Backward Update}). 
	
	{\bf Day $1$}: By {\bf Step~2} in Algorithm~\ref{alg: Backward Update}, we have
	\begin{align*}
		&\underline{u}_1(1) = [0.3144, 0.0373, 0.0633, 0.5849]\\
		&\underline{u}_2(1) = [0.3559, 0.0676, 0.0633, 0.5131]\\
		&\underline{u}_3(1) = [0.3144, 0.0373, 0.0633, 0.5849].
	\end{align*}
	After testing node 2, we know that node 2 is positive. We use the test result to infer about the state of the nodes in prior times. In particular, we update the posterior probability on day $0$ ($\underline{w}_i(0)$). Denoting the updated posterior probability as $\underline{e}_i(0)$, by Step~1 in Algorithm~\ref{alg: Backward Update}, we find
	\begin{align*}
		&\underline{e}_1(0) = [0.3144, 0.0373, 0.0633, 0.5849]\\
		&\underline{e}_2(0) = [0.9615, 0.0385, 0.0, 0.0]\\
		&\underline{e}_3(0) = [0.3144, 0.0373, 0.0633, 0.5849].
	\end{align*}
	We can now say that at the end of day $0$, node 2 was infectious with probability $0.9615$  and it was in the latent state with probability $0.0385$. Moreover, we see that the (posterior) infection probabilities of nodes $1$ and $3$ on day $0$ have increased since they may have infected node $2$ on day $0$, i.e., $0.3517=e_I^{(i)}(0)+e_L^{(i)}(0)>1/3=w_I^{(i)}(0)+w_L^{(i)}(0)$ with $i=1, 3$. Next, we obtain the posterior probability on day $1$. Recall that $\underline{w}_i(t)$ describes the posterior probability vector of node $i$ { at the end of} day $t$. By Step~1 in Algorithm~\ref{alg: Backward Update},
	\begin{align*}
		&\underline{w}_1(1) = [0.4008, 0.0, 0.1268, 0.4724]\\
		&\underline{w}_2(1) = [0.90, 0.0, 0.10, 0.0]\\
		&\underline{w}_3(1) = [0.4008, 0.0, 0.1268, 0.4724].
	\end{align*}
	One may wonder why the posterior probability is $[0.9, 0, 0.1, 0]$ rather than $[1, 0, 0, 0]$. This is because testing is done in the beginning of time $t$ and the posterior probabilities are defined at the end of time slots $t$. The infected node may have recovered by the end of time $t=1$ and this is reflected in the posterior probabilities computed.
	
	{\bf Day $2$}: We can get the prior probability vectors on day $2$ by our  forward update (making predictions):
	\begin{align*}
		&\underline{u}_1(2) = [0.2883, 0.0, 0.1268, 0.5849]\\
		&\underline{u}_2(2) = [0.8135, 0.0, 0.1865, 0.0]\\
		&\underline{u}_3(2) = [0.2883, 0.0, 0.1268, 0.5849].
	\end{align*}
	
	On the other hand, if we apply the naive updating rule defined in Remark~\ref{remark1}, on day $2$, we find
	\begin{align*}
		&\underline{u}'_1(2) = [0.4074, 0.0896, 0.0948, 0.4082]\\
		&\underline{u}'_2(2) = [0.81, 0.0, 0.19, 0.0]\\
		&\underline{u}'_3(2) = [0.4074, 0.0896, 0.0948, 0.4082].
	\end{align*}
	
	Recall that we use Assumption~\ref{assu: independence in terminate} in the proposed algorithm (Algorithm~\ref{alg: Backward Update}), and the Naive approach in Remark~\ref{remark1} does not have the backward step, so both approaches do not capture the correlations among nodes. By Monte Carlo simulations, the correlations among nodes are captured, and the nodes' probability vectors are approximated on day $2$ as follows:
	\begin{align*}
		&\underline{v}_1(2) = [0.3235,0.0976, 0.0196, 0.5593]\\
		&\underline{v}_2(2) = [0.7244, 0,  0.2756, 0]\\
		&\underline{v}_3(2) = [0.3158, 0.1072, 0.019,  0.558]
	\end{align*}
	
	which yields the following comparison for the incurred estimation errors:
	\begin{align*}
0.4342=\sum_{i=1}^{3}||\underline{u}_i(2)-\underline{v}_i(2)|| < \sum_{i=1}^{3}||\underline{u}'_i(2)-\underline{v}_i(2)|| = 0.5018.
	\end{align*}
	The left hand side shows the estimation error under our proposed backward-forward update and the right hand side shows the estimation error under the naive approach.

	\section{Delay of Testing Results}\label{App: Delay of Testing Results}
	One can extend the framework to a more realistic case where testing results are not able to be obtained on the same day, but will be obtained after a delay $a$. In other words, if nodes are tested on day $t-a$, the test results are provided on day $t$. The extended framework is summarized as follows.

	On day $t$, before getting the test results of day $t-a$, the algorithm knows the following information: (i) the network topology from day $t-a-1$ to day $t$, i.e., $\mathcal{G}(t-a-1), \cdots, \mathcal{G}(t)$ (it is affected by the past actions); (ii) the posterior probability of nodes on day $t-a-1$, $\big\{\underline{w}_i(t-a-1)\big\}_{i\in\mathcal{G}(t-a-1)}$; and (iii) and the prior probability vectors of nodes from day $t-a$ to day $t$, i.e., $\big\{\underline{u}_i(t-a)\big\}_{i\in\mathcal{G}(t-a)},\cdots,\big\{\underline{u}_i(t)\big\}_{i\in\mathcal{G}(t)}$. 
	
	After getting the test results on day $t-a$, we can obtain the updated posterior probability vectors on day $t-a-1$, and the posterior probability vectors on day $t-a$, i.e., $\big\{\underline{e}_i(t-a-1)\big\}_{i\in\mathcal{G}(t-a-1)}$, and $\big\{\underline{w}_i(t-a)\big\}_{i\in\mathcal{G}(t-a)}$, by {\bf Step~1} in Algorithm~\ref{alg: Backward Update}.
	
	Based on $\{\underline{w}(t-a)\}_{i\in\mathcal{G}(t-a)}$,  by {\bf Step~2} in Algorithm~\ref{alg: Backward Update}, we update the prior probability from day $t-a+1$ to day $t$, and obtain the prior probability on day $t+1$, i.e., $\{\underline{u}_i(t+1)\}_{i\in\mathcal{G}(t+1)}$. 
	
	Repeating the process, we can compute the estimated probability vectors of nodes and apply the exploration and exploitation policies.

	\section{Proof of Theorem~\ref{theorem:backward}} \label{example:backward}

	\noindent{\bf Step~1}: Preliminaries.

	We divide the distributions of initial infectious nodes into two complementary events:
	\begin{align*}
		\mathcal{I}_1 =& \{\text{No node is infectious}\}\\
		\mathcal{I}_2 =& \mathcal{I}_1^c.
	\end{align*}
	Let $N$ be sufficiently large, 
	\begin{align*}
		&\Pr\{\mathcal{I}_1\} = (1-1/N)^{N} \approx 1/e\\
		&\Pr\{\mathcal{I}_2\}  \approx 1-1/e.
	\end{align*}
	
	In event $\mathcal{I}_1$, since there is no infection on the initial day, then no node is infectious in the future, i.e., the true probability of nodes $v_{I}^{(i)}(t) = 0$ for all $i\in\mathcal{V}(t)$ and $t\geq 1$.

	Note that in Example~\ref{ex: Necessity of Backward Updating}, each node can be in one of two states, $S$ and $I$. The transmission probability $\beta=1$. So, on day $t$, the probability of node $i$ being in state $I$ includes the infection of node $i$ on day $t-1$, and the infection from its neighbors. Then, based on (\ref{eq: estimate matrix form1}), we have
	\begin{equation}\label{eq: example-u_I}
		\begin{aligned}
		u_{I}^{(i)}(t) = & w_I^{(i)}(t-1) + \{1-w_I^{(i)}(t-1)\}\Big\{1-\big(1-w_I^{(i-1)}(t-1)\big)\big(1-w_I^{(i+1)}(t-1)\big)\Big\}\\
		=& 1 - \{1 - w_{I}^{(i-1)}(t-1)\}\{1 - w_{I}^{(i)}(t-1)\}\{1 - w_{I}^{(i+1)}(t-1)\}.
		\end{aligned}
	\end{equation}	
	For convention, we assume that nodes $0$ and $N+1$ are two virtual nodes with no probability of infection, i.e., $u_I^{(0)}(t) = u_I^{(N+1)}(t)=0$ for all $t$, and no tests are applied to these two nodes all the time. 
	
	Since $w_I^{(i)}(t+1), w_I^{(i-1)}(t+1)$, $w_I^{(i+1)}(t+1)\in[0, 1]$, then from (\ref{eq: example-u_I}),
	\begin{align}\label{eq: inequality}
		u_I^{(i)}(t) \geq 1 - 1\times (1 - w_{I}^{(i)}(t-1)) \times 1 =  w_{I}^{(i)}(t-1).
	\end{align}
	Thus, by symmetry over $w_{I}^{(i-1)}(t-1)$, $w_{I}^{(i)}(t-1)$, and $w_{I}^{(i+1)}(t-1)$ we get the inequality
	\begin{align}\label{eq: u www}
		u_{I}^{(i)}(t)\geq \max\{w_{I}^{(i-1)}(t-1), w_{I}^{(i)}(t-1), w_{I}^{(i+1)}(t-1)\}.
	\end{align}

	\noindent{\bf Step~2}: Consider the computation of $\{\underline{u}_i(t)\}_{i}$  based on~(\ref{eq: estimate matrix form1}) \big(equivalently (\ref{eq: example-u_I})\big) under event~$\mathcal{I}_1$.

	Recall that $B(t)=1$ for all $t$. On any day $t$, if node $i_0$ is tested, then the result is negative, and $w_{I}^{(i_0)}(t) = 0$, and 
	\begin{align}\label{eq: no test w=u}
		w_I^{(i)}(t) = u_I^{(i)}(t)\,\,\text{for all}\,\,i\neq i_0.
	\end{align}
	In (\ref{eq: u www}), at most one of $w_I^{(i-1)}(t-1)$, $w_I^{(i)}(t-1)$, and $w_I^{(i+1)}(t-1)$ is updated to $0$. We first prove the following facts. 
	
	{\bf Fact~1}. $u_I^{(i)}(t)\geq \frac{1}{N}$ for all $t$. On any day $t$, $w_I^{(i)}(t)\geq\frac{1}{N}$ with $i\neq i_0$, where $i_0$ is the index of node tested on day $t$.
	
	{\it Proof}. We prove {\bf Fact~1} by mathematical induction. On the initial day, by model assumption in Example~\ref{ex: Necessity of Backward Updating}, $u_I^{(i)}(0)=\frac{1}{N}$ for all $i$. Then, if node $i_0$ is tested, then as mentioned above, $w_{I}^{(i_0)}(0) = 0$, and by (\ref{eq: no test w=u}) $w_I^{(i)}(0) = u_I^{(i)}(0)=1/N$ for all $i\neq i_0$.
	
	Suppose {\bf Fact~1} holds for all $\tau\leq t-1$. Now, we consider $\tau=t$. From (\ref{eq: u www}), we have $u_I^{(i)}(t) \geq \max\{w_{I}^{(i-1)}(t-1), w_{I}^{(i)}(t-1), w_{I}^{(i+1)}(t-1)\}\geq 1/N$. Then, if node $i_0$ is tested, we have $w_{I}^{(i_0)}(t) = 0$, and then by (\ref{eq: no test w=u}), $w_I^{(i)}(t) = u_I^{(i)}(t)\geq 1/N$ for all $i\neq i_0$. \qed

	{\bf Fact~2}. If node~$i$ has not been tested up to day $t$, then $u_{I}^{(i)}(t)$ tends to $1$ as $t \rightarrow \infty$.
	
	{\it Proof}. 
	Since node $i$ is not tested from the initial day to day $t$, then 
	\begin{align}\label{eq: fact2proof w=u}
		w_I^{(i)}(\tau) = u_I^{(i)}(\tau),\,\,\tau\leq t.
	\end{align}
	
	Note that at most one of its neighbors is tested on day $t$. By (\ref{eq: example-u_I}) and {\bf Fact~1},
	\begin{align*}
		u_I^{(i)}(t) \geq 1 - (1-1/N)(1-w_I^{(i)}(t-1))=1 - (1-1/N)(1-u_I^{(i)}(t-1)),,
	\end{align*}
	which implies
	\begin{align*}
		(1-1/N)(1-u_I^{(i)}(t-1))\geq 1 - u_I^{(i)}(t),
	\end{align*}
	which implies
	\begin{align*}
		1 - u_I^{(i)}(t)\leq (1-1/N)^{t}(1-u_I^{(i)}(0)) = (1-1/N)^{t+1}.
	\end{align*}
	Letting $t\to\infty$ completes the proof.\qed
	
	{\bf Fact~3}. If node $i$ is not tested on day $t-1$, then
	\begin{align*}
		u_I^{(i)}(t)\geq w_I^{(i)}(t-1) + \frac{1}{N}(1-w_{I}^{(i)}(t-1))w_I^{(i)}(t-1).
	\end{align*}
	
	{\it Proof}. 
	By {\bf Fact~3}, if node $i$ is not tested on day $t-1$, then $w_I^{(i)}(t-1)>0$.
	From (\ref{eq: example-u_I}), by some algebra,	\begin{equation*}
		\begin{aligned}
			u_{I}^{(i)}(t) =& w_{I}^{(i)}(t-1) +  (1 -  w_{I}^{(i)}(t-1))(w_{I}^{(i-1)}(t-1)+w_{I}^{(i+1)}(t-1)-w_{I}^{(i-1)}(t-1)w_{I}^{(i+1)}(t-1))\\
			=&(1+\epsilon)w_{I}^{(i)}(t-1)
		\end{aligned}
	\end{equation*}
	where 
	\begin{equation*}
		\begin{aligned}
			&\epsilon = \frac{1-w_{I}^{(i)}(t-1)}{w_{I}^{(i)}(t-1)}\times(w_{I}^{(i-1)}(t-1) + w_{I}^{(i+1)}(t-1) - w_{I}^{(i-1)}(t-1)w_{I}^{(i+1)}(t-1)).
		\end{aligned}
	\end{equation*}
	Note that at most one of the neighbors of node $i$ is tested on day $t-1$, then
	\begin{align*}
		\frac{1-w_{I}^{(i)}(t-1)}{w_{I}^{(i)}(t-1)}\geq& 1-w_{I}^{(i)}(t-1)\\
		w_{I}^{(i-1)}(t-1) + w_{I}^{(i+1)}(t-1) - w_{I}^{(i-1)}(t-1)w_{I}^{(i+1)}(t-1)\geq& \max\{w_{I}^{(i-1)}(t-1), w_{I}^{(i+1)}(t-1)\}.
	\end{align*}
	From {\bf Fact~1}, $\max\{w_{I}^{(i-1)}(t-1), w_{I}^{(i+1)}(t-1)\}\geq 1/N$.
	Thus, $\epsilon \geq (1-w_{I}^{(i)}(t-1)) \times 1/N $. Hence, $u_I^{(i)}(t)\geq w_I^{(i)}(t-1) + \frac{1}{N}(1-w_{I}^{(i)}(t-1))w_I^{(i)}(t-1)$.\qed

	Since we consider all possible sequential testing policies, then we divide all nodes into two sets
	\begin{align*}
		\mathcal{S}_1(t) =& \{\text{nodes that have not been tested up to day $t$}\}\\
		\mathcal{S}_2(t) = &\mathcal{S}_1^c(t).
	\end{align*}
	In the following proof, let $t\to\infty$. By {\bf Fact~2}, $u_I^{(i)}(t)\to1$ if $i\in\mathcal{S}_1(t)$.
	Next, we focus on the set $\mathcal{S}_2(t)$. Denote the index of node which is tested on day $t-1$ as $i_0(t)$. 
	By {\bf Fact~1}, $w_I^{(i)}(t-1)\geq 1/N$ for all $i\neq i_0(t)$. Then, we define 
	\begin{align*}
		\mathcal{S}_{21}(t) =& \{i|1/N\leq w_I^{(i)}(t-1)< 1-1/N\}\\
		\mathcal{S}_{22}(t) = &\{i|1-1/N \leq w_I^{(i)}(t-1)\}.
	\end{align*}
	Thus, we have $\mathcal{S}_2(t) = \mathcal{S}_{21}(t) \cup\mathcal{S}_{22}(t)\cup\{i_0(t)\}$.
	Due to the equivalence of norms, without loss of generality, we consider $L_1$ norm in the rest of the proof.  
	\begin{itemize}
		\item [(i)] If $i\in\mathcal{S}_1(t)$, then $u_I^{(i)}(t) \to 1$. Thus $||\underline{u}_i(t) - \underline{v}_i(t)||_1\to2$.
		\item [(ii)] If $i\in\mathcal{S}_{21}(t)$, then $||\underline{u}_i(t) - \underline{v}_i(t)||_1\geq ||\underline{u}_i(t-1) - \underline{v}_i(t-1)||_1 + \frac{2(N-1)}{N^3}$.
		In fact, since $i\in\mathcal{S}_{21}(t)$, then $i\neq i_0(t)$, thus by (\ref{eq: no test w=u}) and {\bf Fact~3}, 
		\begin{align*}
			u_I^{(i)}(t) \geq u_I^{(i)}(t-1) + \frac{1}{N}(1-w_{I}^{(i)}(t-1))w_I^{(i)}(t-1).
		\end{align*}
		Note that $N$ is sufficiently large, so $1/N<1/2<1-1/N$. If $x\in[1/N, 1-1/N)$, then the fuction $f(x)=x(1-x)$ has the minimum value $\frac{N-1}{N^2}$ when $x=1/N$. Thus,
		\begin{align}\label{eq: wuN-1N2}
			u_I^{(i)}(t) \geq u_I^{(i)}(t-1) + \frac{N-1}{N^3}.
		\end{align}
		
		Recall that $v_I^{(i)}(t)=0$ and $v_S^{(i)}(t)=1$ for all $t$, and $u_I^{(i)}(t) + u_S^{(i)}(t)=1$, then
		\begin{align}\label{eq: norm12}
			||\underline{u}_i(t) - \underline{v}_i(t)||_1 = |u_I^{(i)}(t)-v_I^{(i)}(t)| + |u_S^{(i)}(t)-v_S^{(i)}(t)| =2|u_I^{(i)}(t)-v_I^{(i)}(t)|.
		\end{align}
		From (\ref{eq: wuN-1N2}), 
		\begin{align*}
			&||\underline{u}_i(t) - \underline{v}_i(t)||_1 = 2|u_I^{(i)}(t)-v_I^{(i)}(t)|\geq 2|u_I^{(i)}(t-1) + \frac{N-1}{N^3} - v_I^{(i)}(t-1)| \\
			&\geq 2|u_I^{(i)}(t-1)-v_I^{(i)}(t-1)| + \frac{2(N-1)}{N^3} = ||\underline{u}_i(t-1) - \underline{v}_i(t-1)||_1 + \frac{2(N-1)}{N^3}.
		\end{align*}
		\item [(iii)] If $i\in\mathcal{S}_{22}(t)$, then node $i$ is not tested on day $t$, thus from (\ref{eq: u www}), $u_I^{(i)}(t)\geq w_I^{(i)}(t-1)=1-1/N$. Thus, by (\ref{eq: norm12}), $||\underline{u}_i(t) - \underline{v}_i(t)||_1\geq \frac{2(N-1)}{N}$.
	\end{itemize}
	Since we consider $N$ sufficiently large, then we can prove the following lemma.
	\begin{lemma}\label{lem: emptyS}
		$\lim_{t\to\infty}\mathcal{S}_{21}(t)=\varnothing$.
	\end{lemma}
	{\it Proof}. 
	We first prove the following Claims.
	
	{\bf Claim~1}. If (i) $u_I^{(i-1)}(t)\geq 1-1/N$ and node $i-1$ is not tested on day $t$, or (ii) $u_I^{(i+1)}(t)\geq 1-1/N$ and node $i+1$ is not tested on day $t$, or (iii) $u_I^{(i-1)}(t)\geq 1-1/N$ and $u_I^{(i+1)}(t)\geq 1-1/N$, then $u_I^{(i)}(t+1)\geq 1-1/N$.
	
	{\bf Proof}. By (\ref{eq: example-u_I}) and (\ref{eq: no test w=u}), we can derive $u_I^{(i)}(t+1)\geq 1-1/N$ directly.\qed
	
	{\bf Claim~2}.  No node can stay in $\mathcal{S}_{21}(t)$ for successive $\left \lceil{N^3/(N-1)}\right \rceil$ days.
	
	{\bf Proof}. if node $i$ stays in $\mathcal{S}_{21}(t)$ for successive $\left \lceil{N^3/(N-1)}\right \rceil$ days, i.e., from day $\tau$ to day $\tau+\left \lceil{N^3/(N-1)}\right \rceil$, then by  (\ref{eq: wuN-1N2}), $u_I^{(i)}(\tau+\left \lceil{N^3/(N-1)}\right \rceil)>1$, which contradicts with $u_I^{(i)}(t)\leq1$ for all $t$.\qed
	
	Now, we prove the lemma by contradiction. Based on {\bf Claim~2}, assume  there exists at least one $j$ and an increasing sequence $\{t_i\}_{i=0}^{\infty}$ with $\lim_{n\to\infty}t_n=\infty$, such that  $j\in\mathcal{S}_{21}(t_i)$ for all $\{t_i\}_{i=0}^{\infty}$. 
	
	For some $i$,  node $j$ is in $\mathcal{S}_{22}(t_i-1)$ on day $t_i-1$, and  node $j$ is in $\mathcal{S}_{21}(t_i)$ on day $t_i$. In other words, $u_I^{(j)}(t_i)<1-1/N\leq u_I^{(j)}(t_i-1)$.
	From (\ref{eq: example-u_I}) and  {\bf Calim~1}, $u_I^{(j)}(t_i)<1-1/N\leq u_I^{(j)}(t_i-1)$ holds only because node $j$ is tested on day $t_i-1$, and all of its neighbors (i.e., nodes $j-1$, $j+1$) have $u_I^{(j-1)}(t_i-1)<1-1/N$ and $u_I^{(j+1)}(t_i-1)<1-1/N$. However, since $u_I^{(j)}(t_i-1)\geq 1-1/N$ and node $j$ is tested on day $t_i-1$, then  by {\bf Claim~1}, $u_I^{(j-1)}(t_i)\geq 1-1/N$ and $u_I^{(j+1)}(t_i)\geq 1-1/N$. Subsequently, by {\bf Claim~1}, we have $u_I^{(j)}(t_i+1)\geq1 - 1/N$. Thus, on day $t_i+1$, at least one of its neighbors, say $j-1$, has $u_I^{(j-1)}(t_i+1)\geq1 - 1/N$. By {\bf Claim~1}, node $j$ never fall into $\mathcal{S}_{21}(t)$ for $t\in\{t_{i+1}, t_{i+2}, \cdots\}$, which contradicts with the assumption.\qed

	From Lemma~\ref{lem: emptyS}, when $t\to\infty$, we have $|\mathcal{S}_1(t)| = \Theta(N)$ or $|\mathcal{S}_{22}(t)| = \Theta(N)$. Thus, 
	$\sum_{i=1}^{N}||\underline{u}_i(t) - \underline{v}_i(t)||_1 = \Theta(N)$.
	
	\noindent{\bf Step~3}: Consider the computation of $\{\underline{u}_i(t)\}_{i}$  based on Algorithm~\ref{alg: Backward Update}.
	
	In this step, we consider a specific testing policy: We test node~$i$ on day~$k$, where $k \equiv i-1 (mod\,\,M)$ for all $1\leq i\leq M$.

	In event $\mathcal{I}_2$, since the transmission probability $\beta=1$, then all nodes are infected at most $N$ days because there is no recovery. Thus, no node with positive testing result is repeatedly tested. So in at most $2N$ days, all nodes are infectious, and the algorithm finds all infected nodes, so $\underline{u}_i(t) = \underline{v}_i(t)$, $t \geq 2N$.

	In event $\mathcal{I}_1$,  whenever a node is tested, it is negative.
	Node~$1$ is tested on day $0$, the result is negative. On day $1$, node $2$ is tested and the result is negative. By backward updating, since $\beta=1$ and no recovery, then nodes $1\&3$ are inferred to be in state $S$ on day $0$. Since node $2$ is in state $S$ on day $1$. Then, node $1$ is inferred in state $S$ on days $0$ and $1$. 
	
	Assume that nodes $1,2,\cdots, k-2$ are inferred to be in state $S$ by day $k-1$. Now, we day $k$, where $k\leq N$. On day $k$, node $k-1$ is tested negative, hence by backward updating, nodes $k-2$ and $k$ are inferred to be in state $S$ on day $k-1$. By the testing result of node $k-1$ on day $k$,  nodes $1,2,\cdots,k-1$ are inferred in state $S$ by day $k$. By induction, after $N$ days, it clears every node, so $\underline{u}_i(t) = \underline{v}_i(t)$, $t \geq N$.

	From {\bf Steps 1$\sim$3}, we complete the proof.

	\section{$\alpha$-linking Backward Updating}\label{App: Reduction of Complexity}
	
	\subsection{Complexity Reduction}
	Let  $\{x_j\}_{j\in\mathcal{O}(t)}$ be a realization of $\{\theta_j(t)\}_{j\in\mathcal{O}(t)}$ and  $\{y_l\}_{l\in\Theta_i(t)}$  be a realization of $\{\zeta_l(t-1)\}_{l\in\Theta_i(t)}$. Let node~$i$ have state $x$ in day $t-1$. Consider one node $k\in\partial^+_j(t-1)\backslash\{i\}$ and the probability $$\Pr\big(x_j|\{y_l\}_{l\in\partial^+_j(t-1)\backslash\{i\}}, x\big), j\in\Psi_i(t).$$
	Since node $k$ is not infectious if $y_k=L$, $y_k=R$ or $y_k=S$, then the probability above remians the same no matter whether $y_k=L$, $y_k=R$ or $y_k=S$.
	
	Thus, we introduce a new state, denoted by $E$, to be a replacement of $\{L, R, S\}$, and 
	\begin{align*}
		\Pr\big(y_k=E\big) = \sum_{x\in \{L, R, S\}}\Pr\big(y_k=x\big).
	\end{align*}
	Next, denote $\mathcal{X}' = \{I, E\}$. Equation (\ref{eq: fenzi34}) can be re-written as follows:
	\begin{equation}\label{eq: fenzi35}
		\begin{aligned}
			&\Pr\big(\{Y_j(t)\}_{j\in\Psi_i(t)}|\zeta_i(t-1) = x\big)\\
			&=\sum_{\{x_j\}_{j\in\Psi_i(t)}}\prod_{j\in \Psi_i(t)}\Pr\big(Y_j(t)|\theta_j(t)\big) \times\sum_{\{y_l\}_{l\in\Theta_i(t)}}\prod_{j\in\Psi_i(t)}{\tt P}_j\big(x_j|\{y_l\}_{l\in\partial^+_j(t-1)\backslash\{i\}},x\big)\\
			&\times\prod_{z_l\in\mathcal{X}',l\in\Theta_i(t)}\Pr\big(\zeta_l(t-1)=z_l\big),
		\end{aligned}
	\end{equation}
	with reduces  the computation complexity. 
	Subsequently, $\underline{e}_i(t-1)$ in (\ref{eq: chain rule of e2}) can be calculated by (\ref{eq: fenzi35}) directly.
	
	\subsection{$\alpha$-linking Backward Updating}
	In the backward step, the computation complexity is large even in (\ref{eq: fenzi35}). To further reduce the complexity in (\ref{eq: fenzi35}), one way is to update the posterior probability $\underline{e}_i(t)$ in a sparser network. Now, we define {\it $\alpha$-linking Backward Updating} as follows:
	\begin{itemize}
		\item [(i)] We generate a subgraph $\mathcal{G}_\alpha(t)$ based on the pre-determined graph $\mathcal{G}(t)$: Suppose that each edge (in $\mathcal{G}(t)$) exists with probability $\alpha$, $0\leq\alpha\leq1$. If $\alpha=1$, then $\mathcal{G}_\alpha(t) = \mathcal{G}(t)$; if $\alpha=0$, then $\mathcal{G}_\alpha(t)$ is a graph with no edges.   
		\item [(ii)] Backward updating in $\mathcal{G}_\alpha(t)$: Similar with $\partial_{i}(t)$, $\Psi_{i}(t)$, $\Phi_i(t)$ and $\Theta_{i}(t)$, we define $\partial_{i,\alpha}(t)$, $\Psi_{i,\alpha}(t)$, $\Phi_{i,\alpha}(t)$ and $\Theta_{i,\alpha}(t)$ on graph $\mathcal{G}_\alpha(t)$, respectively. Subsequently, replace $\partial_{i}(t)$, $\Psi_{i}$, $\Phi_i(t)$ and $\Theta_{i}(t)$ by $\partial_{i,\alpha}(t)$, $\Psi_{i,\alpha}(t)$, $\Phi_{i,\alpha}(t)$ and $\Theta_{i,\alpha}(t)$ in (\ref{eq: fenzi35}), respectively.
	\end{itemize}

	\section{Proof of Theorem ~\ref{theorem:explor}}\label{example:explor}
	
	\noindent{\bf Step~1}. Preliminaries.

	In  Example~\ref{ex: effect of prob}, $\beta=1$, $\lambda=0$, and $\gamma=0$, there is no recovery and we assume no latent state. 	 Based on (\ref{eq: reward r}), the expression of rewards $\hat{r}_i(t)$ for every node is given as follows. If node $i$ has two neighbors (without quarantine)
	\begin{equation}\label{eq: two neighbors}
		\begin{aligned}
			\hat{r}_i(t) =& u_S^{(i-1)}(t)(1-u_I^{(i-2)}(t))u_I^{(i)}(t)+u_S^{(i+1)}(t)(1-u_I^{(i+2)}(t))u_I^{(i)}(t).
		\end{aligned}
	\end{equation}
	If node $i$ only has one neighbor, then
	\begin{equation}\label{eq: one neighbor}
		\begin{aligned}
			\hat{r}_i(t) =& u_S^{(i+d)}(t)\big(1-u_I^{(i+2d)}(t)\big)u_I^{(i)}(t),\, d\in\{-1,1\}.
		\end{aligned}
	\end{equation} 
	For simplicity, we introduce artificial nodes $-1, 0, N+1, N+2$ with $u_I^{(-1)}(t) = u_I^{(0)}(t) = u_I^{(N+1)}(t) =u_I^{(N+2)}(t)=0$ for all $t$, and these $4$ nodes are never tested.

	\noindent{\bf Step~2}. The RbEx policy.
	
	Under the RbEx policy, the algorithm always tests the nodes with maximum rewards. Let an infectious node be found, for the first time, on day $aN$, where $a$ is a positive real number. Note that until the first infected node is found, in any application of the RbEx policy, $u_I^{(i)}(t)$ is the same for any given $i$, and hence $\hat{r}_i(t)$ is also the same. So, $a$ is the same for any application of the RbEx policy. Recall that in Example~\ref{ex: effect of prob}, nodes that are tested positive will be isolated. The cumulative infections is at least $\min\{aN, N\}$ in the end.

	\noindent{\bf Step~3}. Consider the exploration process of the specific exploration policy.
	
	Recall that  from {\bf Step~2}, an infectious node is found, for the first time, by the RbEx policy with budget $10$ tests on day $aN$. Under the specific defined exploration policy, we can choose a specific $b'$ with $b'<a$,  such that no infectious node is tested by the RbEx policy with budget $9$ tests before and including day $t=b'N$. 
	
	We know that on day $\tau$, nodes $1,2,\cdots, \tau$ are infectious since $\beta=1$.  Note that one test is applied to exploration (randomly choice) on every day, so with probability
	\begin{align}\label{eq: no infectious}
		\prod_{\tau'=1}^{\tau}(1-\frac{\tau'}{N}),
	\end{align}
	no infectious node is explored from the initial day to day $\tau$. 
	Then,
	with probability
	\begin{align*}
		\prod_{\tau'=1}^{\tau-1}(1-\frac{\tau'}{N})\cdot\frac{\tau}{N},
	\end{align*}
	one infectious node is detected on day $\tau$. Thus, with probability
	\begin{align}\label{eq: no infectious prob}
		\sum_{\tau=1}^{t}\prod_{\tau'=1}^{\tau-1}(1-\frac{\tau'}{N})\cdot\frac{\tau-1}{N},
	\end{align}
	one infectious node is tested by exploration process on day $\tau$ ($\tau\leq t$), and this node is not the new infectious one on day $\tau$, i.e., has index $\tau$.
	The probability defined in (\ref{eq: no infectious prob}) increases with $t$ when $N$ is fixed, and it can be close to $1$ when $t$ close to $N$.
	Therefore,
	We can choose proper parameters $b'$ and $N$ such that the probability defined in (\ref{eq: no infectious prob}) is larger than or equal to $p_0$. In particular, if $N$ is large, we can choose a relatively small $b'$. In  Theorem~\ref{theorem:explor}, we set $p_0\geq 99/100$.

	Let the infectious node detected (for the first time) by the exploration process have index $j$ on day $t'$, where $t' \leq t$. As discussed above, node $j$ is not the new infectious node on day $t'$, so we have $j<t'$. In other words, node $j+1$ must be infecitous on day $t'$ with a positive test result, i.e., $Y_j(t') = 1$. By Step~1 in Algorithm~\ref{alg: Backward Update}, the updated posterior probability of node $j$ 
	\begin{align}\label{eq: e j}
		e_{I}^{(j)}(t'-1) = 1,\quad e_S^{(j)}(t'-1) = 0.
	\end{align}
	Again, by Step~1 in Algorithm~\ref{alg: Backward Update},
	\begin{align}\label{eq: w j}
		w_{I}^{(j-1)}(t') =  w_I^{(j+1)}(t') = 1.
	\end{align}
	Then, by Step~2 in Algorithm~\ref{alg: Backward Update},
	\begin{align}\label{eq: u j}
		u_{I}^{(j-2)}(t'+1) = u_{I}^{(j-1)}(t'+1) =u_I^{(j)}(t'+1)=u_{I}^{(j+1)}(t'+1) =u_{I}^{(j+2)}(t'+1) = 1.
	\end{align}
	Since $j$ is detected and isolated on day $t'$, then,
	\begin{equation}\label{eq: rewards j1}
		\begin{aligned}
			\hat{r}_j(t'+1) = 0.
		\end{aligned}
	\end{equation}
	By (\ref{eq: one neighbor}) and (\ref{eq: u j}), 
	\begin{equation}\label{eq: rewards j2}
		\begin{aligned}
			&\hat{r}_{j-1}(t'+1) =u_S^{(j-2)}(t'+1)\big(1-u_I^{(j-3)}(t'+1)\big)=0\\
			&\hat{r}_{j+1}(t'+1) = u_S^{(j+2)}(t'+1)\big(1-u_I^{(j+3)}(t'+1)\big)=0.
		\end{aligned}
	\end{equation}
	By (\ref{eq: two neighbors}) and (\ref{eq: u j}), 
	\begin{equation}\label{eq: rewards j3}
		\begin{aligned}
			&\hat{r}_{j-2}(t'+1) =u_S^{(j-3)}(t'+1)\big(1-u_I^{(j-4)}(t'+1)\big)\\
			&\hat{r}_{j+2}(t'+1) = u_S^{(j+3)}(t'+1)\big(1-u_I^{(j+4)}(t'+1)\big).
		\end{aligned}
	\end{equation}

	\noindent{\bf Step~4}.The exploitation 
	process of the specific exploration policy.
	
	We first study an extreme case where no tests are applied. In this case, denote the prior probability of node $i$ on day $\tau$ as $U_I^{(i)}(\tau)$, which can be calculated by the following recursion:
	\begin{align}\label{eq: max recursion}
		U_I^{(i)}(\tau+1) =  U_I^{(i)}(\tau) + U_S^{(i)}(\tau)\big(1 - (1-U_I^{(i-1)}(\tau))(1-U_I^{(i+1)}(\tau))\big).
	\end{align}
	Based on (\ref{eq: max recursion}), recall that $U_I^{(i)}(0)=0$ if $i\leq \frac{9N}{10}$, and $U_I^{(i)}(0)=\frac{10\epsilon}{N}$ if $\frac{9N}{10}< i\leq N$, then $U_I^{(i)}(\tau)$ increases over $\tau$ and is a function of $\epsilon$.  Then, given $b'$, $N$ and $t=b'N$, we can choose a small enough $\epsilon$, denoted by $\epsilon(b', N)$, such that $U_I^{(i)}(2t)<\frac{1}{2}$ for all $i$. Since $U_I^{(i)}(\tau)$ increases over $\tau$, then
	$U_I^{(i)}(\tau)<\frac{1}{2},\,\tau\leq 2t$.

	Now, we introduce the exploitation process. Let $t=b'N<\min\{\frac{9}{40}, a\}N$. There are at most $2t$ infectious nodes on day $2t$, i.e., nodes~$1,2,\cdots,2t$. Since $t<\min\{\frac{9}{40}, a\}N$, then nodes with index from  $9N/10-2t$ to $N$ are in state $S$, which implies nodes with index from  $9N/10-2t$ to $N$ can never be tested positive before day $2t$. Thus, on any day $\tau\leq 2t$, for $9N/10-2t\leq i\leq N$, if node $i$ is tested, and the testing result is negative. Recall that $U_I^{(i)}(\tau)$ in (\ref{eq: max recursion}) is calculated without any negative testing results. Hence, $u_I^{(i)}(\tau)\leq U_I^{(i)}(\tau)$. Furthermore, with the condition $t=b'N<\min\{\frac{9}{40}, a\}N$, we can find a small enough $\epsilon(b', N)$, such that under the specific exploration policy,
	\begin{align}\label{eq: increase over t r}
		u_I^{(i)}(\tau)<\frac{1}{2},\quad \tau\leq 2t,\,\, 9N/10-2t\leq i\leq N.
	\end{align}

	In the rest, we divide the nodes in to $3$ sets: $\mathcal{Q}_1 = \{i| i\leq 2t\}$, $\mathcal{Q}_2 = \{i| 2t<i<9N/10-2t\}$, and $\mathcal{Q}_3 = \{i| 9N/10-2t\leq i\leq N\}$.

	{\bf Fact~1}. For $i\in\mathcal{Q}_1$ and $\tau\leq 2t$,  $u_I^{(i)}(\tau)=1$ or $u_I^{(i)}(\tau)=0$.
	
	{\it Proof}.  If no test is applied to $\mathcal{Q}_1$, then $u_I^{(i)}(\tau)=0$ for all $i\in\mathcal{Q}_1$. 
	
	On some day $\tau\leq 2t$, if one node with index $j\in \mathcal{Q}_1$ is tested positive on day $\tau-1$, then by  (\ref{eq: u j}), $u_{I}^{(j-2)}(\tau) = u_{I}^{(j-1)}(\tau) =u_I^{(j)}(\tau)=u_{I}^{(j+1)}(\tau) =u_{I}^{(j+2)}(\tau) = 1$. In other words, if node $j$ is tested positive on day $\tau-1$, then node $j$, its neighbors and neighbors of neighbors have probability of infection equal to $1$ on day $\tau$.
	
	On some day $\tau$, if node $j$ is not tested positive on day $\tau-1$, and neither of its neighbors and neighbors of neighbors are is not tested positive, then $u_I^{(j)}(\tau)=1$ only when $u_I^{(j)}(\tau-1)=1$, or $u_I^{(j-1)}(\tau-1)=1$ or $u_I^{(j+1)}(\tau-1)=1$ since $\beta=1$. Otherwise 
	$u_I^{(j)}(\tau)=0$.\qed

	{\bf Fact~2}. For $i\in\mathcal{Q}_1$ and $\tau\leq 2t$, 
	$\hat{r}_i(\tau)=1$ or $\hat{r}_i(\tau)=0$. 
	
	{\it Proof}. If $u_I^{(i)}(\tau)=0$, then $\hat{r}_i(\tau)=0$ by (\ref{eq: two neighbors}) and (\ref{eq: one neighbor}). 
	
	Now, we consider $u_I^{(i)}(\tau)=1$ in the following cases:
	(i) If both neighbors of node $i$ are isolated, then $\hat{r}_i(\tau)=0$. (ii) If one of neighbors of node $i$ (for example, node $i-1$) is isolated, then by (\ref{eq: one neighbor}), $\hat{r}_i(\tau)=0$ when $u_I^{(i+1)}(\tau)=1$, and $\hat{r}_i(\tau)=1$ when $u_I^{(i+1)}(\tau)=0$. (iii) If both neighbors are not isolated, then $u_I^{(i-1)}(\tau-1)=1$ or $u_I^{(i+1)}(\tau-1)=1$, otherwise, $u_I^{(i)}(\tau)=0$. Since there is no recovery, then $u_I^{(i-1)}(\tau)=1$ or $u_I^{(i+1)}(\tau)=1$.
	By {\bf Fact~1}, $u_I^{(j)}(\tau)=1$ or $u_I^{(j)}(\tau)=0$ when $j\in\mathcal{Q}_1$. If $u_I^{(i+1)}(\tau)=u_I^{(i-1)}(\tau)=1$, then 
	$\hat{r}_i(\tau)=0$. If $u_I^{(i+1)}(\tau)=1$, then by (\ref{eq: two neighbors}), $\hat{r}_i(\tau)=0$ when $u_I^{(i-2)}(\tau)=1$, $\hat{r}_i(\tau)=1$ when $u_I^{(i-2)}(\tau)=0$. If $u_I^{(i-1)}(\tau)=1$, then by (\ref{eq: two neighbors}), $\hat{r}_i(\tau)=0$ when $u_I^{(i+2)}(\tau)=1$, $\hat{r}_i(\tau)=1$ when $u_I^{(i+2)}(\tau)=0$. \qed

	From (\ref{eq: increase over t r}), for all $\tau\leq 2t$ and $i\in\mathcal{Q}_3$, we have 
	\begin{align}\label{eq: smallerthan1}
		\hat{r}_i(\tau)\leq 2u_I^{(i)}(\tau)<1.
	\end{align}
	Note that only nodes in $\mathcal{Q}_1$ and $\mathcal{Q}_3$ may have positive probability of infection. For all $\tau\leq 2t$ and $i\in\mathcal{Q}_2$,  since $u_I^{(i)}(\tau)=0$, then $\hat{r}_i(\tau)=0$. Therefore, a node with reward equal to $1$ has the largest reward.

	Recall that on day $t'$, node $j$ is tested positive. From (\ref{eq: rewards j3}), nodes $j-2$ and $j+2$ have largest rewards ($=1$) on day $t'+1$, which are exploited on day $t'+1$, and all other nodes in $\mathcal{Q}_1$ have rewards $0$. This is because  $t^\prime$ is the first day when a positive node is found. Since node $j$ is tested positive and isolated on day $t'$, then all infectious nodes with index less than $j$ can no longer infect other nodes in the line network. Now, we consider the nodes with index larger than $j$. Recall that $j<t'$, so node $j+1$ must be infectious on day $t'$, and node $j+2$ must be infectious on day $t'+1$ since $\beta=1$. Thus, node $j+2$ is tested positive and is isolated. Since the network is a line,  both nodes $j+1$ and $j+2$ can no longer infect other nodes once node $j+2$ is isolated. Note that nodes in $\mathcal{Q}_3$ have positive rewards. When $N$ is sufficiently large, in the rest of the exploitation process, nodes in $\mathcal{Q}_3$ are tested. Recall that we have one test for exploration, and we can isolate at least $2$ infectious nodes with index larger than $j$. 
	
	Repeat the process, we exploit nodes $j+4, j+6, \cdots$ on day $t'+2, t'+3,\cdots$, respectively. Consider the direction from node $1$ to node $N$. On every day, there is at most one new infectious node, but at least two infectious nodes can be isolated. On some day, denote as day $t'+x$, 
	the exploitation process can progress beyond the infections (exceeding by one node) for the first time. In other words, node $j+2x$ is tested negative on day $t'+x$. By Step~2 in Algorithm~\ref{alg: Backward Update}, $e_I^{(j+2x)}(t'+x-1)=0$. However, since $w_I^{(j+2x-1)}(t'+x)=1$ becuase node $j+2x-2$ is tested positive on day $t'+x-1$. By Step~1 in Algorithm~\ref{alg: Backward Update}, $u_I^{(j+2x)}(t'+x+1)=1$, hence by (\ref{eq: rewards j3}), $\hat{r}_{j+2x}(t'+x+1)=1$, which implies node $j+2x$ has the largest reward and is exploited on day $t'+x+1$, and it will be tested positive. On day $t'+x+1$, all infectious nodes are isolated. 
	
	Finally, we can calculate the total number of infections to be
	\begin{align*}
		j + 2(t'-j) = 2t'-j \leq 2t'\leq 2t = 2b'N.
	\end{align*}
	Let $b=2b'$. This is an improvement by a factor of at least $\frac{a}{b}$ in comparison to the RbEx strategy, where $\frac{a}{b}$ can be as large as desired by increasing the value of $N$ or decreasing $p_0$.

	\section{Construction of Networks and Further Results}\label{sec: Further Results and Real-Data Networks}
	
	\subsection{Constructions of SBM and V-SBM}\label{App: simulations SBM V-SBM}
	
	In this section, we construct SBMs and its variants.

	\paragraph{SBM}
	
	The SBM is a generative model for random graphs. The graph is divided into several communities, and subsets of nodes are characterized by being connected with one another with particular edge densities.\footnote{Here, we assume that $M$ is an exact divisor of $N$.}  The intra-connection probability is $p_1$, and inter-connection probability is $p_2$. We denote the SBM as SBM$(N, M, p_1, p_2)$.  Note that the (expected) number of edges, denoted by $|\mathcal{E}|$, is
	\begin{align}\label{eq: number of edge}
		|\mathcal{E}| = \frac{p_1}{2}N(\frac{N}{M}-1) + \frac{p_2}{2}\frac{N^2}{M}(M-1).
	\end{align}
	Now, we fix $|\mathcal{E}|$, and choose the pair $(p_1, p_2)$ under a fixed $|\mathcal{E}|$ in (\ref{eq: number of edge}). The aim of fixing $|\mathcal{E}|$ is to guarantee that the transmission of the disease would not be affected by edges.

	\paragraph{V-SBM}
	
	Now, we consider a variant of SBM, denoted by V-SBM.  Different from SBM, we only allow nodes in cluster~$i$ to connect to nodes in successive clusters (the neighbor clusters).  Denote the V-SBM as V-SBM$(N, M, p_1, p_2)$. Similarly, the expected number of edges, denoted by $|\mathcal{E}|$, is
	\begin{align}\label{eq: number of edge1}
		|\mathcal{E}| = \frac{p_1}{2}N(\frac{N}{M}-1) + p_2\frac{N^2}{M}.
	\end{align}
	Now, we fix $|\mathcal{E}|$, and choose the pair $(p_1, p_2)$ under a fixed $|\mathcal{E}|$ in (\ref{eq: number of edge1}). The aim of fixing $|\mathcal{E}|$ is to guarantee that the transmission of the disease would not be affected by edges.

	\subsection{The impact of $\gamma_c$ and $L_p$ individually}\label{App: The impact of parameters individually}
	In this subsection, we investigate the role of $\gamma_c$ and $L_p$ individually, not through the common factor $\delta$. We consider different WS networks with degrees $d=4,6$, and then adjust the rewiring probability $\delta$, such that one of $(\gamma_c, L_p)$ is almost constant, and the other is varying. We can see that the trend is similar to what we observed by varying $\delta$ in Table~\ref{tabla: WS clustering coefficient}.

	\begin{table}
		\centering
		\begin{tabular}{|c|c|c|c|c|c|}
			\hline  
			WS, $(d, \delta)$&  $\gamma_c$ &  $L_p$ & ${\tt Ratio}$\\
			\hline
			$(6, 0.05)$& {\bf 0.504}&$4.952$ &$.0003$ \\
			\hline
			$(4,0)$& {\bf 0.500}& $62.876$& $0.191$ \\
			\hline  		
			$(6,0.1)$ &{\bf 0.456} & $5.718$ & $-0.027$\\
			\hline
			$(4,0.03)$ & {\bf 0.456} & $10.810$ & $0.097$ \\
			\hline
		\end{tabular}
		\caption{Clustering coefficients of WS networks}\label{tabla: fixed clustering coefficient}
	\end{table}

	\begin{table}
		\centering
		\begin{tabular}{|c|c|c|c|c|}
			\hline  
			WS, $(d, \delta)$&  $\gamma_c$ &  $L_p$ &${\tt Ratio}$ \\
			\hline
			$(6,.001)$& $0.599$&{\bf 21.188} & 0.209\\
			\hline
			$(4,.0075)$& $0.489$& {\bf 21.264}& $0.182$\\
			\hline  		
			$(6,.005)$ &$0.592$ & {\bf 14.310}& $0.211$\\
			\hline
			$(4,.015)$ & $0.473$ & {\bf 14.253} & $0.174$\\
			\hline		
			$(6,.009)$ &$0.585$ & {\bf 12.081}& $0.137$\\
			\hline
			$(4,.0225)$ & $0.467$ & {\bf 12.171} & $0.125$\\
			\hline
		\end{tabular}
		\caption{Clustering coefficients of WS networks}\label{tabla: fixed path length}
	\end{table}

\end{document}